\let\originalcline\cline
\documentclass[pdflatex,sn-mathphys-num,oneside]{sn-jnl}
\let\cline\originalcline

\usepackage{graphicx}%
\usepackage{amsmath,amssymb,amsfonts}%
\usepackage{amsthm}%
\usepackage{mathrsfs}%
\usepackage{xcolor}%
\usepackage{textcomp}%
\usepackage{manyfoot}%
\usepackage{booktabs}%
\usepackage{algorithm}%
\usepackage{algorithmicx}%
\usepackage{algpseudocode}%
\usepackage{listings}%
\usepackage{comment}

\usepackage{multirow}%
\usepackage[ruled,vlined,algo2e]{algorithm2e}

\usepackage[title]{appendix}%
\makeatletter
\AtBeginDocument{%
  \let\sn@orig@appendices\appendices
  \def\appendices{%
    \sn@orig@appendices
    \def\sectionname{Supplementary Section}
    \def\figurename{Supplementary Fig.}
    \def\tablename{Supplementary Table}
    \gdef\thefigure{S\arabic{figure}}
    \gdef\thetable{S\arabic{table}}
    \gdef\theequation{S\arabic{equation}}
  }%
}%
\makeatother

\theoremstyle{thmstyleone}
\newtheorem{theorem}{Theorem}
\newtheorem{proposition}{Proposition}
\newtheorem{lemma}{Lemma}

\newtheorem{corollary}{Corollary}

\usepackage{geometry}
\begin{document}

\title[Article Title]{Topology-induced Operators Reveal Complementary Graph Representations without Training}


\author[1]{\fnm{Meng} \sur{Qin}}

\author[1*]{\fnm{Jinqiang} \sur{Cui}}\email{cuijq@pcl.ac.cn}

\author[2*]{\fnm{Hongwei} \sur{Zheng}}\email{hwzheng@pku.edu.cn}

\author[3]{\fnm{Weihua} \sur{Li}}

\author[4*]{\fnm{Sen} \sur{Pei}}\email{sp3449@cumc.columbia.edu}

\affil[1]{\orgdiv{Department of Strategic \& Advanced Interdisciplinary Research}, \orgname{Pengcheng Laboratory (PCL)}, \orgaddress{\city{Shenzhen}, \state{Guangdong}, \country{China}}}

\affil[2]{\orgname{Beijing Academy of Blockchain and Edge Computing}, \orgaddress{\city{Beijing}, \country{China}}}

\affil[3]{\orgdiv{LMIB, NLSDE, and School of Artificial Intelligence}, \orgname{Beihang University}, \orgaddress{\city{Beijing}, \country{China}}}

\affil[4]{\orgdiv{Mailman School of Public Health}, \orgname{Columbia University}, \orgaddress{\city{New York}, \country{USA}}}

\abstract{
Graph representation learning has largely focused on designing increasingly sophisticated models to transform graph topology into vector representations, or embeddings. However, the extent to which embedding quality depends on model learning, rather than on the underlying topological transformations, remains unclear. Here, we show that informative embeddings can be derived without complicated model design and gradient-based training. Propagating random features through implicit hierarchical structures induced by random walks and anonymous walks yields embeddings that capture node proximity and structural role, respectively. These two training-free embeddings preserve complementary aspects of graph organization and perform competitively with classic and recent methods across various node-, edge-, and graph-level tasks. They often require substantially less computation, resulting in a favorable quality–efficiency trade-off. Combining the two types of embeddings further improves inference quality of some tasks compared with using either embedding type alone. 
Our results suggest that informative graph embeddings can arise from carefully chosen topological transformations before any learning operation is applied.
}


\maketitle

\section{Introduction}\label{Sec:Intro}
Graphs (also called networks) provide a universal abstraction for modeling complex systems across nature and engineering, including ecological food webs \cite{reji2025species}, neural connectomes \cite{sporns2005human}, the Internet \cite{pastor2007evolution}, power grids \cite{brummitt2012suppressing}, and transportation systems \cite{guimera2005worldwide}.
Many scientific and engineering problems in these systems \cite{giot2003protein,schneider2011mitigation} can be formulated as inference tasks on graphs, such as community detection \cite{fortunato202220,su2022comprehensive}, anomaly detection \cite{qiao2025deep,ekle2024anomaly}, and link prediction \cite{martinez2016survey,kumar2020link,qin2023temporal}.
Graph representation learning has emerged as a fundamental technique for graph inference, mapping nodes, edges, or entire graphs into low-dimensional vector representations (also called embeddings) that preserve information relevant to downstream tasks.
Modern graph inference models, exemplified by graph convolutional networks \cite{Bhattacharya2026} and graph transformers \cite{shehzad2026graph}, increasingly rely on sophisticated architectures and trainable transformations to obtain such embeddings. Yet it remains unclear how much of the resulting embeddings is produced by model learning and how much is already determined by the topological transformations applied to the graph. Understanding this distinction is important for clarifying how graph topology is converted into representations and which aspects of this process genuinely require learning.

A key difficulty in addressing this question is that graph topology contains multiple types of properties suitable for different inference tasks. Two fundamental yet distinct examples are node positions and identities.
Node position concerns where a node lies relative to others in the graph. For instance, nodes belonging to the same community, with dense local connectivity, tend to have a high overlap of neighbors and hence nearby positions \cite{fortunato202220}. Position embedding (alternatively called proximity-preserving embedding \cite{rossi2020proximity,zhu2021node}) is designed to preserve such relative proximity.
Node identity, in contrast, concerns what structural role a node plays in the graph, which is relevant to the node's centrality \cite{grando2018machine}, ego-net structure \cite{mengirwe}, and degree patterns \cite{ribeiro2017struc2vec}. Identity embedding (also defined as structural role-based embedding \cite{rossi2020proximity,zhu2021node}) aims to encode this structural role.
These two properties provide complementary views of graph organization. Nodes from different communities with distinct positions can share similar structural roles, whereas nodes in the same community can also play substantially different roles.
Hence, a graph representation that preserves one aspect well does not necessarily capture the other.

This distinction also makes it challenging to determine how different topological properties emerge in modern graph representations. Most recent graph inference models, especially graph neural networks (GNNs) \cite{ju2024comprehensive,ju2025survey,zheng2026graph}, combine node attributes, topology-induced operators, trainable transformations, and task-specific objectives. Their resulting embeddings reflect intertwined effects of topology, attributes, and learning, making it hard to isolate which topological transformations preserve a specific property. This ambiguity is particularly relevant for attributed graphs, where attributes may be partially correlated or even inconsistent with topology \cite{newman2016structure,peel2017ground,qin2018adaptive,qin2021dual,wang2020gcn}.
Recent efforts have attempted to enhance topology awareness regarding node positions and identities by designing dedicated mechanisms for extracting structural information \cite{zhang2024disentangled,tang2025hyperrole,you2021identity,you2019position,devvrit2022s3gc,liu2024revisiting}.
However, these approaches often emphasize one type of topology property and combine topology-aware transformations with feature processing and gradient-based model optimization. Consequently, it remains unclear how complementary topology properties relate to one another and, more fundamentally, whether sophisticated learning is necessary to encode them. This separation is also relevant to efficiency, because elaborate feature extraction and training procedures introduce substantial computational costs.

Here, we address the above challenges from an operator-level perspective. Our key insight is that different stochastic transformations of graph topology can expose distinct structural channels even without learning model parameters.
Many classic position embedding methods, such as node2vec \cite{grover2016node2vec}, exploit random walks (RWs), where nodes with higher overlaps of neighbors and thus closer positions are more likely to be visited by RW sampling.
In contrast, anonymizing an RW according to the order in which nodes first occur removes their location-specific labels (Fig.~\ref{Fig:PoC}b) while retaining the pattern of structural relationships \cite{ivanov2018anonymous}. The resulting anonymous walks (AWs) have been theoretically shown to contain sufficient information to reconstruct node ego-nets \cite{micali2016reconstructing}.
Motivated by these findings, we develop PI-HIST (\underline{P}osition and \underline{I}dentity embedding from \underline{H}idden h\underline{I}erarchical \underline{S}tructures of graph \underline{T}opology).
It treats RWs and AWs as topology-induced operators rather than viewing them merely as sampling procedures. We use random features as the common inputs these operators and derive embeddings via a only single training-free feedforward propagation (FFP), which can expose complementary information about node positions and identities.
This simple design also suggests the potential for a better quality-efficiency trade-off since there is no iterative gradient-based optimization.

We systematically evaluate this framework across $7$ node-, edge-, and graph-level inference tasks, comparing PI-HIST with $18$ classic and state-of-the-art baselines on $8$ real-world and $9$ synthetic graphs.
Our results show that even with a random feature input and without gradient-based training, informative position and identity embeddings can emerge from PI-HIST's topology-induced propagation. The derived two types of embeddings exhibit different strengths across inference tasks. Their simple combination further improves the inference quality of some tasks compared with either embedding type alone.
The training-free design also leads to favorable quality-efficiency trade-offs relative to more elaborate embedding methods.
Collectively, these outcomes suggest that a substantial fraction of structural signals for informative graph embeddings can arise from the topology-induced operator before any learned transformation is introduced.
More broadly, disentangling complementary graph topology channels provides a new insight into studying graph embedding mechanisms at the operator level and offers a simple perspective on designing efficient and interpretable graph inference models.

\begin{figure}[t]
    \centering
    \includegraphics[width=1.0\linewidth,trim=0 0 0 0,clip]{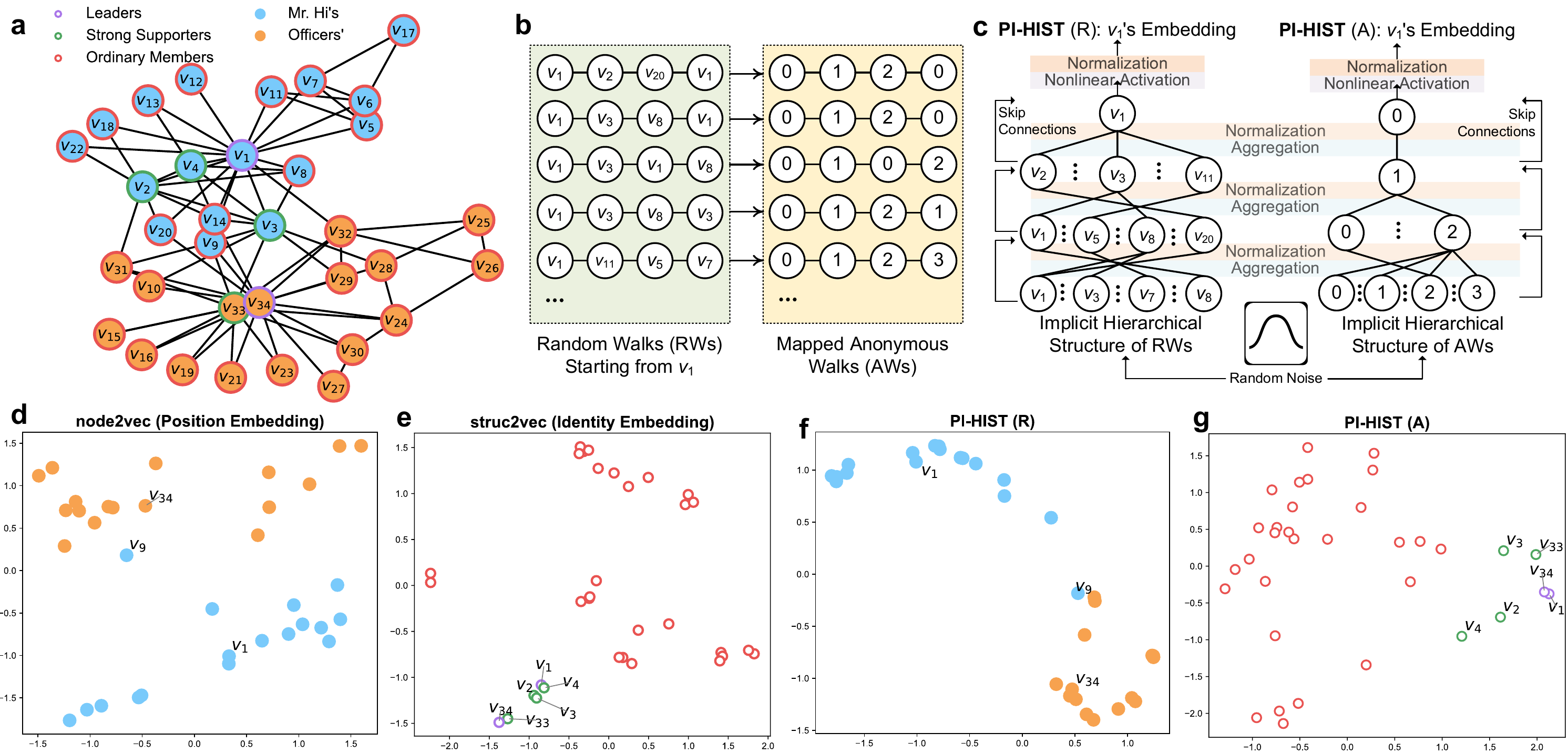}
    \caption{\textbf{Overview of PI-HIST with RWs, AWs, and position/identity embeddings}.
    \textbf{a}, Zachary's karate club network with two communities (blue and orange nodes) and three structural roles (purple, green, and red circles).
    \textbf{b}, Example RWs starting from node $v_1$ with length $L=3$ and their corresponding AWs.
    \textbf{c}, Architecture of PI-HIST based on RW- and AW-induced hierarchical structures. Starting from random features, PI-HIST derives node embeddings through a single FFP over a hierarchical structure, combined with simple skip connections, normalization, and nonlinear activation.
    PI-HIST~(R) and PI-HIST~(A) denote variants using RW and AW, respectively.
    \textbf{d}-\textbf{g}, Visualization of node2vec, struc2vec, and PI-HIST embeddings, obtained by reducing the original embedding dimensionality $d=16$ to $2$ via $t$-SNE.
    As classic position and identity embedding methods, node2vec and struc2vec can distinguish between community members and roles corresponding to node positions and identities, respectively. Surprisingly, PI-HIST~(R) and (A) can obtain position and identity embeddings consistent with node2vec and struc2vec, even with a single untrained FFP from random feature inputs.
    }
    \label{Fig:PoC}
\end{figure}

\section{Results}\label{Sec:Res}
\subsection{Proof-of-concept study of PI-HIST on position and identity embeddings}
To isolate the contribution of graph topology, we consider an unsupervised embedding setting in which topology is the only available input.
The proposed PI-HIST method adopts a simple design (Fig.~\ref{Fig:PoC}c) that uses RW- and AW-induced hierarchical structures as parameter-free backbones, whose layers from bottom to top correspond to the reversed order of RWs and AWs.
We feed random features to these backbones and directly get embeddings via a single FFP, combined with simple normalization, skip connections, and nonlinear activation.
For simplicity, we denote the two variants based on RW and AW as PI-HIST~(R) and PI-HIST~(A), respectively.
We later demonstrate that PI-HIST~(R)'s embedding derivation is equivalent to propagating random features through an untrained GNN with mean aggregation (cf. `Methods' section).
Different from RW, there are currently no closed-form solutions to directly obtain AW statistics.
We therefore develop an efficient algorithm to estimate AW statistics and extract associated hierarchical structures for PI-HIST~(A) before FFP. Full technical details are elaborated in `Methods' section and Supplementary Section~\ref{App:Meth}.
We also provide theoretical support for the effectiveness of PI-HIST in Supplementary Section~\ref{App:Th-Sup}.

The ability of PI-HIST to capture specific properties can be inferred from whether the resulting embeddings behave consistently with existing approaches.
We apply node2vec \cite{grover2016node2vec} and struc2vec \cite{ribeiro2017struc2vec}, which are classic position and identity embedding methods, to the widely studied Zachary's karate club network \cite{zachary1977information}, with the embedding dimensionality $d$ reduced from $16$ to $2$ using $t$-SNE \cite{van2008visualizing}.
This network records social connections between $34$ club members (Fig.~\ref{Fig:PoC}a), who were divided into $2$ competing factions after a conflict between the head coach, Mr. Hi ($v_{34}$), and club officers led by the president, John A. ($v_1$).

In Fig.~\ref{Fig:PoC}d, node2vec can distinguish most members of the two factions (or communities) highlighted by blue and orange nodes, where nodes from the same community have high overlaps of neighbors, close relative positions, and thus similar position embeddings.
Consistent with the original findings of Zachary~\cite{zachary1977information}, node2vec may misclassify $v_9$, which is structurally closer to the Officers' faction (orange nodes) but has to join Mr. Hi’s camp (blue nodes) due to a black belt test.
Nodes in this network also play different roles. Mr. Hi ($v_1$) and John A. ($v_{34}$) should be leaders of the two factions. We define nodes with high centrality as strong supporters of a faction and the rest nodes as ordinary members. In Fig.~\ref{Fig:PoC}e, struc2vec has the potential to distinguish ordinary members (red circles) from leaders and strong supporters (purple and green circles).

Surprisingly, both PI-HIST variants can obtain embeddings consistent with node2vec and struc2vec (Fig.~\ref{Fig:PoC}d-e) even through a single training-free FFP from random inputs.
PI-HIST~(R) discriminates most members from the two communities (Fig.~\ref{Fig:PoC}f), while PI-HIST~(A) effectively separates nodes with different roles (Fig.~\ref{Fig:PoC}g). Particularly, PI-HIST~(A) shows stronger identity-capturing ability than struc2vec, enabling the further distinction between leaders (purple circles) and strong supporters (green circles).
These observations provide initial evidence that the training-free propagation through RW- and AW-induced hierarchical structures can generate distinct position and identity embeddings.
They also motivate a broader question: whether the complementary information exposed by the RW- and AW-induced operators persists across different graph inference tasks.
We further conduct systematic evaluations and case studies at node, edge, and graph levels (Table~\ref{Tab:Task}), examining both the individual and joint contributions of the two types of embeddings.

\begin{table}[t]
\centering
\caption{Summary of experimental evaluations and case studies in this work.}
\label{Tab:Task}
\begin{tabular}{l|l|l|l}
\hline
\textbf{Levels} & \textbf{Tasks} & \textbf{\begin{tabular}[c]{@{}l@{}}Position/\\ Identity-Related\end{tabular}} & \textbf{Outcomes} \\ \hline
\multirow{4}{*}{Node-level} & Node identity classification & Identity & \multirow{4}{*}{Fig.~\ref{Fig:NC_Eva}, Supplementary Fig.~\ref{Fig:NC_Eva_Apx}} \\
 & Community detection & Position &  \\ \cline{2-3}
 & Node position classification & Position &  \\
 & Node identity clustering & Identity &  \\ \hline
\multirow{2}{*}{Edge-level} & Link prediction & Uncertain & \multirow{2}{*}{Fig.~\ref{Fig:GRLP_Eva}, Supplementary Fig.~\ref{Fig:GRLP_Eva_Apx}} \\
 & Graph reconstruction & Uncertain &  \\ \hline
Graph-level & Graph superfamily identification & Uncertain & Fig.~\ref{Fig:GSI_Eva}, Supplementary Fig.~\ref{Fig:GSI_Eva_PI_HIST},~\ref{Fig:GSI_Eva_Apx} \\ \hline
\end{tabular}
\end{table}

\begin{figure}[t]
    \centering
    \includegraphics[width=1.0\linewidth,trim=0 0 0 0,clip]{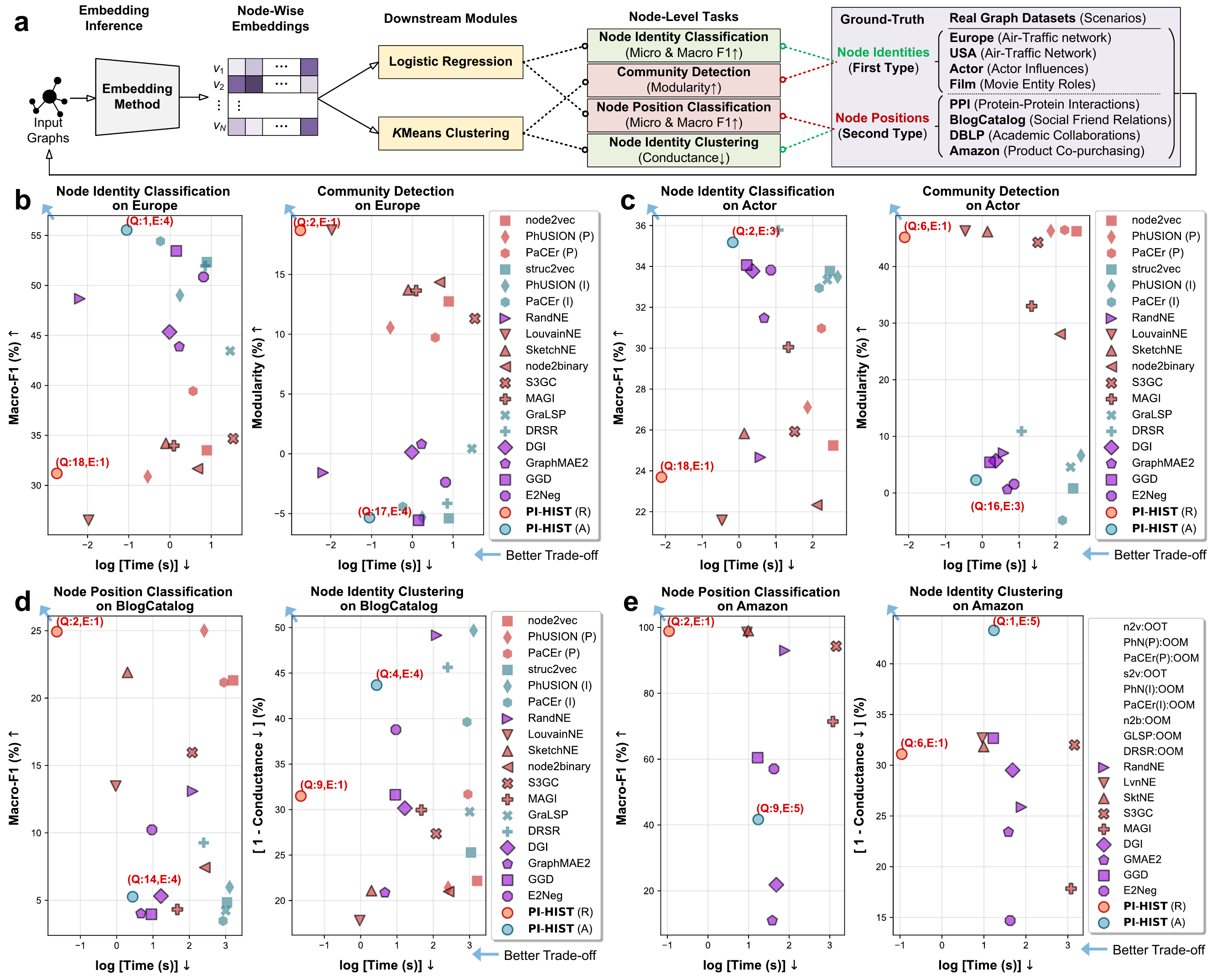}
    \caption{
    \textbf{Representative evaluation results of node-level tasks}.
    \textbf{a}, Overview of evaluation settings on $8$ public real-world graph datasets with $2$ types of ground-truth using $4$ node-level tasks.
    \textbf{b-e}, Evaluation on representative datasets with identity (Europe and Actor) or position (BlogCatalog and Amazon) ground-truth. Specifically, node identity classification (macro F1) and community detection (modularity) assess the quality of position and identity embeddings ($y$-axis) on identity-ground-truth datasets (b-c), whereas node position classification (macro-F1) and identity clustering (conductance) evaluate the quality of corresponding embeddings ($y$-axis) on position-ground-truth datasets (d-e).
    The logarithm of inference time (sec) is used to measure the efficiency of all cases ($x$-axis). $\uparrow$ and $\downarrow$ denote higher-is-better metrics and lower-is-better metrics, respectively.
    OOT and OOM represent out-of-time and out-of-memory exceptions.
    Q and E denote rankings of quality and efficiency obtained by the two PI-HIST variants.
    Blue arrows indicate directions toward better quality–efficiency trade-offs.
    Red and green dots denote position and identity embedding baselines, respectively. It is uncertain for those with purple dots which property they capture.
    Black-edged dots highlight efficient or scalable baselines.
    Surprisingly, PI-HIST~(R) and (A), despite their extremely simple architectures, achieve top-tier quality for position and identity embedding while maintaining top-tier efficiency in most cases, demonstrating their effectiveness in capturing distinct topology properties and achieving a favorable quality–efficiency trade-off.
    Additional results with more metrics and on other datasets are detailed in Supplementary Fig.~\ref{Fig:NC_Eva_Apx} and Supplementary Section~\ref{App:Res}.
    }
    \label{Fig:NC_Eva}
\end{figure}

\subsection{Evaluation of node-level tasks for position and identity embedding}

We first quantitatively examine whether both PI-HIST variants can derive informative position and identity embeddings using $4$ node-level tasks evaluated on $8$ public real-world graph datasets (Fig.~\ref{Fig:NC_Eva}a).
The datasets cover diverse scenarios and can be categorized into two types according to their ground-truth. Europe \cite{ribeiro2017struc2vec}, USA \cite{ribeiro2017struc2vec}, Actor \cite{jiao2021survey}, and Film \cite{jiao2021survey} are the first type of datasets providing labels about node identities (e.g., function roles or influences), while PPI \cite{grover2016node2vec}, BlogCatalog \cite{grover2016node2vec}, DBLP \cite{yang2015defining}, and Amazon \cite{yang2015defining} are the second type of datasets with position ground-truth (e.g., community memberships). The scales of these datasets also span a wide range from thousands to millions of edges (cf. Supplementary Table~\ref{Tab:Data} for details of datasets).

As summarized in Fig.~\ref{Fig:NC_Eva}a, we use supervised node identity (or position) classification to measure the ability to capture node identities (or positions) on the first (or second) type of datasets with corresponding ground-truth.
The two classification tasks quantitatively evaluate whether the derived embeddings can accurately separate nodes with different labels of positions and identities.
Community detection (also called node position clustering), a straightforward unsupervised position-related task, is employed to evaluate position-capturing capacity on the first type of datasets without position ground-truth.
By extending spectral clustering \cite{von2007tutorial}, we introduce node identity clustering, an unsupervised identity-related task, to evaluate identity-capturing quality on the second type of datasets without identity ground-truth.
The two clustering tasks test whether the given embeddings can successfully group together nodes with similar positions and identities, characterized by community structures and high-order degree patterns, respectively.
In addition, we also measure efficiency by recording the overall embedding inference time (sec) of each method.
Full experiment settings are detailed in Supplementary Section~\ref{App:Set-Task}.

We compare the quality and efficiency of PI-HIST~(R) and (A) over $18$ classic and state-of-the-art baselines with diverse categories (Supplementary Table~\ref{Tab:Meth}).
Representative evaluation results on Europe, Actor, BlogCatalog, and Amazon are visualized in Fig.~\ref{Fig:NC_Eva}b-e. Additional results with more quality metrics and on the remaining datasets are shown in Supplementary Fig.~\ref{Fig:NC_Eva_Apx}.
Given a pair of quality and efficiency metrics, one can map each method to a dot in a $2$D space.
We use red and green dots to denote position and identity embedding baselines. For the rest baselines with purple dots, it remains uncertain which property they can preserve. Efficient or scalable baselines are highlighted by black dot edges.
Because the inference time of all methods spans several orders of magnitude, from approximately $10^{-3}$ to $10^3$ seconds, we use a logarithmic scale for time.
We define that a method encounters the out-of-time (OOT) exception, if it fails to derive feasible embeddings within $10^4$ seconds.
OOM denotes the out-of-memory exception.
The blue arrow in each plot indicates the direction toward a better quality-efficiency trade-off.
For PI-HIST~(R) and (A), their rankings of quality and efficiency among all approaches are denoted by Q and E.

In most cases, one can observe consistent outcomes among different categories of methods. When evaluating identity-capturing ability (i.e., node identity classification and clustering on the first and second types of datasets), identity-preserving baselines (green dots), such as struc2vec \cite{ribeiro2017struc2vec}, generally outperform position embedding approaches (red dots), such as node2vec \cite{grover2016node2vec}. Similarly, in the rest cases measuring position-capturing ability, position embedding methods (red dots) tend to obtain better quality. This consistency supports the validity of our experimental design.

Surprisingly, even with a single training-free FFP from random inputs, at least one PI-HIST variant ensures quality competitive to or even better than the strongest baselines.
Particularly, PI-HIST~(R) and (A) have results consistent with classic position and identity embedding baselines (red and green dots), respectively.
PI-HIST~(R) has significantly better embedding quality than PI-HIST~(A) under the position-related evaluation settings (i.e., community detection and node position classification on the first and second types of datasets) and vice versa in the rest cases.
These outcomes therefore indicate that complementary topological information can emerge from different topology-induced operators even when the input only includes uninformative random features and no parameters are optimized.
They also support the view that informative graph embeddings can arise from carefully chosen topology-induced operators before any learned transformation is introduced.

The training-free design also provides a favorable computational profile. Since PI-HIST requires neither time-consuming feature augmentation nor iterative gradient-based optimization, both variants deliver top-tier efficiency in most cases.
Identity embedding methods (green dots) are generally less efficient than the rest approaches (red and purple dots). None of the identity embedding baselines can be considered efficient (black dot edges). They all suffer from OOT or OOM on large graphs (DBLP and Amazon).
This implies that effectively capturing node identities may sometimes require higher computational overhead than node positions. While less efficient than PI-HIST~(R) and some efficient baselines (e.g., LouvainNE \cite{bhowmick2020louvainne} and SketchNE \cite{xie2023sketchne}) in several cases, PI-HIST~(A) still achieves significantly better efficiency than identity embedding baselines and ranks within the top-$5$ among all methods in terms of efficiency.
In summary, PI-HIST has the potential to obtain a favorable trade-off between the inference quality and efficiency for position and identity embedding.

\begin{figure}[t]
    \centering
    \makebox[\linewidth]{
        \includegraphics[width=0.87\linewidth,trim=0 0 0 0,clip]{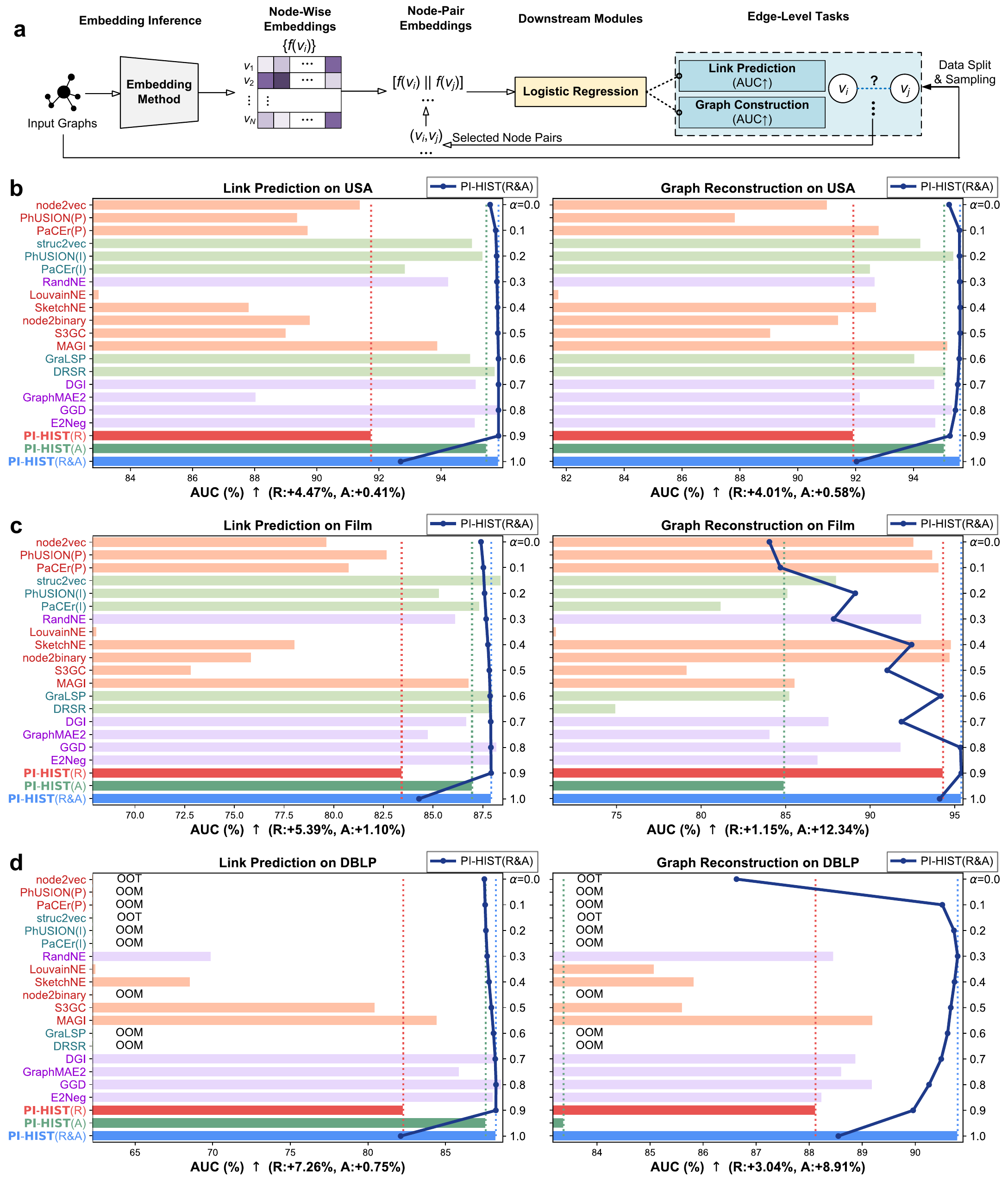}
    }
    \caption{
    \textbf{Representative evaluation results of edge-level tasks}.
    \textbf{a}, Overview of evaluation settings for the $2$ edge-level tasks, link prediction and graph reconstruction, on $8$ real-world graph datasets (Fig.~\ref{Fig:NC_Eva}a).
    \textbf{b}-\textbf{d}, Representative results of link prediction (AUC) and graph reconstruction (AUC) on USA, Film, and DBLP.
    Baselines with red and green labels/bars denote position and identity embedding methods, respectively. It is uncertain for methods with purple labels/bars which property they capture.
    R and A indicate improvements of PI-HIST (R\&A) over PI-HIST~(R) and (A), respectively.
    In some cases, PI-HIST~(R\&A), a simple combination of the two variants, achieves substantial quality improvements (e.g., $>1.0$\%), suggesting that integrating complementary position and identity information can improve the inference quality of some tasks.
    Additional evaluation results on the remaining datasets are detailed in Supplementary Fig.~\ref{Fig:GRLP_Eva_Apx} and Supplementary Section~\ref{App:Res}.
    }
    \label{Fig:GRLP_Eva}
\end{figure}

\subsection{Evaluation of edge-level tasks for combining node positions and identities}
In addition to the position- and identity-related tasks shown in Fig.~\ref{Fig:NC_Eva}a, it may be uncertain which property is preferred for some other tasks.
Based on the verified ability of PI-HIST~(R) and (A) to encode node positions and identities, we explore whether these two properties provide complementary topological information and together improve the inference quality.
We consider two edge-level tasks, link prediction and graph reconstruction, as representative examples (Fig.~\ref{Fig:GRLP_Eva}a). The tasks are evaluated on the same real-world graph datasets used in node-level experiments (Fig.~\ref{Fig:NC_Eva}a), with area under the receiver operating characteristic curve (AUC) used to quantify inference quality.

PI-HIST originally generates node-wise embeddings. Let $f(v_i) \in \mathbb{R}^d$ be the embedding of node $v_i$. To enable $\{ f(v_i) \}$ to support edge-level tasks, we obtain a node-pair embedding $e(v_i, v_j) := [f(v_i) || f(v_j)]$ by concatenating the corresponding node embeddings for each pair $(v_i, v_j)$. $e(v_i, v_j)$ is then fed into a downstream classifier, such as Logistic regression in Fig.~\ref{Fig:GRLP_Eva}a, to determine whether there is an edge or anomaly between $(v_i, v_j)$. This is a standard setting widely adopted by classic embedding approaches for edge-level tasks \cite{yang2020homogeneous}.
Full experiment settings are elaborated in Supplementary Section~\ref{App:Set-Task}.

Let $\{ f_{\rm{R}} (v_i) \}$ and $\{ f_{\rm{A}} (v_j) \}$ be node-wise embeddings from PI-HIST~(R) and (A). We consider a simple combination of the two variants, denoted as PI-HIST~(R\&A). For each node $v_i$, it computes a fused embedding via $f_{\rm{R\&A}}({v_i}): = {\gamma _{{\rm{norm}}}} ( [\alpha f_{\rm{R}}({v_i})||(1 - \alpha ){f_{\rm{A}}}({v_i})] )$, which works best with the downstream module for the two edge-level tasks. This operation concatenates embeddings from two sources and introduces a hyper-parameter $\alpha \in [0, 1]$ to adjust their relative fusion contributions. $\gamma_{\rm{norm}} (\cdot)$ is an optional normalization operation eliminating the magnitude difference between $\{ f_{\rm{R}} (v_i) \}$ and $\{ f_{\rm{A}} (v_j) \}$ (cf. `Method' Section for details).
Following the pipeline in Fig.~\ref{Fig:GRLP_Eva}a, we apply $\{ f_{\rm{R\&A}}({v_i}) \}$ to the two edge-level tasks, where we tune $\alpha \in \{ 0.0, 0.1, \cdots, 1.0\}$ for PI-HIST~(R\&A) and compare its best inference quality with the two PI-HIST variants and other baselines.

Representative evaluation results on USA, Film, and DBLP are depicted in Fig.~\ref{Fig:GRLP_Eva}b-d, with additional outcomes on the remaining datasets shown in Supplementary Fig.~\ref{Fig:GRLP_Eva_Apx}.
In each plot, red labels and bars highlight position-preserving baselines, while green ones mark identity-preserving counterparts.
It is uncertain for those with purple labels and bars which property they can preserve. R and A represent the quality improvements of PI-HIST~(R\&A) over PI-HIST~(R) and (A), respectively.
In addition to the reported metrics of PI-HIST~(R\&A) highlighted by blue bars, we also visualize detailed outcomes with respect to different settings of $\alpha$ using dark-blue curves.

Our results first show that the relative importance of node positions and identities varies across datasets and tasks.
For graph reconstruction on USA and PPI, classic identity embedding approaches (green bars) generally outperform position embedding baselines (red bars). Consistent with these baselines, PI-HIST~(A) (dark-green bar) achieves better graph reconstruction quality than PI-HIST~(R) (dark-red bar) on USA and PPI.
In contrast, position-preserving methods (red bars) tend to outperform identity-preserving counterparts on Film and Actor.
Similarly, PI-HIST~(R) (dark-red bar) also obtains better quality on the same datasets.
Although in most cases of link prediction, identity embedding baselines (green bars) generally attain superior performance than position-preserving approaches (red bars), the two types of methods achieve close quality on some datasets, such as Europe and BlogCatalog.
These observations indicate that link prediction and graph reconstruction are neither typically position- nor identity-dominant tasks.
Moreover, the relative performance of PI-HIST~(R) and (A) broadly follows that of established position- and identity-preserving embedding methods, providing further evidence that the two training-free variants capture distinct topology properties.

More importantly, combining the two types of embeddings generally improves inference quality over using either type alone.
Across all evaluation settings, PI-HIST~(R\&A) (blue bars) delivers better inference quality than both PI-HIST~(R) and (A) (dark-red and dark-green bars), with the gain varying across datasets and tasks.
Surprisingly, such an improvement is substantial in some cases (e.g., more than $1$\% for link prediction and graph reconstruction on Film), despite involving only a weighted fusion of the two variants.
Furthermore, PI-HIST~(R\&A) yields quality competitive to or even better than the strongest baselines.

The dependence on $\alpha$ further supports the complementary nature of the two PI-HIST variants.
For detailed outcomes of PI-HIST~(R\&A) (dark-blue curves), $\alpha \in \{ 0.1, 0.2, \cdots, 0.9\}$, which corresponds to fusing both position- and identity-related information, can ensure significantly better quality than the two corner cases $\alpha \in \{0.0, 1.0\}$ that rely exclusively on one embedding type.
Hence, the improvement does not simply result from selecting the better-performing PI-HIST variant for each dataset. Incorporating both types of embeddings can provide additional information for some inference tasks, including link prediction and graph reconstruction.

Together, these results extend the node-level findings, where the two types of embeddings derived by PI-HIST are not only distinguishable but also complementary.
The two topological operators based on RW- and AW-induced hierarchical structures expose different aspects of graph organization. Combining the resulting position and identity embeddings can improve inference when both aspects are relevant.
This provides empirical support for disentangling graph topology into complementary channels and subsequently recombining them according to the requirements of a downstream task.

\subsection{Case studies of graph-level tasks for combining positions and identities}
We next explore whether the position and identity embeddings derived by PI-HIST can also provide meaningful graph-level representations.
While PI-HIST's node-wise embeddings $\{ f(v_i) \}$ are not specifically designed to represent a whole graph $G$, the graph-level embedding $g (G) \in \mathbb{R}^d$ obtained by applying a standard permutation-invariant pooling operation to $\{ f(v_i) \}$ provides a way to investigate whether PI-HIST can retain graph-level structural patterns.
As a case study, we consider graph superfamily identification following the settings of Milo et al.~\cite{milo2004superfamilies} and Dong et al.~\cite{dong2017structural}.
This task tests whether the derived graph-level features or representations can distinguish graphs from different superfamilies while grouping those with similar structural patterns together.
A superfamily typically corresponds to a specific graph generation mechanism or application domain.

\begin{figure}[t]
    \centering
    \includegraphics[width=1.0\linewidth,trim=0 0 0 0,clip]{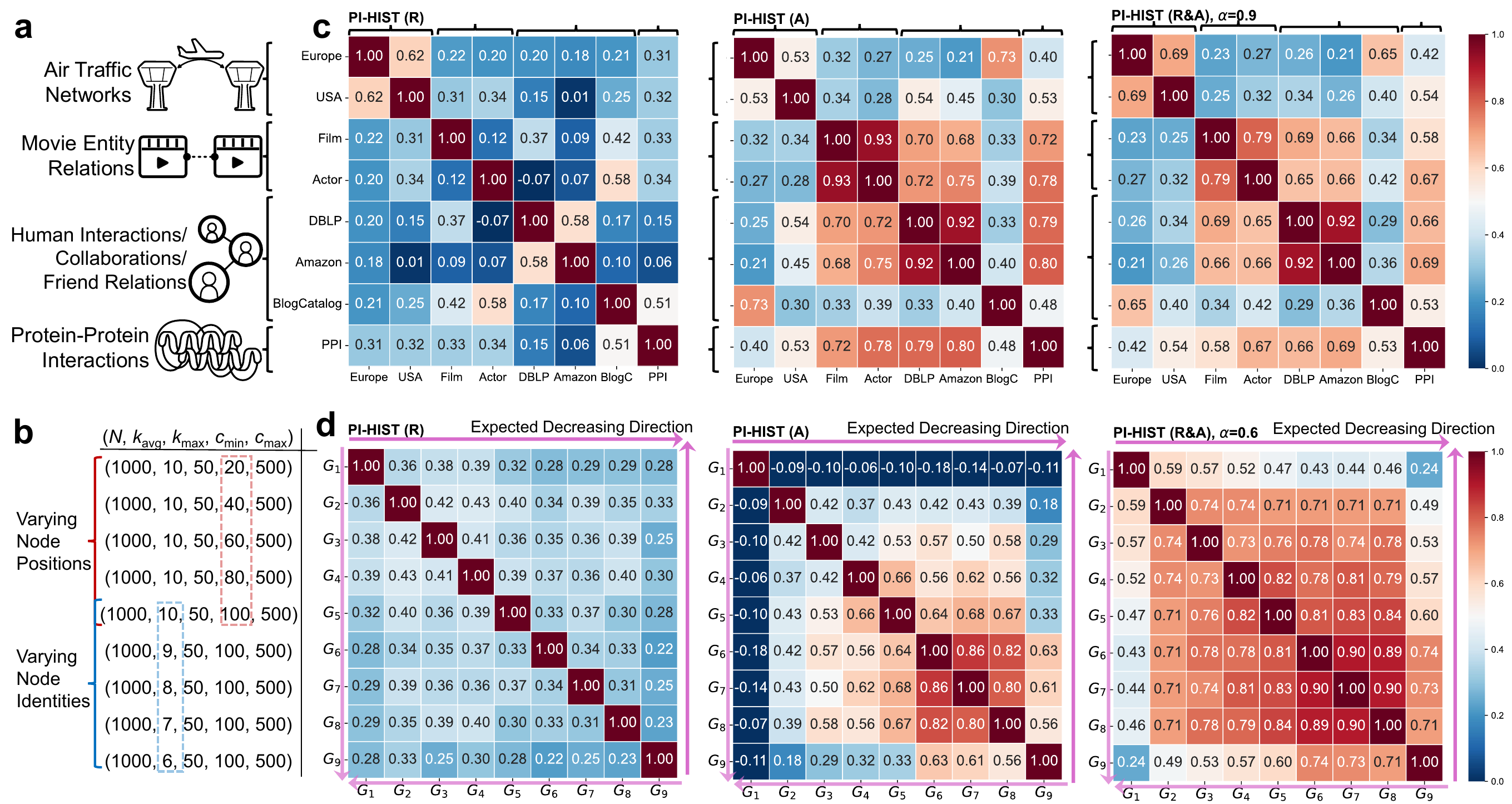}
    \caption{\textbf{Representative case studies for graph-level tasks}.
    \textbf{a}, The superfamily partition of real-world graph datasets (Fig.~\ref{Fig:NC_Eva}a) according to their scenarios.
    \textbf{b}, Parameter settings of the LFR synthetic benchmark, where $N$ denotes the number of nodes; $k_{\rm{avg}}$ and $k_{\rm{max}}$ are average and maximum degrees; $c_{\rm{min}}$ and $c_{\rm{max}}$ are minimum and maximum community sizes. $( G_1, G_2, \cdots, G_5)$ and $( G_5, G_6, \cdots, G_9)$ simulate gradual variations in node positions and identities, respectively.
    \textbf{c}-\textbf{d}, Correlation matrices of PI-HIST~(R), (A), and (R\&A) for superfamily identification on real-world and synthetic graphs. Purple arrows highlight directions along which correlation values are expected to decrease (i.e., away from the main diagonal) for results on synthetic graphs.
    Compared with PI-HIST~(R), PI-HIST~(A) generally exhibits clearer correlation patterns for several expected groupings, although a few cases captured by PI-HIST~(A) are less apparent than PI-HIST~(R).
    PI-HIST~(R\&A), with an appropriate $\alpha \in (0, 1)$, retains the strengths of both variants and shows a more consistent graph-level similarity structure.
    These observations suggest that fusing node positions and identities may provide complementary structural information for graph-level inference.
    Full results of PI-HIST~(R\&A) with the rest settings of $\alpha$ are detailed in Supplementary Fig.~\ref{Fig:GSI_Eva_PI_HIST}.
    Additional results for other baselines (BoD, BoHD, SSP, and CNS) are shown in Supplementary Fig.~\ref{Fig:GSI_Eva_Apx}.
    }
    \label{Fig:GSI_Eva}
\end{figure}

Given node-wise embeddings $\{ f(v_i) \}$ derived from the topology of a graph $G$, we obtain a graph-level embedding $g(G) := {\mathop{\rm MEAN}\nolimits} (\{ f(v_i) \})$ that simply averages representations of all nodes.
The superfamily identification on a set of graphs $\{ G \}$ can be evaluated by computing the correlation coefficient matrix of corresponding graph-level embeddings $\{ g(G) \}$.
Let $\{ f_{\rm{R}} (v_i) \}$ and $\{ f_{\rm{A}} (v_j) \}$ be node-wise embeddings from PI-HIST~(R) and (A). For each node $v_i$, PI-HIST~(R\&A) derives a fused embedding via $f_{\rm{R\&A}} (v_i) := \alpha f_{\rm{R}} (v_i) + (1 - \alpha) f_{\rm{A}} (v_i)$, which works best with the downstream computation of correlation matrix. The fused embedding of a graph $G$ is then obtained by $g_{\rm{R\&A}}(G) := {\gamma _{{\rm{norm}}}}({\mathop{\rm MEAN}\nolimits} (\{ f_{\rm{R\&A}}(v_i) \}))$, with ${\gamma _{{\rm{norm}}}} (\cdot)$ as an optional normalization operation (cf. `Methods' Section for details).

We first consider the $8$ real-world graphs used in node- and edge-level evaluations.
Following Milo et al.~\cite{milo2004superfamilies} and Dong et al.~\cite{dong2017structural}, we partition them into $4$ superfamilies according to their application scenarios (Supplementary Table~\ref{Tab:Data}) and rearrange their layout order based on this partition (Fig.~\ref{Fig:GSI_Eva}a) when visualizing a correlation matrix. 
Under this setting, entries within each main-diagonal block of a visualized matrix, which corresponds to a specific superfamily, are expected to have higher correlations than those across different blocks.

We further construct a controlled synthetic setting based on the Lancichinetti-Fortunato-Radicchi (LFR) benchmark \cite{lancichinetti2008benchmark}.
It uses a set of parameters to generate synthetic graphs and can mimic various topology patterns by adjusting these parameters, allowing us to systematically vary topological aspects related to node positions and identities.
Given that the two topology properties are closely related to community structures and degrees, we simulate gradual variations in node positions and identities by adjusting the minimum community size $c_{\min}$ and average degree $k_{\rm{avg}}$ in the LFR benchmark, respectively.
We generate $9$ synthetic graphs in total (Fig.~\ref{Fig:GSI_Eva}b), with $( G_1, \cdots, G_5)$ and $(G_5, \cdots, G_9)$ corresponding to progressively varying position- and identity-related structure, respectively.
Based on this layout order of graphs $( G_1, \cdots, G_9)$, entries in a visualized correlation matrix are expected to decrease as they move away from the main diagonal (purple arrows in Fig.~\ref{Fig:GSI_Eva}d).
A stronger diagonal trend therefore indicates better preservation of the underlying graph-level similarity.
Full details of experiment settings are elaborated in Supplementary Section~\ref{App:Set-Task}.

Similar to the edge-level evaluation, we tune $\alpha \in \{ 0.0, 0.1, \cdots, 1.0 \}$ for PI-HIST~(R\&A) and compare the best result with those of PI-HIST~(R) and (A).
Representative correlation matrices for the real-world and synthetic graphs are visualized in Fig.~\ref{Fig:GSI_Eva}c-d.
Full results of PI-HIST~(R\&A) under all settings of $\alpha$ are shown in Supplementary Fig.~\ref{Fig:GSI_Eva_PI_HIST}.
Following Dong et al.~\cite{dong2017structural}, we also compare PI-HIST with $4$ baselines that can derive graph-level features or statistics for the downstream correlation matrix computation.
The selected baselines include bag of degrees (BoD), bag of high-order degrees (BoHD), subgraph significance profile (SSP) \cite{milo2004superfamilies}, and common neighbor signature (CNS) \cite{dong2017structural}, with SSP and CNS as sophisticated methods that characterize graphs using statistics of motifs \cite{milo2002network} and structural diversity \cite{ugander2012structural}.
Additional results of these baselines are depicted in Supplementary Fig.~\ref{Fig:GSI_Eva_Apx}.

For case studies on real-world graphs (Fig.~\ref{Fig:GSI_Eva}c), PI-HIST~(A) exhibits relatively higher within-group correlations (dark red) for several expected superfamilies, such as \{Film, Actor\} and \{DBLP, Amazon\}, although some expected similarities are not fully recovered.
For instance, it may fail to group Europe and USA because the correlation between \{Europe, BlogCatalog\} is much larger than that between \{Europe, USA\}.
In contrast, PI-HIST~(R) shows a clearer correlation pattern for \{Europe, USA\}, despite its weaker diagonal correlation trend.
More importantly, PI-HIST~(R\&A) produces a more consistent block structure with the predefined grouping that preserves success cases of both PI-HIST variants while making PPI more distinguishable from other graphs, just as we anticipated.
Among the $4$ baselines (Supplementary Fig.~\ref{Fig:GSI_Eva_Apx}a), SSP also shows several within-group correlation patterns but is less consistent with the tentative grouping than PI-HIST~(R\&A).

The controlled synthetic experiment provides another qualitative illustration of how the two PI-HIST variants encode complementary structural information.
According to our simulation settings, PI-HIST~(R) and (A) are expected to have advantages in grouping $( G_1, \cdots, G_5 )$ and $( G_5, \cdots, G_{9} )$, respectively. Consistent with this expectation, the correlation matrix of PI-HIST~(A) (Fig.~\ref{Fig:GSI_Eva}d) shows that entries near the main-diagonal block induced by $(G_5, \cdots, G_9)$ generally decrease along the purple arrow's direction, although some failure cases remain near the block corresponding to $(G_1, \cdots, G_5)$.
Conversely, in PI-HIST~(R)'s correlation matrix, more success cases can be found near the main-diagonal block formed by $(G_1, \cdots, G_5)$, despite a weaker overall trend.
Furthermore, PI-HIST~(R\&A) leverages the advantages of both PI-HIST~(R) and (A). Almost all entries in the correlation matrix decrease gradually as they move away from the main diagonal.
The superfamily identification on such synthetic graphs is challenging. Only CNS out of the $4$ baselines (Supplementary Fig.~\ref{Fig:GSI_Eva_Apx}b) yields a correlation matrix with an overall trend matching that of PI-HIST~(R\&A). Compared with CNS, the correlation matrix of PI-HIST~(R\&A) exhibits a clearer diagonal correlation trend with a higher maximum correlation and fewer failure cases.

The most visually consistent correlation patterns of PI-HIST~(R\&A) in both the real-world and synthetic settings are generally obtained when $0 < \alpha < 1$, rather than at the two boundary cases $\alpha \in \{0, 1\}$ (Supplementary Fig.~\ref{Fig:GSI_Eva_PI_HIST}).
This observation, together with outcomes of node- and edge-level evaluations, further suggests that disentangling topology structures into position and identity embeddings encodes complementary information regarding graph topology.
An appropriate fusion of the two types of embeddings facilitates improved graph inference.
Overall, our graph-level case studies provide additional qualitative support for viewing position and identity embeddings as complementary channels of graph topology and for using their simple combination when downstream inference depends on both.

\section{Discussion}\label{Sec:Disc}
Graph embedding serves as a fundamental framework of modern graph inference models, exemplified by GNNs \cite{ju2024comprehensive,ju2025survey,zheng2026graph}.
Despite their effectiveness in various inference tasks, it remains difficult to disentangle the contributions of graph topology, feature processing, and parameter learning to the resulting embeddings.
This difficulty arises in part because graph topology contains multiple complementary properties, such as node positions and identities, whose relevance can vary across tasks.
Here, we address this challenge from an operator-level perspective.
PI-HIST probes RW- and AW-induced hierarchical structures with random features through a single training-free FFP. The resulting PI-HIST~(R) and (A) embeddings preferentially capture node positions and identities, respectively.
Across $7$ node-, edge-, and graph-level inference tasks on $8$ real-world and $9$ synthetic graphs, both PI-HIST variants achieve inference quality competitive to or even better than classic and recent embedding methods while generally requiring substantially less computation.
Moreover, combining the two types of embeddings improves the inference quality of several tasks, indicating that the two topology-induced operators expose complementary structural information.
These findings suggest that informative graph embeddings can emerge from carefully chosen topology-induced transformations before extensive gradient-based training is introduced.

The correlation between node position and identity has been discussed in prior research from different views. Srinivasan et al.~\cite{srinivasanequivalence} proved, using invariant theory, that identity and position embeddings share the same relationship as probability distributions and their samples.
Zhu et al.~\cite{zhu2021node} comprehensively evaluated $35$ combinations of node pair proximities, nonlinear filters, and dimension reduction strategies for graph embeddings. They showed that the choice of node proximity metric largely determines whether the derived embeddings capture node positions or identities.
Yan et al.~\cite{yan2024pacer} developed PaCEr, which first learns position embeddings from the RW-with-restart proximity distribution and then converts them to identity embeddings via a simple sorting-based transformation.
From a view of graph signal processing, Qin et al.~\cite{qin2025efficient,qin2025infraredgp} demonstrated that low- and high-pass filtering, which preserve the low- and high-frequency information of graph topology, can help yield position and identity embeddings, respectively.
Our work complements these studies by examining the same distinction through a different mechanism about stochastic-process-induced hierarchical structures, with RW and AW as examples.
While some embedding approaches \cite{jin2020gralsp,yan2024topological,mengirwe} also leverage RW and AW, they use these stochastic processes for feature augmentation or loss construction, which subsequently rely on parameter learning and may easily suffer from low efficiency and scalability.
In PI-HIST, RWs and AWs are not treated as sampling procedures or auxiliary inputs, but as topology-induced operators that expose different topology properties, which generally achieves a more favorable quality-efficiency trade-off.

In this article, we manually tune parameters and other optional settings of PI-HIST (cf. `Methods' section and Supplementary Section~\ref{App:Set-Param} for details), as there is no learning procedure.
The training-free formulation of PI-HIST isolates the information exposed by topology-induced operators from that introduced by feature learning and gradient-based optimization. It also provides a good starting point for investigating how limited parameter learning can improve the adaptability of topology-induced embeddings while preserving their computational efficiency.
As a natural extension of PI-HIST, we can simply let $\{ \varepsilon, \alpha \}$ be learnable parameters. This further allows PI-HIST to assign distinct aggregation weights $\varepsilon$ to different layers.
The selection of $L$, $\psi_{\rm act}(\cdot)$, and $\psi_{\rm norm} (\cdot)$ can also be formulated as a soft combination of all possible settings, with the combination weights learned by the attention mechanism.
Since real-world applications may have various requirements for inference quality and efficiency, a promising future direction is to enable PI-HIST to achieve a tunable trade-off between the two aspects. One possible strategy is to adaptively control the extent of training (e.g., adjusting the number of iterations), where more training may lead to better inference quality yet sacrifice a certain degree of efficiency.

To eliminate the uncertain influence from graph attributes \cite{newman2016structure,peel2017ground,qin2018adaptive,qin2021dual,wang2020gcn}, we consider the unsupervised position and identity embedding, where topology is the only available input source and attributes are unavailable.
This isolation is important for identifying the structural information exposed by RWs and AWs, but practical graph inference often involves additional node and edge attributes.
Integrating attributes into PI-HIST requires another systematic investigation of the underlying correlation between the given attributes and different topology properties.
An advanced extension is to adaptively weight attributes according to their correlation with graph topology.
The subset of attributes aligned with the topology property required by a downstream task can receive higher incorporation weights to improve inference quality, whereas inconsistent attributes should be downweighted to avoid potential quality degradation compared with using topology alone. This extension can also be implemented based on the attention mechanism.

Beyond the RW and AW considered in this article, many other graph processes and algorithms also induce latent hierarchical structures, such as tree structures hidden in graph coarsening \cite{joly2025taxonomy} and the Weisfeiler-Lehman test \cite{huang2021short}, which may capture useful graph properties.
Applying the PI-HIST principle to these topology-induced operators could provide a systematic way to compare the structural information exposed by different graph transformations.
Moreover, PI-HIST adopts the naive Euclidean space to encode embeddings passing through RW- and AW-induced hierarchical structures.
Compared with these Euclidean embeddings, it has been proved that recent advances in hyperbolic representation learning \cite{peng2021hyperbolic,sadat2026hyperbolic} can better preserve hierarchical structures with lower distortion and in lower-dimensional spaces.
Hence, extending PI-HIST to hyperbolic space may further improve inference quality with little additional computational overhead.

\section{Methods}\label{Sec:Meth}

\subsection{Problem statements and preliminaries}
This article considers the unsupervised position and identity embedding on undirected unweighted graphs.
In general, one can describe a graph as $G := (V, E)$, with $V:= \{v_1, v_2, \cdots, v_N\}$ and $E := \{ (v_i, v_j) | v_i, v_j \in V\}$ as the sets of nodes and edges.
To fully study pure topology properties of node position and identity while eliminating the uncertain influence from graph attributes \cite{newman2016structure,peel2017ground,qin2018adaptive,qin2021dual,wang2020gcn}, we assume that graph topology is the only available input and leave the integration of attributes as a significant future direction (`Discussion' section).

\textbf{Node Position}. Let $N_l(v_i)$ be the set of neighbors with exact distance $l$ from node $v_i$. Positions of nodes $\{ v_i \}$ can be characterized by the overlap of their $L$-hop neighbors $\cup _{l = 1}^L{N_l}({v_i})$ relevant to the community structure, where topology linkages are dense within each community but relatively loose between communities.
Nodes $\{ v_i\}$ with high overlaps of $L$-hop neighbors are expected to belong to the same community and thus have similar positions.
For the example graph in Fig.~\ref{Fig:PoC}a, members from the same community (blue or orange nodes) have close positions.

\textbf{Node Identity}.
Node identity describes the structural role that each node $v_i$ plays in the graph topology, which usually corresponds to a specific role or influence level in real-world network systems,  such as the club leader in Fig.~\ref{Fig:PoC}a, and is relevant to node centrality.
Such a role can be characterized by the (\romannumeral1) $L$-hop sub-graph $G_L (v_i)$ rooted at $v_i$ (also called $v_i$'s ego-net) induced by $ \cup _{l = 1}^L{N_l}({v_i})$ or (\romannumeral2) degree statistics ${\delta _{1:L}}: = ({\delta _1}({v_i}),{\delta _2}({v_i}), \cdots ,{\delta _L}({v_i}))$, with $\delta_l (v_i)$ as the degree count or distribution of $N_l(v_i)$.
Nodes $(v_i, v_j)$ with similar ego-nets $(G_L(v_i), G_L(v_j))$ or degree statistics $(\delta_{1:L}(v_i), \delta_{1:L}(v_j))$ are expected to play similar structural roles and thus have similar identities.
For instance, ordinary members in Fig.~\ref{Fig:PoC}a (nodes with red circles) have relatively low high-order degrees and are expected to have similar identities.

\textbf{Position and Identity Embedding}. We formulate graph embedding as a function $f: V \to {\mathbb{R}}^d$ that maps each node $v_i \in V$ to a low-dimensional vector (also called embedding) $f(v_i) \in {\mathbb{R}}^d$ ($d \ll N$), with certain graph properties preserved.
We define $f(v_i)$ as the position (or identity) embedding, if $\{ f(v_i) \}$ preserves node positions (or identities). Namely, nodes $(v_i, v_j)$ with similar positions (or identities) should have close embeddings $(f(v_i), f(v_j))$ in terms of high similarity or short distance in the embedding space.

The derived node-wise embeddings $\{ f(v_i) \}$ can be used as inputs to some downstream modules to support node-level inference tasks, including Logistic regression for node classification and $K$Means for node clustering. To enable these node-wise embeddings to support edge- and graph-level tasks, such as link prediction and graph superfamily identification, we further derive node-pair embeddings $\{ e(v_i, v_j)\}$ and a graph-wise embedding $g(G)$ based on $\{ f(v_i) \}$. Given a pair of nodes $(v_i, v_j)$, we follow a commonly used strategy \cite{yang2020homogeneous} that lets $e (v_i, v_j) := [f(v_i) || f(v_j)]$, with $||$ as the concatenation operation. To get the embedding of a graph $G = (V, E)$, we employ the standard mean pooling operation defined as $g(G) := {\mathop{\rm MEAN}\nolimits} \{ f(v_i) | v_i \in V\}$.

\textbf{RW and AW}. An RW with length $L$ is a sequence of nodes $w = (v^{(0)}, v^{(1)}, \cdots, v^{(L)})$, where $v^{(l)}$ represents the $l$-th node and $(v^{(l)}, v^{(l+1)}) \in E$. Let the index $l$ start from $0$. The RW $w$ can be mapped to an AW $(P_w(v^{(0)}), P_w(v^{(1)}), \cdots, P_w(v^{(L)}))$, with $P_w(v^{(l)})$ denoting the first occurrence index of $v^{(l)}$ in $w$.
As illustrated in Fig.~\ref{Fig:PoC}b, $(v_1, v_2, v_8, v_1)$ is a valid RW of the Zachary's karate club network, with $(0, 1, 2, 0)$ as its AW.

Particularly, nodes within the same community tend to share more neighbors and are more likely to be visited by RWs with a short length $L$, compared with those from different communities.
Micali et al. \cite{micali2016reconstructing} have proved that AW statistics estimated from sufficient RWs originating at each node $v_i$ can be utilized to recover $v_i$'s ego-net (Supplementary Section~\ref{App:Th-Sup}).
The anonymization operation allows RWs from nodes with long distances to share common AWs and thus similar AW statistics. In Fig.~\ref{Fig:PoC}a, $(0, 1, 2, 0)$ is the common AW of RWs $(v_1, v_3, v_8, v_1)$ and $(v_{34}, v_{24}, v_{30}, v_{34})$ sampled from $v_1$ and $v_{34}$, which belong to different communities.
In summary, RW and AW have the potential to encode node positions and identities, respectively.
Motivated by these findings, PI-HIST treats RWs and AWs as topology-induced operators that expose complementary information about node positions and identities.
It derives position and identity embeddings by propagating random features over RW- and AW-induced hierarchical structures (Fig.~\ref{Fig:PoC}b-c) using a single training-free FFP.
Further theoretical support for the effectiveness of PI-HIST is provided in Supplementary Section~\ref{App:Th-Sup}.

\subsection{PI-HIST~(R): position embedding inference}
Given a graph $G = (V, E)$, its topology structure can be described by an adjacency matrix ${\bf{A}} \in \{ 0, 1\}^{N \times N}$, where ${\bf{A}}_{i,j} = {\bf{A}}_{j,i} = 1$ if $(v_i, v_j) \in E$ and ${\bf{A}}_{i,j} = {\bf{A}}_{j,i} = 0$ otherwise; $N := |V|$ is the number of nodes. Let ${\bf{D}} := {\mathop{\rm diag}\nolimits} (k_1, k_2, \cdots, k_N)$ be the degree diagonal matrix, with $k_i = \sum\nolimits_j {{{\bf{A}}_{i,j}}}$ as the degree of node $v_i$.
$({\bf{D}}^{-1}{\bf{A}})$ is defined as the one-step transition matrix of RW on $G$. Accordingly, $({\bf{D}}^{-1}{\bf{A}})^{L}$ describes the $L$-step RW transition probabilities, where $[({\bf{D}}^{-1}{\bf{A}})^{L}]_{i,j}$ encodes the probability of visiting node $v_j$ via an RW originated at $v_i$ and with an exact length $L$.
In essence, $({\bf{D}}^{-1}{\bf{A}})^{L}$ can be interpreted as an $(L+1)$-layer hierarchical structure (Fig.~\ref{Fig:PoC}c), whose layers from bottom to top correspond to the reverse order of RWs. Adjacent layers are connected by weighted links. $({\bf{D}}^{-1}{\bf{A}})_{i,j}$ gives the weight from the $i$-th unit in layer $l$ to the $j$-th unit in layer $(l+1)$ for $0 \le l \le L$.

For simplicity, one can arrange embeddings $\{ f(v_i)\}$ of all nodes $V$ into a matrix form ${\bf{Z}} \in \mathbb{R}^{N \times d}$, where the $i$-th row ${\bf{Z}}_{i,:}$ corresponds to the embedding $f(v_i)$ of node $v_i$.
Given a random input ${\bf{\Theta }} \in \mathbb{R} ^{N \times d} \sim \mathcal{N} (0, 1/d)$, ${({{\bf{D}}^{ - 1}}{\bf{A}})^L}{\bf{\Theta }}$ is equivalent to propagating random features through the weighted hierarchical structure induced by RW with an exact length $L$.
To further leverage characteristics of RWs with length $l < L$, we extend ${({{\bf{D}}^{ - 1}}{\bf{A}})^L}{\bf{\Theta }}$ to a multi-layer architecture using skip connection and normalization from state-of-the-art graph learning approaches.
Let ${\bf{Z}}^{(l-1)} \in \mathbb{R}^{N \times d}$ and ${\bf{Z}}^{(l)} \in \mathbb{R}^{N \times d}$ be the input and output of the $l$-th layer, with ${\bf{Z}}^{(0)} = {\bf{\Theta }}$ denoting the random input.
The $l$-th layer is defined as
\begin{equation}\label{Eq:PI-HIST-R-FFP}
    {\bf{Z}}^{(l)} = {\varphi _{\mathop{\rm norm}\nolimits}} ( \varepsilon {\bf{Z}}^{(l-1)} + (1 - \varepsilon )({{\bf{D}}^{ - 1}} {\bf{A}}){\bf{Z}}^{(l-1)} ),
\end{equation}
where $\varepsilon \in [0, 1)$ is a tunable hyper-parameter ; ${\varphi _{\mathop{\rm norm}\nolimits}} (\cdot)$ represents an optional normalization operation.
Here, we consider the z-score normalization adapted from the batch normalization commonly used in state-of-the-art learning techniques. Given a matrix ${\bf{Q}} \in \mathbb{R}^{N \times d}$, ${\varphi _{\mathop{\rm norm}\nolimits}} ({\bf{Q}})$ normalizes each column of ${\bf{Q}}$ by letting ${{\bf{Q}}_{:,r}} \leftarrow [{{\bf{Q}}_{:,r}} - m({{\bf{Q}}_{:,r}})]/s({{\bf{Q}}_{:,r}})$, where $m(\cdot)$ and $s(\cdot)$ are functions to compute the mean and standard deviation over all entries in a vector or matrix.
It should be noted that the second component $(({\bf{D}}^{-1}{\bf{A}}){\bf{Z}}^{(l)})$ in (\ref{Eq:PI-HIST-R-FFP}) corresponds to the standard GNN mean aggregation \cite{hamilton2017inductive}. It takes the mean of vector representations from each node $v_i$'s neighbors $N_1(v_i)$ to get $v_i$'s output vector in the current layer. The first component ${\bf{Z}}^{(l)}$ introduces a skip connection to integrate outputs from the previous layer. $\varepsilon \in [0, 1)$ adjusts their relative contributions, which can also be interpreted as the extent to which each layer preserves the aggregated features $({\bf{D}}^{-1}{\bf{A}}){\bf{Z}}^{(l)}$.
For RW with a length $L$, ${\bf{Z}}^{(L)}$ is the output obtained by propagating random inputs through the RW-induced hierarchical structure.

Some simplified GNNs \cite{wu2019simplifying} apply nonlinear activation only to their output layers, with all nonlinearities of intermediate layers eliminated.
PI-HIST leverages this simple design, with all learnable parameters removed, to perform nonlinear activation and normalization on ${\bf{Z}}^{(L)}$.
Namely, PI-HIST~(R) derives its embeddings $\{ f(v_i)\}$ arranged in the matrix form ${\bf{Z}} \in \mathbb{R}^{N \times d}$ via
\begin{equation}\label{Eq:PI-HIST-R-NA}
    {\bf{Z}} = {\psi _{\mathop{\rm norm}\nolimits}}({\psi _{\mathop{\rm act}\nolimits}}({{\bf{Z}}^{(L)}})),
\end{equation}
where ${\psi _{\mathop{\rm act}\nolimits}} (\cdot)$ denotes an optional nonlinear activation function (e.g., tanh, sigmoid, and ReLU); ${\psi _{\mathop{\rm norm}\nolimits}} (\cdot)$ is another optional normalization operation.
Hence, in addition to the interpretation of a weighted hierarchical structure, PI-HIST~(R)'s embedding derivation is also equivalent to propagating random features through an untrained GNN with standard mean aggregation, skip connections, and nonlinear activation.

We summarize PI-HIST~(R)'s overall embedding inference procedure in Supplementary Algorithm~\ref{Alg:PI-HIST-P}. Let $M := |E|$ be the number of edges in the input graph $G$. As real-world large graphs are usually sparse, one can assume that $L < d \ll N < M$.
The complexity of propagating random features through one layer is no more than $O ((M+N)d)$ using the efficient sparse-dense matrix multiplication. Hence, the overall complexity of PI-HIST~(R) is within $O((M+N)dL) \approx O(N+M)$.

\subsection{PI-HIST~(A): identity embedding inference}

Different from RW, there are currently no closed-form solutions to directly obtain AW statistics.
We develop an efficient algorithm to estimate AW statistics via RW sampling and extract AW-induced hierarchical structure from such statistics. This algorithm contains six steps of RW sampling, AW mapping, AW indexing, AW counting, statistic gathering, and hierarchical structure extraction.

We first sample $n$ RWs starting from each node $v_i$ with length $L$ (RW sampling) and map them to corresponding AWs (AW mapping).
Given a walk length $L$, the union of all candidate AWs can be encoded by a tree $\mathcal{T}_L$ with height $(L+1)$, where each path from the root to a leaf represents a unique AW (cf. Supplementary Fig.~\ref{Fig:AW_EST}a for examples).
For efficient estimation, we map each observed AW $\omega$ to an index $I[\omega]$ based on $\mathcal{T}_L$'s encoding using dynamic programming (AW indexing).
One can estimate the frequency of each AW $\omega$ for node $v_i$ by counting instances of the node-index pair $(v_i, I[\omega])$ derived from RWs originating at $v_i$ (AW counting).
For all nodes in a large-scale graph, this procedure of sampling and counting is usually performed in multiple batches, with each run processing only one batch.
An additional step is required to combine results from all runs and obtain the complete AW statistics (statistics gathering).
Finally, we extract a hierarchical structure $\mathcal{H} (v_i)$ for each node $v_i$ from its estimated AW statistics (hierarchical structural extraction).
Most steps of the aforementioned algorithm are implemented with GPU-friendly operations and thus benefit from GPU acceleration, as is the case for most modern graph inference models, exemplified by GNNs.
Full technical details about this algorithm are elaborated in Supplementary Section~\ref{App:Meth-PI-HIST-A}.

For an arbitrary AW $\omega$ with length $L$, such as $\omega \in \{ (0,0,0),(0,0,1), \cdots ,(0,1,2)\}$ for $L=2$, the $l$-th-to-last element in $\omega$ can only take a value in $\{ 0, 1, \cdots, L+1-l\}$.
Hence, the hierarchical structure $\mathcal{H} (v_i)$ of node $v_i$ has $(L+1)$ layers, with the $l$-th layer (from bottom to top) containing units $\{ 0, \cdots, L+1-l\}$. 
There are also weighted links between adjacent layers, which are derived from the estimated AW statistics of $v_i$.
For simplicity, we denote the $s$-th unit in the $l$-th layer of $\mathcal{H} (v_i)$ as $h(t, l; v_i)$ (cf. Supplementary Fig.~\ref{Fig:AW_EST}c for examples of $\mathcal{H} (v_i)$).
Similar to PI-HIST~(R), PI-HIST~(A) derives the embedding $f(v_i)$ of $v_i$ via just one FFP through $\mathcal{H} (v_i)$ with random features as the input.
We extend this operation to a group of nodes using tensor computations, which enables the simultaneous FFP for all nodes.
A tensor ${\bf{B}}^{(l)} \in \mathbb{R}^{N \times L \times L}$ is introduced to describe weighted links between the $l$-th and $(l+1)$-th layers, where ${\bf{B}}^{(l)}_{i,t,p}$ represents the weight between $h(t, l; v_i)$ and $h(p, l+1; v_i)$. ${\bf{B}}^{(l)}_{i,t,p} = 0$ means that there is no link between corresponding units.
Let ${\bf{U}}^{(l-1)} \in \mathbb{R}^{N \times L \times d}$ and ${\bf{U}}^{(l)} \in \mathbb{R}^{N \times L \times d}$ be the (tensor) input and output of the $l$-th layer, with ${\bf{U}}^{(0)} = {\bf{\Theta}} \in {\mathbb{R}}^{N \times L \times d} \sim \mathcal{N} (0, 1/d)$.
The FFP of the $l$-th layer is defined as
\begin{equation}\label{Eq:PI-HIST-A-FPP}
    {{\bf{U}}^{(l)}} = {\varphi _{\mathop{\rm norm}\nolimits}} (\varepsilon {{\bf{U}}^{(l-1)}} + (1 - \varepsilon ){{\bf{B}}^{(l)}}{{\bf{U}}^{(l-1)}}),
\end{equation}
where ${{\bf{B}}^{(l)}}{{\bf{U}}^{(l-1)}}$ applies the tensor multiplication with ${({{\bf{B}}^{(l)}}{{\bf{U}}^{(l-1)}})_{i,l,r}}: = \sum\nolimits_{t = 1}^L {{\bf{B}}_{i,l,t}^{(l)}{\bf{U}}_{i,t,r}^{(l-1)}}$;
$\varepsilon \in [0, 1)$ is a hyper-parameter balancing contributions between the aggregated features and residual features from the previous layer; ${\varphi _{\mathop{\rm norm}\nolimits}} (\cdot)$ is an optional normalization operation.
We extend ${\varphi _{\mathop{\rm norm}\nolimits}} (\cdot)$ in (\ref{Eq:PI-HIST-R-FFP}) to the z-score normalization on a tensor ${\bf{Q}} \in \mathbb{R}^{N \times L \times d}$ over all nodes.
As there are $(L+2-l)$ units in the $l$-th layer of an AW-induced hierarchical structure (cf. Supplementary Fig.~\ref{Fig:AW_EST}c for examples), $\varphi_{\rm norm} ({\bf{Q}})$ conducts the normalization by setting ${{\bf{Q}}_{:,t,r}} \leftarrow [{{\bf{Q}}_{:,t,r}} - m({{\bf{Q}}_{:,t,r}})]/s({{\bf{Q}}_{:,t,r}})$ for $1 \le t \le (L+2-l)$ and $1 \le r \le d$.

Note that the last layer of $\mathcal{H} (v_i)$ only contains one unit. We can simply obtain node-wise representations from ${\bf{U}}^{(L)}_{:,1,:} \in \mathbb{R}^{N \times d}$. These embeddings are further combined with the same (optional) nonlinear activation and normalization in PI-HIST~(R).
Therefore, PI-HIST~(A) finally derive its embeddings $\{ f (v_i) \}$ for all nodes arranged in a matrix ${\bf{Z}} \in \mathbb{R}^{N \times d}$ via
\begin{equation}
    {\bf{Z}} = {\psi _{\mathop{\rm norm}\nolimits}}({\psi _{\mathop{\rm act}\nolimits}} ({\bf{U}}_{:,1,:}^{(L)}) ),
\end{equation}
where ${\psi _{\mathop{\rm act}\nolimits}} (\cdot)$ and ${\psi _{\mathop{\rm norm}\nolimits}} (\cdot)$ share the same definitions as those in (\ref{Eq:PI-HIST-R-NA}).

The overall embedding inference procedure of PI-HIST~(A) is summarized in Supplementary Algorithm~\ref{Alg:PI-HIST-A}.
Using the proposed algorithm for AW statistic estimation and hierarchical structure extraction, we can obtain $\{ {\bf{B}}^{(0)}, \cdots, {\bf{B}}^{L}\}$ that describe $\{ \mathcal{H}(v_i)|v_i \in V\}$ within the complexity of $O(NLn + N{\tilde \eta})$, where $\tilde \eta$ is the number observed AWs during estimation. Further details about the complexity analysis are provided in Supplementary Section~\ref{App:Meth-PI-HIST-A}.
Since there are $(L+1)$ layers in $\mathcal{H}(v_i)$ and at most $(L+1)$ units in each layer, the complexity of one FFP through $\{ \mathcal{H}(v_i) \}$ is $O (NL^2d)$.
For large sparse graphs, we can assume that $L < d \ll N, \tilde \eta$.
The overall complexity of PI-HIST~(A) is $O(N(L^2d + Ln + \tilde \eta)) \approx O(N(1 + n + \tilde \eta))$.

\subsection{Combination of PI-HIST~(R) and PI-HIST~(A)}
In addition to using PI-HIST to disentangle graph topology into position and identity embeddings, we also explore whether combining these two types of embeddings can improve graph inference.
Let $f_R(v_i)$ and $f_A(v_i)$ denote embeddings of node $v_i$ produced by PI-HIST~(R) and (A), respectively.
To construct the fused embedding $f_{\rm R\&A} (v_i)$ of PI-HIST~(R\&A) for $v_i$, two example combination strategies are considered, including the weighted sum ${f_{{\rm{R\& A}}}}({v_i}): = \alpha {f_R}({v_i}) + (1 - \alpha ){f_{\rm{A}}}({v_i})$ and weighted concatenation ${f_{{\rm{R\& A}}}}({v_i}): = [\alpha {f_R}({v_i})||(1 - \alpha ){f_{\rm{A}}}({v_i})]$, with $\alpha$ as a tunable hyper-parameter adjusting relative contributions between the two types of embeddings.
Our experiments suggest that different downstream tasks favor different combination strategies. For instance, weighted concatenation and sum work better for the two edge-level tasks and graph-level superfamily identification, respectively.

\subsection{Model Configurations and Parameter Settings}

The two PI-HIST variants and their combination involve several hyper-parameters $\{ \varepsilon, L, \alpha \}$ and optional settings $\{ \varphi_{\rm act}(\cdot), \varphi_{\rm norm} (\cdot), \gamma_{\rm norm} (\cdot) \}$ that can be selected by using the same tuning strategy as other unsupervised embedding approaches, including baselines in our experiments (cf. Supplementary Section~\ref{App:Set-Param} for details of this strategy).
As outlined in the `Discussion' section, we believe that these settings can be automatically optimized by simply extending PI-HIST to a limited learning setting, which potentially improves the inference quality while introducing additional computational overhead.
We leave this extension as a significant future direction, treating PI-HIST as a good start.

In our experiments, $L$ is selected from $\{ 5, 6, \cdots, 20\}$ and $\{ 5, 6, \cdots, 9\}$ for PI-HIST~(R) and (A), respectively.
 $\varepsilon$ is tuned over $\{ 0.0, 0.1, \cdots, 0.9 \}$ for both variants.
We comprehensively investigate the effects of different activation and normalization choices.
Specifically, we set $\psi_{\rm act} (\cdot)$ to $\emptyset$ (i.e., no activation), ReLU, tanh, sigmoid, or exp. Moreover, $\psi_{\rm norm} (\cdot)$ is set to $\emptyset$ (i.e., no normalization), column-wise z-score normalization, or row-wise $l_2$-normalization on embeddings arranged in a matrix ${\bf{Z}} \in \mathbb{R}^{N \times d}$. For column-wise z-score normalization, $\psi_{\rm norm} (\cdot)$ adopts the same definition as $\varphi_{\rm norm} (\cdot)$ in (\ref{Eq:PI-HIST-R-FFP}).
For each row ${\bf{Q}}_{i, :}$ of a matrix ${\bf{Q}}$, row-wise $l_2$-normalization updates it as ${\bf{Q}}_{i, :} \leftarrow {\bf{Q}}_{i, :}/ |{\bf{Q}}_{i, :}|_2$.
Ablation studies (Supplementary Fig.~\ref{Fig:PI-HIST-R-ppi-p}-\ref{Fig:PI-HIST-A-ppi-p}) show that applying $\psi_{\rm act} (\cdot)$ and $\psi_{\rm norm} (\cdot)$ improves the embedding quality of PI-HIST, where different tasks and datasets favor distinct settings.

For PI-HIST~(R\&A), we first determine the optimal settings of $\{ \varepsilon, L, \psi_{\rm act} (\cdot), \psi_{\rm norm} (\cdot) \}$ inherited from PI-HIST~(R) and (A), after which $\alpha$ and $\gamma_{\rm norm} (\cdot)$ are selected.
Specifically, we tune $\alpha \in \{ 0.0, 0.1, \cdots, 1.0\}$ and set $\gamma_{\rm norm} (\cdot)$ to $\emptyset$ (i.e., no normalization), column-wise z-score normalization, row-wise $l_2$-normalization, or row-wise z-score normalization.
Similar to the column-wise z-score normalization, such as $\varphi_{\rm norm} (\cdot)$ in (\ref{Eq:PI-HIST-R-FFP}), the row-wise z-score version applies the same operation independently to each row ${\bf{Q}}_{i,:}$ of a matrix ${\bf{Q}}$ via ${{\bf{Q}}_{i,:}} \leftarrow [{{\bf{Q}}_{i,:}} - m({{\bf{Q}}_{i,:}})]/s({{\bf{Q}}_{i,:}})$.
Detailed settings for all experimental evaluation cases are summarized in Supplementary Tables~\ref{Tab:Param} and \ref{Tab:Param-RA}.

\backmatter

\bmhead{Data Availability} We provide the sources of all the real-world graph datasets in Supplementary Section~\ref{App:Set-Data}. The preprocessed datasets of both real-world and synthetic graphs used in this research are also available via GitHub at \url{https://github.com/KuroginQin/PI_HIST}.

\bmhead{Code Availability} The code of this research for experimental evaluation and result visualization is available via GitHub at \url{https://github.com/KuroginQin/PI_HIST}. We provide the sources of official implementations of baseline methods in Supplementary Section~\ref{App:Set-Base}.

\bmhead{Acknowledgements}
This research was supported by Major Key Project of PCL (PCL2025A03), National Natural Science Foundation of China (62388101, T2588101), Beijing Natural Science Foundation (Z230001), National Science and Technology Major Project (2024ZD01NL00103), and Interdisciplinary Frontier Research Project of PCL (2025QYB015).




\begin{appendices}

\clearpage
\begin{figure}[p!]
    \centering
    \includegraphics[width=0.95\linewidth,trim=0 0 0 0,clip]{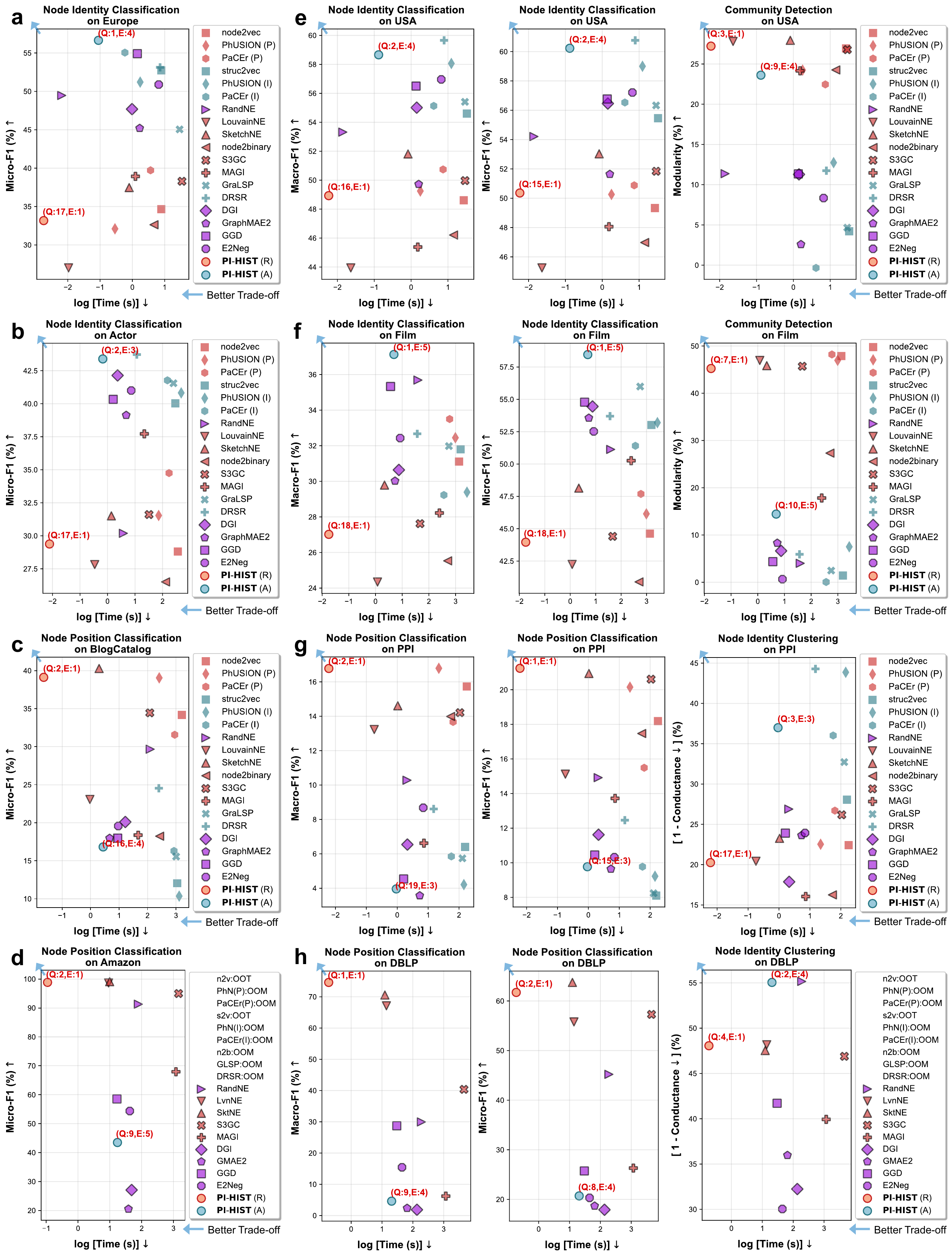}
    \caption{
	\textbf{Additional evaluation results of node-level tasks}.
    \textbf{a}-\textbf{d}, Additional evaluations on Europe, Actor, BlogCatalog, and Amazon using node identity and position classification (micro~F1$\uparrow$) to assess identity and position embeddings, respectively.
    \textbf{e}-\textbf{f}, Evaluations on USA and Film with position ground-truth using node identity classification (macro~F1$\uparrow$ and micro~F1$\uparrow$) and community detection (modularity$\uparrow$) to assess embedding quality.
    \textbf{g}-\textbf{h}, Evaluations on PPI and DBLP with position ground-truth using node position classification (macro~F1$\uparrow$ and micro~F1$\uparrow$) and node identity clustering (conductance$\downarrow$) for quality assessment.
    The logarithm of inference time$\downarrow$ (sec) measures efficiency of all cases. Q and E denote the quality and efficiency rankings of PI-HIST, respectively. Blue arrows indicate directions toward improved quality–efficiency trade-offs.
    }
    \label{Fig:NC_Eva_Apx}
\end{figure}
\clearpage

\begin{figure}[p!]
    \centering
    \includegraphics[width=0.85\linewidth,trim=0 0 0 0,clip]{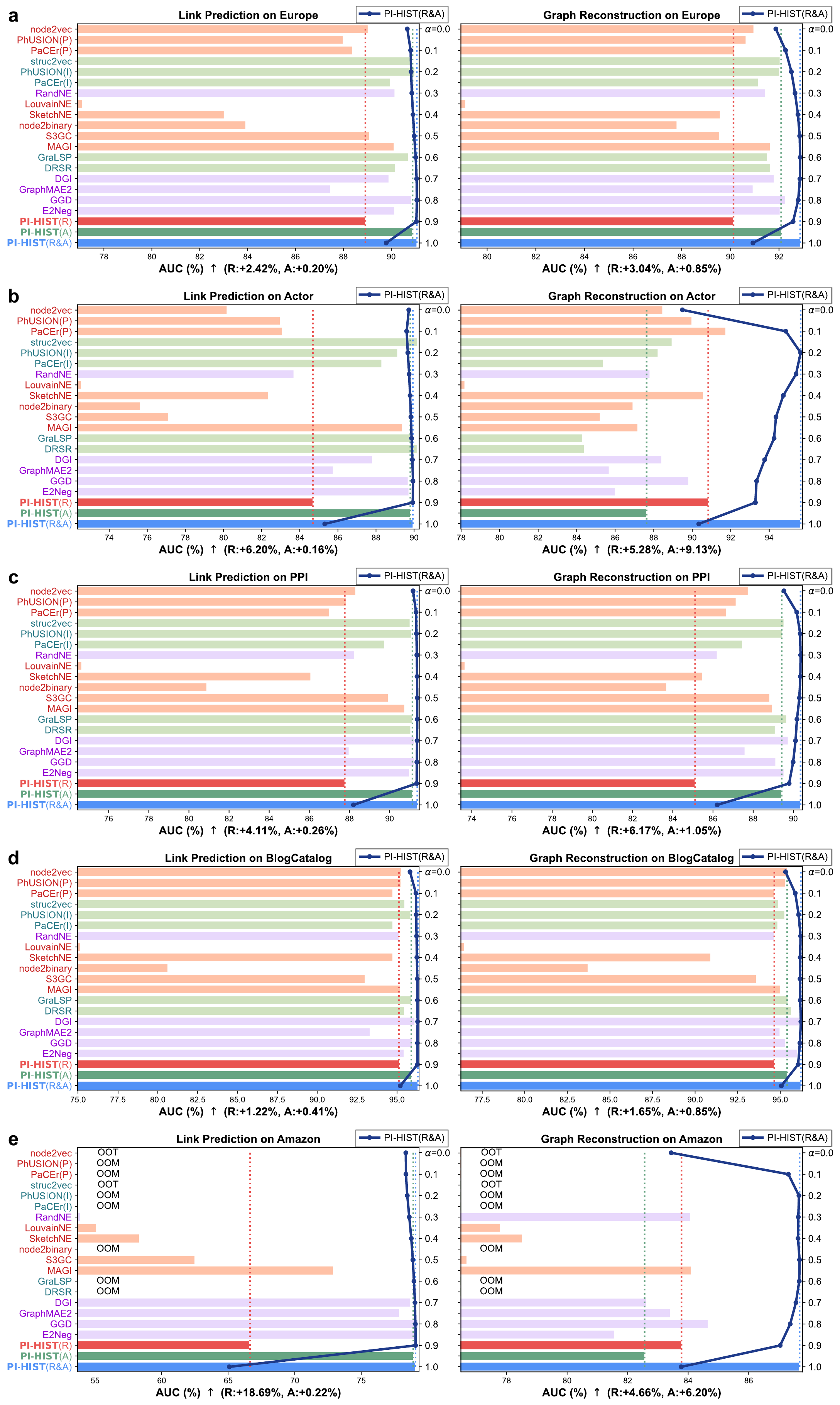}
    \caption{
	\textbf{Additional evaluation results of edge-level tasks}.
    \textbf{a}-\textbf{e}, Results of link prediction (AUC$\uparrow$) and graph reconstruction (AUC$\uparrow$) on Europe, Actor, PPI, BlogCatalog, and Amazon.
    R and A indicate improvements of PI-HIST (R\&A) over PI-HIST (R) and PI-HIST (A), respectively.
    Red and green labels/bars denote position and identity embedding approaches, respectively. It is uncertain for baselines with purple labels/bars which property they capture.
    }
    \label{Fig:GRLP_Eva_Apx}
\end{figure}
\clearpage

\begin{figure}[p!]
    \centering
    \includegraphics[width=1.0\linewidth,trim=0 0 0 0,clip]{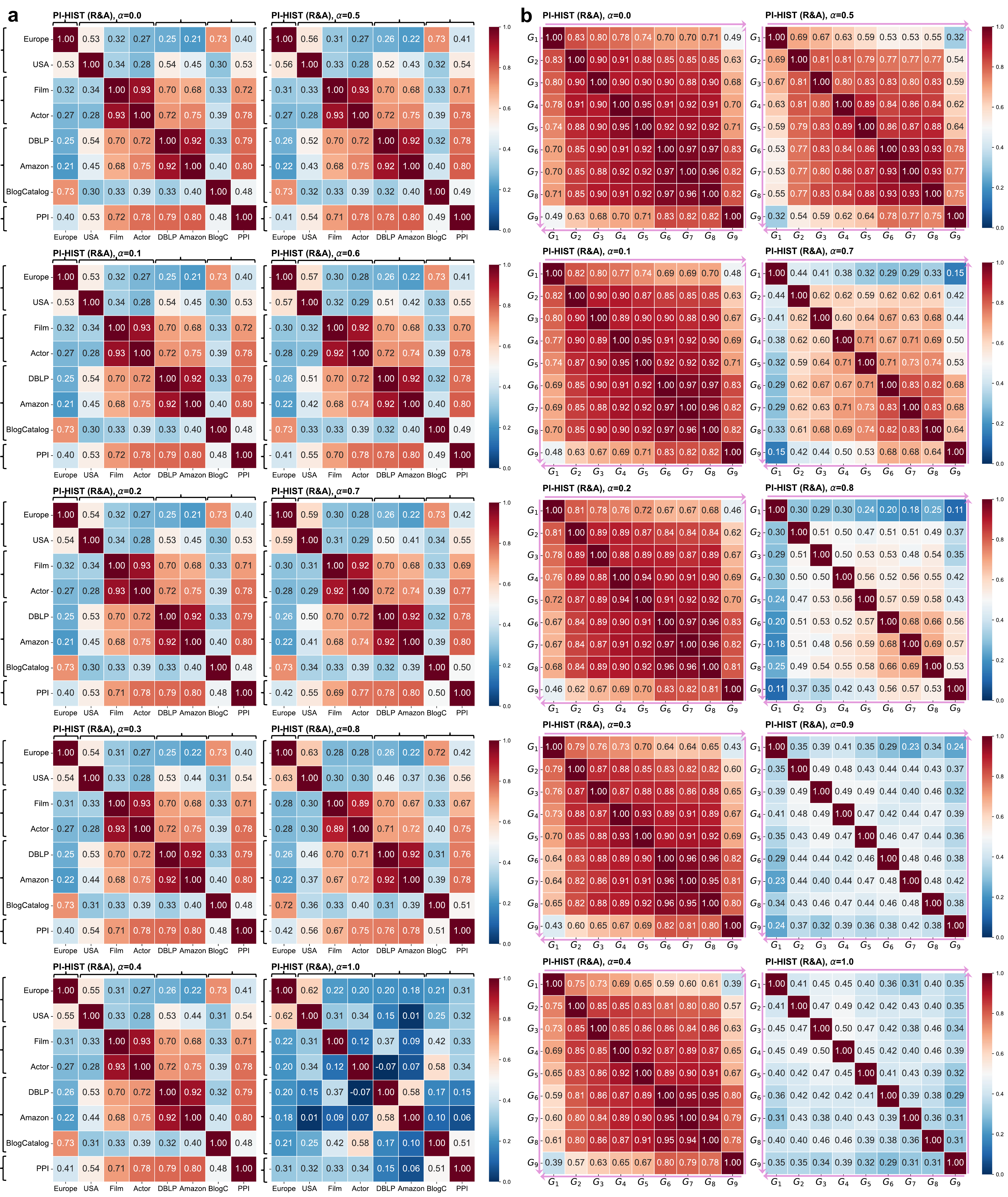}
    \caption{
	\textbf{Full case studies of PI-HIST~(R\&A) for graph-level tasks}.
    \textbf{a}, Graph superfamily identification on real-world graphs with PI-HIST~(R\&A) for $\alpha \in \{ 0.0, \cdots, 0.8, 1.0\}$ complementing Fig.~\ref{Fig:GSI_Eva}a ($\alpha=0.9$).
    \textbf{b}, Graph superfamily identification on synthetic graphs with PI-HIST~(R\&A) for $\alpha \in \{ 0.0, \cdots, 0.5, 0.7, \cdots, 1.0\}$ complementing Fig.~\ref{Fig:GSI_Eva}b ($\alpha=0.6$).
    }
    \label{Fig:GSI_Eva_PI_HIST}
\end{figure}
\clearpage

\begin{figure}[p!]
    \centering
    \includegraphics[width=1.0\linewidth,trim=0 0 0 0,clip]{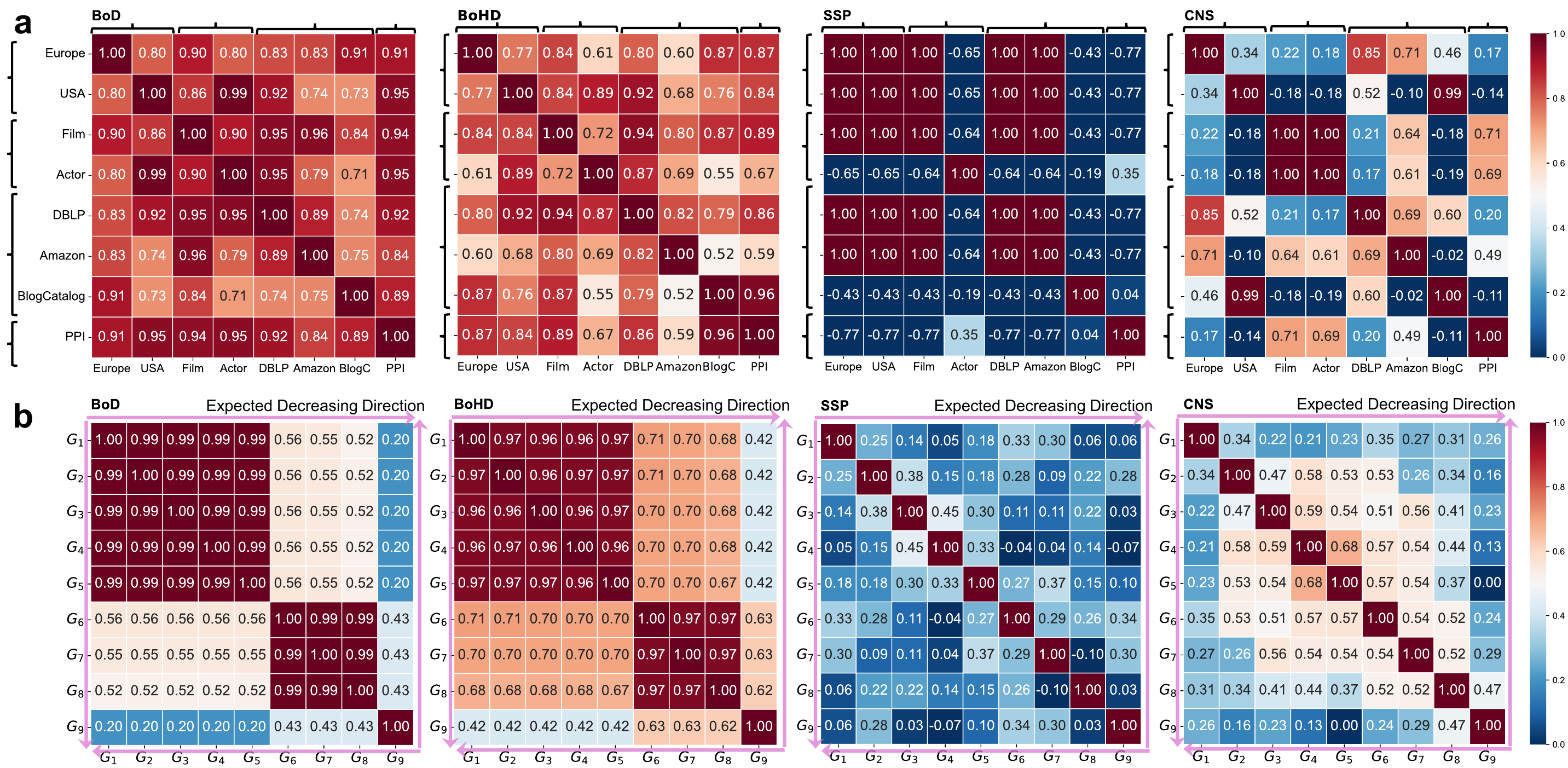}
    \caption{
	\textbf{Additional case studies of baselines for graph-level tasks}.
    \textbf{a}-\textbf{b}, Correlation matrices of BoD, BoHD, SSP, and CNS for superfamily identification on real-world and synthetic graphs. Purple arrows indicate directions along which correlation values are expected to decrease (i.e., away from the main diagonal) for results on synthetic graphs.
    }
    \label{Fig:GSI_Eva_Apx}
\end{figure}
\clearpage

\begin{figure}[p!]
    \centering
    \includegraphics[width=1.0\linewidth,trim=0 0 0 0,clip]{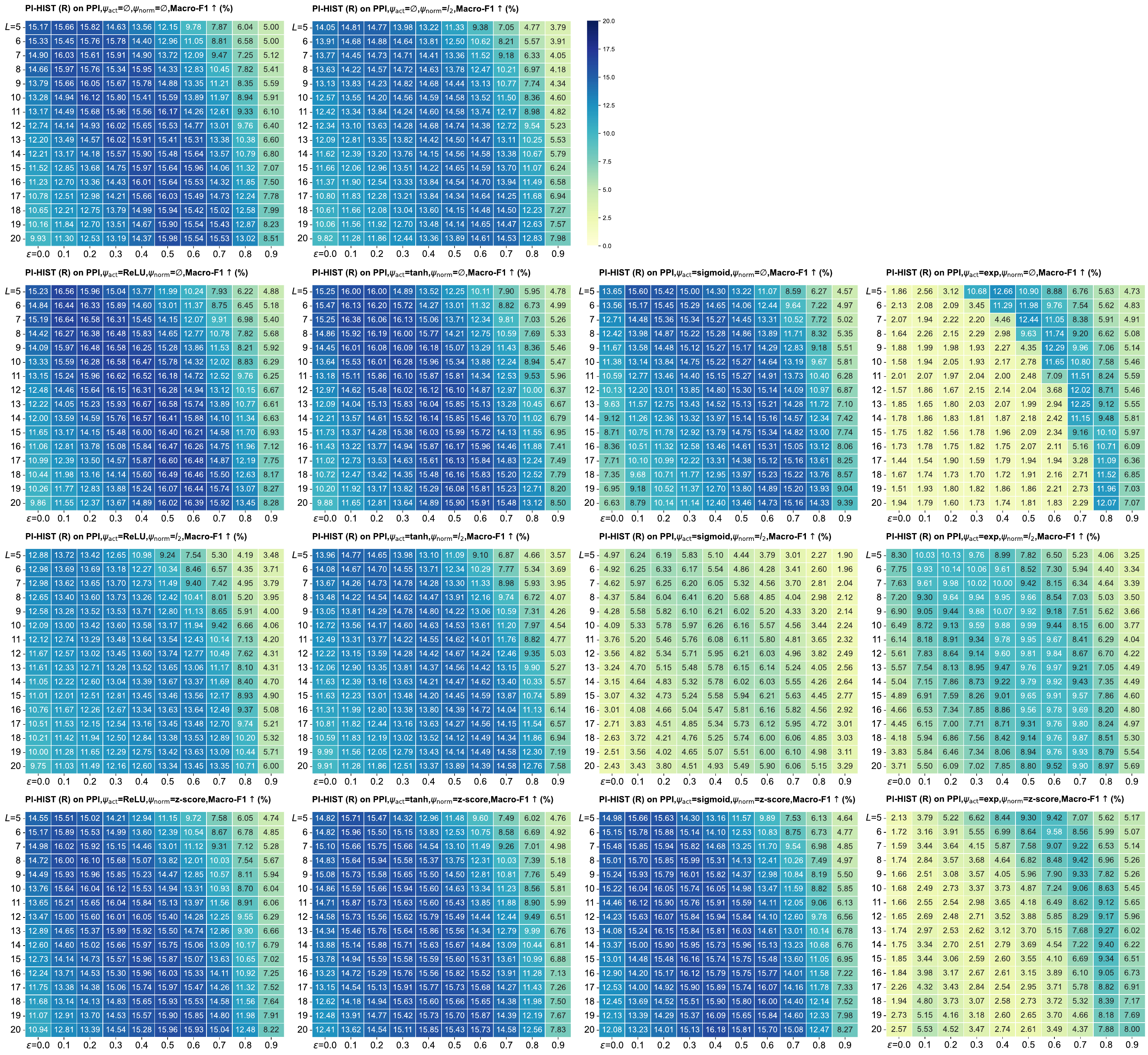}
    \caption{
	\textbf{Ablation study and parameter analysis of PI-HIST~(R) on PPI}. Representative results of node position classification (macro F1$\uparrow$, \%) on the validation set, with $L \in \{ 5, 6, \cdots, 20\}$ ($y$-axis) and $\varepsilon \in \{ 0.0, 0.1, \cdots, 0.9\}$ ($x$-axis). The optional activation function $\psi_{\rm act} (\cdot)$ is set to $\emptyset$ (no activation), ReLU, sigmoid, tanh, or exp. The optional normalization $\psi_{\rm norm} (\cdot)$ is set to $\emptyset$ (no activation), row-wise $l_2$-normalization, and column-wise z-score normalization.
    }
    \label{Fig:PI-HIST-R-ppi-p}
\end{figure}
\clearpage

\begin{figure}[p!]
    \centering
    \includegraphics[width=1.0\linewidth,trim=0 0 0 0,clip]{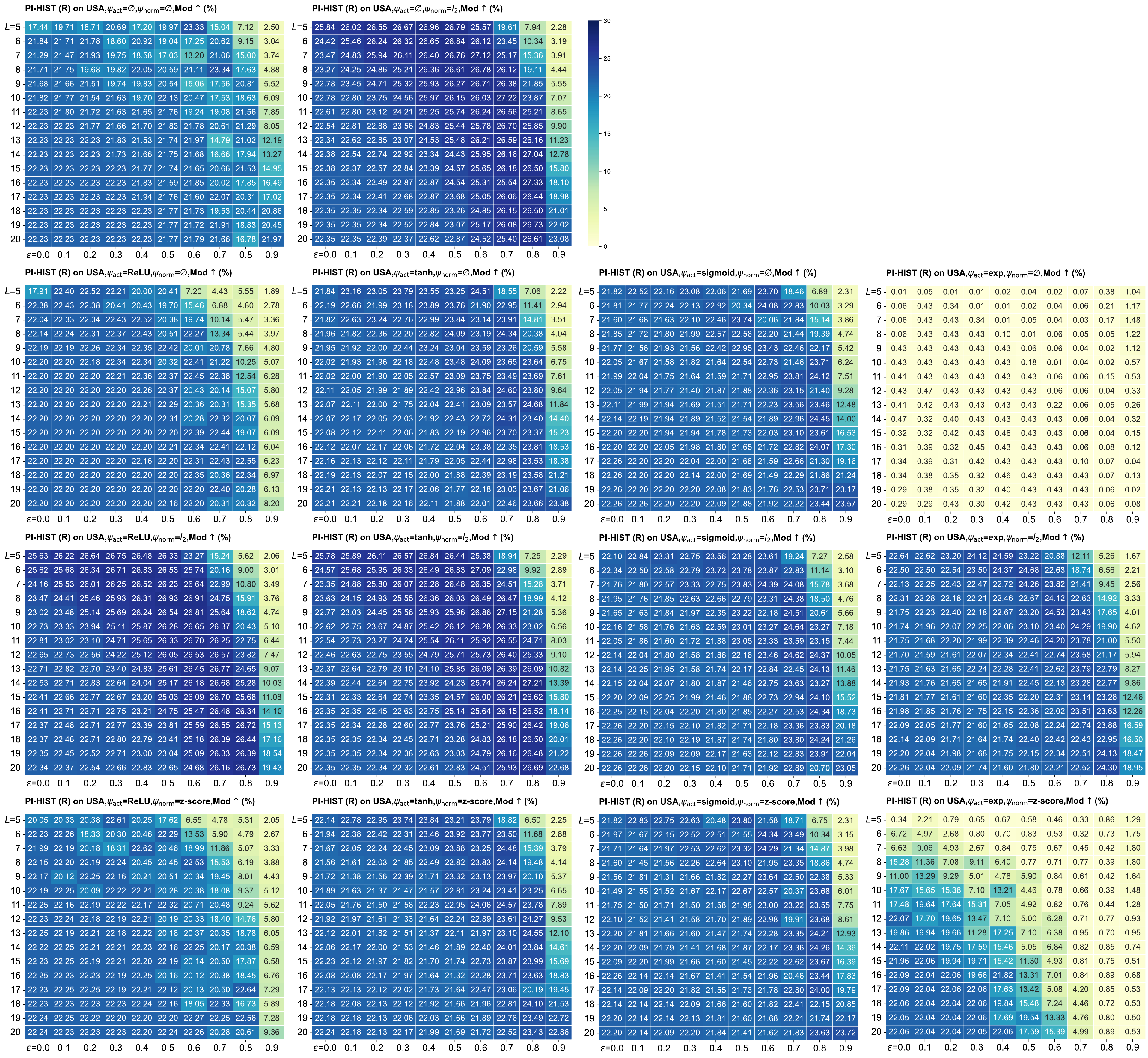}
    \caption{
	\textbf{Ablation study and parameter analysis of PI-HIST~(R) on USA}. Representative results of community detection (modularity$\uparrow$, \%), with $L \in \{ 5, 6, \cdots, 20\}$ ($y$-axis) and $\varepsilon \in \{ 0.0, 0.1, \cdots, 0.9\}$ ($x$-axis). The optional activation $\psi_{\rm act} (\cdot)$ is set to $\emptyset$ (no activation), ReLU, sigmoid, tanh, or exp. The optional normalization $\psi_{\rm norm} (\cdot)$ is set to $\emptyset$ (no activation), row-wise $l_2$-normalization, and column-wise z-score normalization.
    }
    \label{Fig:PI-HIST-R-usa-p}
\end{figure}
\clearpage

\begin{figure}[p!]
    \centering
    \includegraphics[width=1.0\linewidth,trim=0 0 0 0,clip]{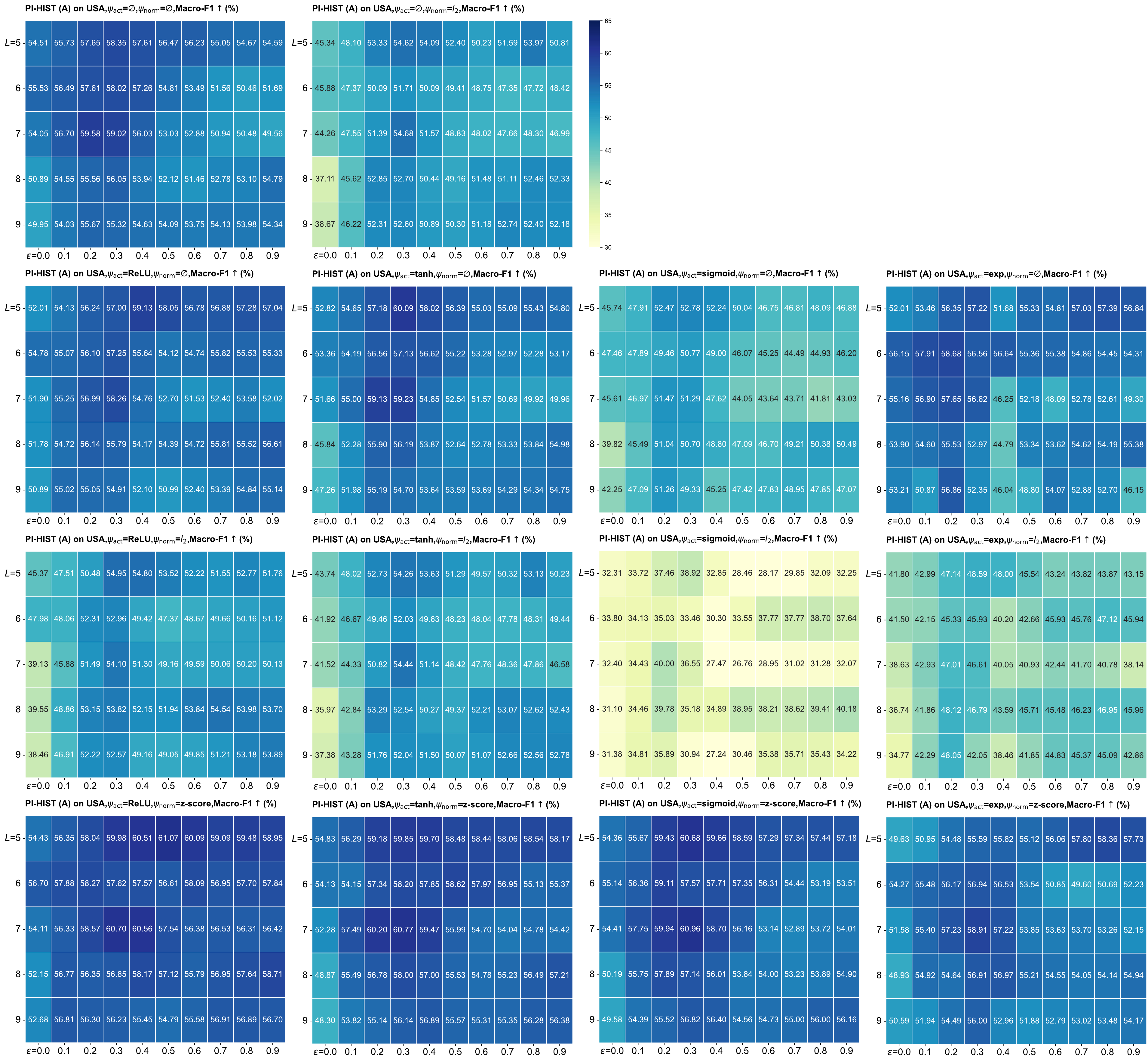}
    \caption{
	\textbf{Ablation study and parameter analysis of PI-HIST~(A) on USA}. Representative results of node identity classification (macro F1$\uparrow$, \%) on the validation set, with $L \in \{ 5, 6, \cdots, 9\}$ ($y$-axis) and $\varepsilon \in \{ 0.0, 0.1, \cdots, 0.9\}$ ($x$-axis). The optional activation $\psi_{\rm act} (\cdot)$ is set to $\emptyset$ (no activation), ReLU, sigmoid, tanh, or exp. The optional normalization $\psi_{\rm norm} (\cdot)$ is set to $\emptyset$ (no activation), row-wise $l_2$-normalization, and column-wise z-score normalization.
    }
    \label{Fig:PI-HIST-A-usa-p}
\end{figure}
\clearpage

\begin{figure}[p!]
    \centering
    \includegraphics[width=1.0\linewidth,trim=0 0 0 0,clip]{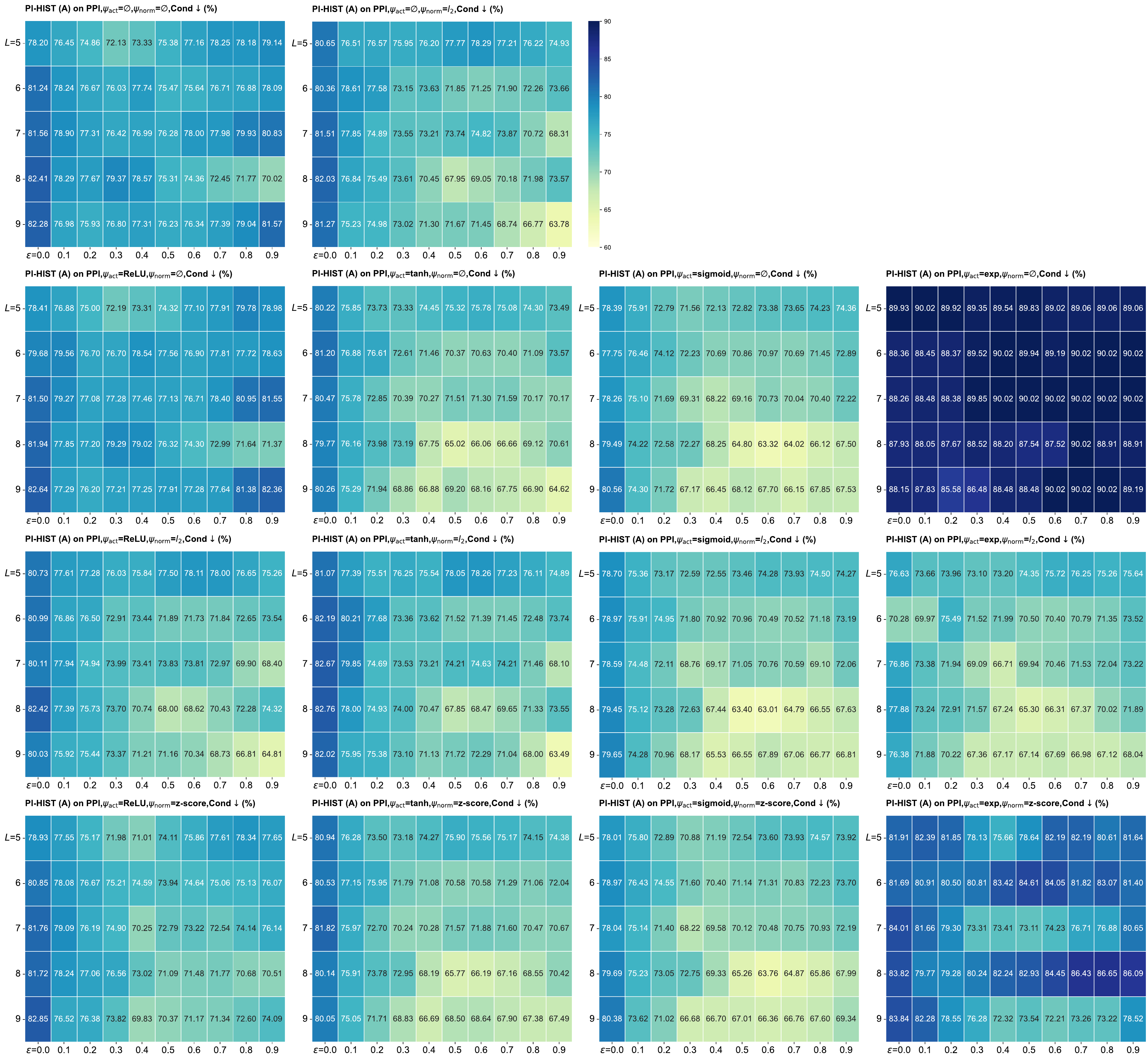}
    \caption{
	\textbf{Ablation study and parameter analysis of PI-HIST~(A) on PPI}. Representative results of node identity clustering (conductance$\downarrow$, \%), with $L \in \{ 5, 6, \cdots, 9\}$ ($y$-axis) and $\varepsilon \in \{ 0.0, 0.1, \cdots, 0.9\}$ ($x$-axis). The optional activation $\psi_{\rm act} (\cdot)$ is set to $\emptyset$ (no activation), ReLU, sigmoid, tanh, or exp. The optional normalization $\psi_{\rm norm} (\cdot)$ is set to $\emptyset$ (no activation), row-wise $l_2$-normalization, and column-wise z-score normalization.
    }
    \label{Fig:PI-HIST-A-ppi-p}
\end{figure}
\clearpage

\begin{figure}[p!]
    \centering
    \includegraphics[width=0.80\linewidth,trim=10 10 10 10,clip]{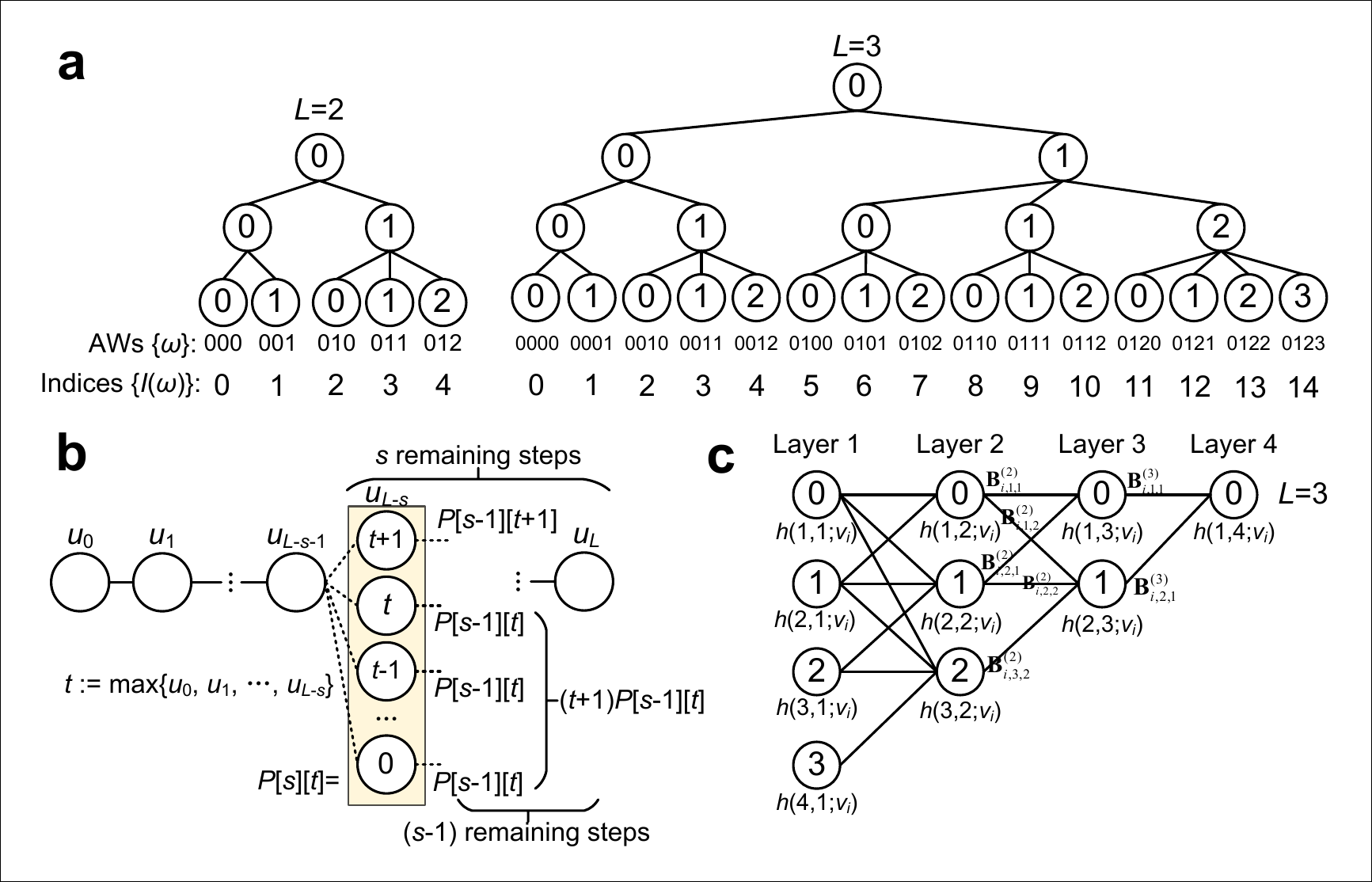}
    \caption{
	\textbf{Examples of key steps in PI-HIST~(A)'s embedding inference}.
    \textbf{a}, Tree $\mathcal{T}_L$ encoding for AW indexing with walk length $L = 2$ and $3$.
    \textbf{b}, Illustration of an entry $P[s][t]$ in DP table $P$ for $s$ remaining steps to generate a length-$L$ AW $\omega = (u_0, u_1, \cdots u_L)$ with $t$ as the maximum element visited so far. $u_l$ denotes the $l$-th element in $\omega$.
	\textbf{c}, An example of the AW-induced hierarchical structures $\mathcal{H}(v_i)$ for a node $v_i$ with length $L = 3$.
    }
    \label{Fig:AW_EST}
\end{figure}
\clearpage

\section{Additional Experiment Results}\label{App:Exp}

In addition to Fig.~\ref{Fig:NC_Eva}, outcomes of node-level tasks in terms of other metrics (micro F1 on Europe, Actor, BlogCatalog, and Amazon) and on the rest datasets (full three metrics on USA, Film, PPI, and DBLP) are visualized in Supplementary Fig.~\ref{Fig:NC_Eva_Apx}.
As a supplement to ig.~\ref{Fig:GRLP_Eva}, Supplementary Fig.~\ref{Fig:GRLP_Eva_Apx} shows evaluation results of edge-level tasks on Europe, Actor, PPI, BlogCatalog, and Amazon.
Following the same settings as Fig.~\ref{Fig:GSI_Eva}a-b, Supplementary Fig.~\ref{Fig:GSI_Eva_Apx} provides results of the $4$ baselines (BoD, BoHD, SSP, and CNS) for graph superfamily identification on both real-world and synthetic graphs.
In addition to the representative outcomes of PI-HIST~(R\&A) in Fig.~\ref{Fig:GSI_Eva}c-d, their results with respect to the rest settings of $\alpha$ are visualized in Supplementary Fig.~\ref{Fig:GSI_Eva_PI_HIST}.
Detailed numerical records of these evaluations, including mean and standard deviation over multiple independent runs and original efficiency scores before log transformation, are further summarized in Supplementary Tables~\ref{Tab:NC-Europe}-\ref{Tab:RA-alpha} and Supplementary Section~\ref{App:Res}.
Particularly, Table~\ref{Tab:Time-A} provides details about the inference time (sec) of different steps in PI-HIST (A).

Following the tuning strategy described in `Methods' section, we also analyze the effects of different settings of $\{ L, \varepsilon, \psi_{\rm act} (\cdot), \psi_{\rm norm} (\cdot)\}$ through ablation studies and parameter analyses.
Taking node-level tasks as examples, representative results of PI-HIST~(R) on PPI (node position classification with macro F1 on the validation set) and USA (community detection with modularity) are shown in Supplementary Fig.~\ref{Fig:PI-HIST-R-ppi-p} and \ref{Fig:PI-HIST-R-usa-p}.
Moreover, Supplementary Fig.~\ref{Fig:PI-HIST-A-usa-p} and \ref{Fig:PI-HIST-A-ppi-p} visualize outcomes of PI-HIST~(A) on USA (node identity classification with macro F1 on the validation set) and PPI (node identity clustering with conductance).
In most cases, applying $\psi_{\rm act} (\cdot)$ and $\psi_{\rm norm} (\cdot)$ can enhance the embedding quality.
For instance, PI-HIST~(A) achieves the best macro F1 ($60.77$\%) on USA when setting $\psi_{\rm act} (\cdot)$ and $\psi_{\rm norm} (\cdot)$ to tanh and column-wise z-score normalization, while the score decreases to $59.58$\% without these operations (i.e., $\psi_{\rm act} (\cdot) = \emptyset$ and $\psi_{\rm norm} (\cdot) = \emptyset$).
The optimal settings of $\{ L, \varepsilon, \psi_{\rm act} (\cdot), \psi_{\rm norm} (\cdot), \alpha\}$ vary across tasks and datasets, as summarized in Supplementary Tables~\ref{Tab:Param} and \ref{Tab:Param-RA}.

\section{Further Technical Details of PI-HIST}\label{App:Meth}

\subsection{Detailed Algorithms of PI-HIST~(R)}\label{App:Meth-PI-HIST-R}

The overall embedding derivation procedure of PI-HIST~(R) is summarized in Algorithm~\ref{Alg:PI-HIST-P}.

\begin{algorithm}[htbp]\small
\caption{Embedding Derivation of PI-HIST~(R)}
\label{Alg:PI-HIST-P}
\LinesNumbered
\KwIn{input graph $G = (V, E)$; embedding dimensionality $d$; RW length $L$; hyper-parameter $\varepsilon$}
\KwOut{derived embeddings ${\bf{Z}} \in \mathbb{R}^{N \times d}$ (matrix form)}
generate random input ${\bf{\Theta}} \in \mathbb{R}^{N \times d} \sim \mathcal{N} (0, 1/d)$ 

${\bf{Z}}^{(0)} \leftarrow {\bf{\Theta}}$

\For{$l \gets 0$ {\bf{to}} $L - 1$}{
    get ${\bf{Z}}^{(l+1)}$ via (1) in the main paper
}

get ${\bf{Z}}$ via (2) in the main paper
\end{algorithm}

\subsection{Detailed Algorithms of PI-HIST~(A)}\label{App:Meth-PI-HIST-A}

\begin{algorithm}[htbp]\small
\caption{Embedding Derivation of PI-HIST~(A)}
\label{Alg:PI-HIST-A}
\LinesNumbered
\KwIn{input graph $G = (V, E)$; embedding dimensionality $d$; total number of RWs $n$ sampled for each node; batch size $\hat n$ (i.e., number of RWs for each node processed in a run); RW/AW length $L$; hyper-parameter $\varepsilon$}
\KwOut{derived embeddings ${\bf{Z}} \in \mathbb{R}^{N \times d}$ (matrix form)}
get number of runs $m \leftarrow \left\lceil {n/\hat n} \right\rceil$

\For{$r \gets 1$ {\bf{to}} $m$}{
    sample $\hat n$ RWs for each node $v_i \in V$ ($\triangleright$\textbf{RW Sampling})
    
    map all sampled RWs $\{ w \}$ to their AWs $\{ \omega \}$ ($\triangleright$\textbf{AW Mapping})

    map all AWs $\{ \omega \}$ to their indices $\{ I[\omega] \}$ via Algorithm~\ref{Alg:AW-IDX} ($\triangleright$\textbf{AW Indexing})

    get statistics ${\bf{C}}^{(r)}$ of current run by counting all the observed node-index pairs $\{(v_i, I[\omega])\}$ ($\triangleright$\textbf{AW counting})
}

get full statistics ${\bf{C}}$ by summing $\{ {\bf{C}}^{(1)}, \cdots, {\bf{C}}^{(m)}\}$ and let ${\bf{C}} \leftarrow {\bf{C}} / n$
($\triangleright$\textbf{Statistic Gathering})

get hierarchical structures
$\{ {\bf{B}}^{(1)}, \cdots {\bf{B}}^{(L)} \}$ from ${\bf{C}}$ via (\ref{Eq:B}) ($\triangleright$\textbf{Hierarchical Structure Extraction})

generate random input ${\bf{\Theta}} \in \mathbb{R}^{L \times N \times d} \sim \mathcal{N} (0, 1/d)$ 

${\bf{U}}^{(0)} \leftarrow {\bf{\Theta}}$

\For{$l \gets 0$ {\bf{to}} $L - 1$}{
    get ${\bf{U}}^{(l+1)}$ via (3) in the main paper
}

get ${\bf{Z}}$ via (4) in the main paper
\end{algorithm}

The overall embedding derivation procedure of PI-HIST~(A) is summarized in Algorithm~\ref{Alg:PI-HIST-A}.
Similar to PI-HIST~(R), PI-HIST~(A) also derives node-wise embeddings, arranged in a matrix ${\bf{Z}} \in \mathbb{R}^{N \times d}$, via one FFP through AW-induced hierarchical structures (lines 14-18 in Algorithm~\ref{Alg:PI-HIST-A}).
As there are currently no closed-form solutions to directly obtain AW statistics of a graph $G = (V, E)$, before FFP, we employ an efficient GPU-friendly algorithm (lines 6-13 in Algorithm~\ref{Alg:PI-HIST-A}) to estimate AW statistics via RW sampling and extract AW-induced hierarchical structures described by tensors $\{ {\bf{B}}^{(1)}, \cdots, {\bf{B}}^{(L)} \}$. This algorithm includes steps of (\romannumeral1) RW sampling, (\romannumeral2) AW mapping, (\romannumeral3) AW indexing, (\romannumeral4) AW counting, (\romannumeral5) statistic gathering, and (\romannumeral6) hierarchical structure extraction.

\textbf{RW Sampling and AW mapping}. Generally, to estimate AW statistics of each node $v_i$, which form a distribution over all length-$L$ AWs, one can in sequence (\romannumeral1) sample $n$ RWs $\{ w \}$ started from $v_i$, (\romannumeral2) map each RW $w$ to its corresponding AW $\omega$, and (\romannumeral3) record the occurrence frequency of each observed AW $\omega$.
The proposed algorithm performs this procedure for all nodes simultaneously.
For large-scale graphs, we sample only $\hat n$ RWs from each node $v_i$ in one run, with $\hat n \le n$. The estimation is then completed in $m := \left\lceil {n/\hat n} \right\rceil $ runs (line 6 in Algorithm 2).
In each run, we employ the GPU-compatible implementation from \texttt{PyTorch Cluster} to facilitate the RW sampling (line 8 in Algorithm~\ref{Alg:PI-HIST-A}) and adopt \texttt{Numba} to implement the parallel mapping of all the sampled RWs to their corresponding AWs (line 9 in Algorithm~\ref{Alg:PI-HIST-A}).

\textbf{AW Indexing}.
For efficient counting, the derived AWs $\{ \omega \}$ are first mapped to corresponding indices $\{ I[\omega] \}$ using a GPU-optimized implementation based on \texttt{Numba} (line 10 in Algorithm~\ref{Alg:PI-HIST-A}).
For a fixed walk length $L$, the union of all candidate AWs can be encoded by a tree $\mathcal{T}_L$ with height $(L+1)$, where each root-to-leaf path represents a unique AW (cf. Supplementary Fig.~\ref{Fig:AW_EST}a for examples of $\mathcal{T}_2$ and $\mathcal{T}_3$).
In essence, $\mathcal{T}_L$ corresponds to the search space of the naive depth-first search process to generate all candidate length-$L$ AWs.
Let $\Omega(\mathcal{T}_L)$ be the sequence of AWs induced by all root-to-leaf paths arranged in order from left to right. The index $I[\omega]$ of an AW $\omega$ is defined as $\omega$'s index in $\Omega(\mathcal{T}_L)$, with $\{ I[\omega] \}$ started from $0$.
For instance, we have $I[(0, 0, 0, 0)] = 0$ and $I[(0, 1, 2, 3)] = 14$ for $L=3$ (Supplementary Fig.~\ref{Fig:AW_EST}a).
However, the direct index mapping involves enumerating or searching all candidate AWs in the worst case, which is computationally intractable when $L$ is large.
We therefore develop a dynamic programming (DP) solution for AW indexing as summarized in Algorithm~\ref{Alg:AW-IDX}.
It maintains an auxiliary table $P$ that considers a process of generating an AW sequence $(u_0, u_1, \cdots, u_L)$ step-by-step starting from $(u_0) = (0)$, with the $l$-th step inserting an element $u_l \in \{ 0, 1, \cdots, L-1\}$.
As illustrated in Supplementary Fig.~\ref{Fig:AW_EST}b, each entry $P[s][t]$ stores the number of AWs that can be generated with $s$ remaining steps (when inserting $u_{L-s}$), where the maximum element visited so far (i.e., $\max \{ u_0, u_1, \cdots, u_{L-s} \}$) is $t$.
Initially, we let $P[0][t] = 1$ for $0 \le t \le L$.
The state transition of $P[s][t]$ for $1 \le s \le L$ and $0 \le t \le L$ is then defined as
\begin{equation}
    P[s][t] = (t + 1)P[s - 1][t] + P[s - 1][t + 1],
\end{equation}
where the first and second terms correspond to cases of inserting $u_{L-s} \in \{ 0, 1, \cdots, t\}$ and $u_{L-s} = (t+1)$, respectively (Supplementary Fig.~\ref{Fig:AW_EST}b).
Since $P$ is determined solely by $L$, it can be precomputed before running the overall PI-HIST~(A) inference algorithm.
Given an AW $\omega = (u_0, u_1, \cdots, u_L)$ with length $L$, we first initialize its index $I [\omega]$ as $0$ (line 21 in Algorithm~\ref{Alg:AW-IDX}). We then in sequence examine each element $u_l$ ($1 \le l \le L$) in $\omega$ to update $I [\omega]$ by adding $u_l \times P[s-1][u_{\max}]$, where $s = (L - l)$ is the number of remaining steps; $u_{\max}$ is the maximum element visited so far (lines 22-27 in Algorithm~\ref{Alg:AW-IDX}).

\begin{algorithm}[htbp]\small
\caption{DP Solution to AW Indexing}
\label{Alg:AW-IDX}
\LinesNumbered
\KwIn{RW/AW length $L$; AW $\omega = (u_0, u_1, \cdots, u_L)$; precomputed DP table $P$}
\KwOut{AW index $I[\omega]$}
$u_{\max} \leftarrow 0$ //Initialize the maximum element visited so far

$s \leftarrow L - 1$ //Initialize steps remained for generation

$I[\omega] \leftarrow 0$ //Initialize the index to be computed

\For{$l \gets 1$ {\bf{to}} $L$}{
    get the $l$-th element $u_l$ from $\omega$

    $I[\omega] \leftarrow I[\omega] + u_l \times P[s-1][u_{\max}]$

    \If{$u_{\max} < u_l$ {\rm{(i.e.,}} $u_l = u_{\max} + 1${\rm{)}} }{
        $u_{\max} \leftarrow u_{\max} + 1$ //Update the maximum element visited so far
    }
    $s \leftarrow s - 1$ //Decrement steps remained
}

\end{algorithm}

\textbf{AW Counting and Statistic Gathering}.
Based on AW indexing, we estimate AW statistics of the current run $r$ by counting the observed node-index pairs $\{ (v_i, I[\omega]) \}$ for all nodes (line 11 in Algorithm~\ref{Alg:PI-HIST-A}), using the GPU-compatible unique value extraction provided by \texttt{PyTorch}.
Let $\eta_L$ be the number of all candidate length-$L$ AWs. Each AW index $I[\omega] \in \{ 0, 1, \cdots, \eta_L - 1\}$. The counting result of current run $r$ can be arranged in a matrix ${\bf{C}}^{(r)} \in \mathbb{R}^{N \times \eta_L}$, with ${\bf{C}}^{(r)}_{i, (I[\omega_j]+1)}$ as the occurrence count of $(v_i, I[\omega_j])$.
Although $\eta_L$ grows exponentially with the increase of $L$, ${\bf{C}}^{(r)}$ is usually very sparse, because only a small ratio of candidate AWs are observed during the estimation.
The full statistics arranged in a matrix ${\bf{C}} \in \mathbb{R}^{N \times \eta_L}$ can be obtained by summing $\{ {\bf{C}}^{(1)}, \cdots {\bf{C}}^{(m)} \}$ of all runs via sparse matrix summation (line 12 in Algorithm~\ref{Alg:AW-IDX}).

\textbf{Hierarchical Structure Extraction}.
For a fixed walk length $L$, $l$-th-to-last element in an AW $\omega$ can only take a value in $\{ 0, 1, \cdots, L+1-l\}$.
Therefore, the AW-induced hierarchical structure $\mathcal{H} (v_i)$ of node $v_i$ has $(L+1)$ layers, with the $l$-th layer (from bottom to top) containing units $\{ 0, \cdots, L+1-l\}$ (cf. Supplementary Fig.~\ref{Fig:AW_EST}c for an example of $L = 3$). We further assign weights to links between adjacent layers using the estimated AW statistics ${\bf{C}}_{i,:}$ of $v_i$, which is equivalent to extracting a weighted hierarchical structure about AW (line 13 in Algorithm~\ref{Alg:PI-HIST-A}).
As mentioned in `Methods' section, we use a tensor ${\bf{B}}^{(l)} \in \mathbb{R}^{N \times L \times L}$ to encode weighted links between the $l$-th and $(l+1)$-th layers in the hierarchical structures of all nodes.
Let $\omega_j [l]$ denote the $l$-th element ($0 \le l \le L$) in the $j$-th AW ($0 \le j \le \eta_L - 1$).
We define
\begin{equation}\label{Eq:B}
    {\bf{B}}_{i,s,t}^{(l)}: = \sum\limits_{j = 0}^{{\eta _L} - 1} {{{\bf{C}}_{i,(j+1)}} \times {\bf{1}}({\omega _j}[l - 1] = s) \times {\bf{1}}(\omega_j [l] = t)},
\end{equation}
where ${\bf{1}}(a=b)$ is the indicator function, with ${\bf{1}}(a=b) = 1$ if $a=b$ and ${\bf{1}}(a=b) = 0$ otherwise.
Namely, the weight ${\bf{B}}^{(l)}_{i,s,t}$ equals the total frequency of AWs that contains $(s, t)$ at the $l$-th step, with ${\omega _j}[l - 1] = s$ and $\omega_j [l] = t$.
From this perspective, we compute $\{ {\bf{B}}^{(1)}, \cdots, {\bf{B}}^{(L)} \}$ by scanning all the observed AWs of each node once, which corresponds to only a tiny subset of non-zero entries in ${\bf{C}}$,
using a GPU-optimized implementation based on \texttt{Numba}.

For a graph $G$ with $N$ nodes, the complexities of sampling $n$ length-$L$ RWs for each node (RW sampling) and mapping these RWs to AWs (AW mapping) are both $O(NLn)$.
The complexity of mapping an AW $\omega$ to its index $I[\omega]$ using Algorithm~\ref{Alg:AW-IDX} is $O(L)$. Since there are in total $(Nn)$ AWs, it takes $O(NLn)$ for the AW indexing step.
To count $\{ (v_i, I[\omega]) \}$, we need to scan all the derived $(Nn)$ AWs once. Hence, the complexity of AW counting is within $O(Nn)$.
Let $\tilde \eta$ be the number of observed AWs during the estimation. 
We usually have $\tilde \eta \ll \eta$, because only a small ratio of candidate AWs can be observed, consisent with that $\{ {\bf{C}}^{(1)}, \cdots, {\bf{C}}^{(m)} \}$ and ${\bf{C}}$ are sparse matrices. The complexity of merging estimation results from $m$ runs (statistic gathering) is $O(Nm{\tilde \eta})$ using sparse matrix summation.
Finally, the step of hierarchical structure extraction, which derives tensors $\{ {\bf{B}}^{(1)}, \cdots, {\bf{B}}^{(L)} \}$ encoding $\{ \mathcal{H}(v_i) | v_i \in V \}$, can be completed within $O(N{\tilde \eta})$ by traversing non-zero entries of ${\bf{C}}$.
For large sparse graphs, we can assume that $L \ll N$ and $m \ll N$.
The overall complexity of the proposed AW statistic estimation procedure is about $O(NLn + N{\tilde \eta}) \approx O(N(n + {\tilde \eta}))$.

\section{Theoretical Support for PI-HIST}\label{App:Th-Sup}
This section provides theoretical support for the effectiveness of PI-HIST. We first show that by simplifying away the nonlinear activations and normalization, the propagation of Gaussian random features through a topology-induced operator in PI-HIST can be interpreted as a Gaussian random projection of high-dimensional structural features.
For PI-HIST~(R), we further demonstrate that the RW diffusion profiles of nodes within the same metastable community are close. For PI-HIST~(A), we show how similarity in AW statistics propagates through the AW-induced hierarchical operator using the established connection between AW statistics and rooted ego-net structures.

\subsection{Random features preserve the geometric of topology-induced structural representations}
Consider the topology-induced operator that associates each node $v_i$ with a high-dimensional representation ${\bf{x}}_i \in \mathbb{R}^{F}$. We can arrange the representations of all nodes in a matrix ${\bf{X}} \in \mathbb{R}^{N \times F}$, with ${\bf{x}}_i := {\bf{X}}_{i,:}$.
After simplifying away the nonlinear activations and normalization, PI-HIST can be formulated as ${\bf{Z}} = {\bf{X}}{\bf{\Theta}}$, where ${\bf{\Theta}} \in \mathbb{R}^{N \times d}$, ${\bf{\Theta}}_{ir} \mathop \sim\limits^{\rm{i.i.d}} \mathcal{N} (0, 1/d)$, and  $d \ll N$; ${\bf{z}}_i := {\bf{Z}}_{i,:}$ represents the resulting low-dimensional embedding of $v_i$.
In this sense, PI-HIST can be viewed as the Gaussian random projection on ${\bf{X}}$, where ${\bf{\Theta}}$ provides a random probe of the structural geometry encoded in ${\bf{X}}$, with its effectiveness guaranteed by the following Gaussian form of the Johnson–Lindenstrauss lemma~\cite{dasgupta2003elementary}.

\begin{lemma}[Gaussian random projection]\label{Th:Rand-Proj}
    For any $0 < \pi < 1$ and $0 < \rho < 1$, if $d \ge \frac{8}{{{\pi ^2}}}\ln \frac{{{N^2}}}{\rho }$, then with probability at least $(1 - \rho)$, we have
    \begin{equation}
        (1 - \pi )||{{\bf{x}}_i} - {{\bf{x}}_j}||_2^2 \le ||{{\bf{z}}_i} - {{\bf{z}}_j} ||_2^2 \le (1 + \pi )||{{\bf{x}}_i} - {{\bf{x}}_j}||_2^2,
    \end{equation}
    for each pair of nodes $(v_i, v_j)$, with ${\bf{z}}_i = {\bf{x}}_i{\bf{\Theta}}$.
\end{lemma}
\begin{proof}
    For each pair of nodes $(v_i, v_j)$, we have $({\bf{z}}_i - {\bf{z}}_j) = [({\bf{x}}_i - {\bf{x}}_j)\bf{\Theta}]$. Since entries in ${\bf{\Theta}}$ are variables following $\mathcal{N} (0, 1/d)$, each coordinate of $[({\bf{x}}_i - {\bf{x}}_j)\bf{\Theta}]$ is also Gaussian with variance $||{\bf{x}}_i - {\bf{x}}_j||_2^2/d$. Therefore, we can obtain
    \begin{equation}
        \frac{{||{{\bf{z}}_i} - {{\bf{z}}_j}||_2^2}}{{||{{\bf{x}}_i} - {{\bf{x}}_j}||_2^2}} = \frac{{||({{\bf{x}}_i} - {{\bf{x}}_j}){\bf{\Theta }}||_2^2}}{{||{{\bf{x}}_i} - {{\bf{x}}_j}||_2^2}} \sim \frac{{\chi_d^2}}{d}.
    \end{equation}
    By $\mathbb{E}[\chi_d^2] = d$, we have the following derivation
    \begin{equation}
        \mathbb{E}[||{{\bf{z}}_i} - {{\bf{z}}_j}||_2^2] = \mathbb{E}[||({{\bf{x}}_i} - {{\bf{x}}_j}){\bf{\Theta }}||_2^2] = ||{{\bf{x}}_i} - {{\bf{x}}_j}||_2^2 \mathbb{E}[\frac{{||({{\bf{x}}_i} - {{\bf{x}}_j}){\bf{\Theta }}||_2^2}}{{||{{\bf{x}}_i} - {{\bf{x}}_j}||_2^2}}] = ||{{\bf{x}}_i} - {{\bf{x}}_j}||_2^2 \mathbb{E}[\frac{\chi_d^2}{d}] = ||{{\bf{x}}_i} - {{\bf{x}}_j}||_2^2.
    \end{equation}
    For $X \sim \chi_d^2$, $\mathbb{E}[X] = d$ and ${\mathop{\rm Var}\nolimits} (X) = 2d$.
    The standard chi-square concentration inequality gives, for $0 < \pi < 1$,
    \begin{equation*}
        \Pr (|X - k| > t) \le 2 \exp \{  - {dt^2}/(8k)\} \Rightarrow
    \end{equation*}
    \begin{equation*}
        \Pr (|\frac{{\chi _d^2}}{d} - E[\frac{{\chi _d^2}}{d}]| > \pi ) = \Pr (|\frac{{\chi _d^2}}{d} - 1| > \pi ) \le 2\exp \{  - \frac{{d{\pi ^2}}}{8}\} \Rightarrow
    \end{equation*}
    \begin{equation*}
        \Pr (|\frac{{\chi _d^2}}{d} - 1| > \pi ) = \Pr (| \frac{{||{{\bf{z}}_i} - {{\bf{z}}_j}||_2^2}}{{||{{\bf{x}}_i} - {{\bf{x}}_j}||_2^2}} - 1 | > \pi ) \le 2\exp \{  - \frac{{d{\pi ^2}}}{8}\} \le \frac{2}{{{N^2}}}\rho \Rightarrow
    \end{equation*}
    \begin{equation}\label{Eq:AuxInq}
        \Pr [(1 - \pi )||{{\bf{x}}_i} - {{\bf{x}}_j}||_2^2 \le ||{{\bf{z}}_i} - {{\bf{z}}_j} ||_2^2 \le (1 + \pi )||{{\bf{x}}_i} - {{\bf{x}}_j}||_2^2 ] \ge 1 - \frac{2}{{{N^2}}}\rho.
    \end{equation}
    By applying (\ref{Eq:AuxInq}) to at most $N^2/2$ pairs of nodes completes the proof.
\end{proof}

\subsection{RW-induced operators preserve node-position information}
Let $w_l$ be the node visited by the $l$-th step of a RW, with the index $l$ starting from $0$. $({\bf{D^{-1}A}})^{L}$ then encodes the transition probability that $({{\bf{D}}^{ - 1}}{\bf{A}})_{i,j}^{(L)} = \Pr ({w_L} = {v_j}|{w_0} = {v_i})$.
The $i$-th row $P_i^{(L)} := ({{\bf{D}}^{-1/2}\bf{A}})^L_{i,:}$ is the probability distribution of nodes visited by a length-$L$ RW starting from $v_i$.

Consider a metastable community $C \subseteq V$ satisfying two standard assumptions.
For simplicity, let $P_{i,C}^{(L)}: = \Pr ({w_L} \in C|{w_0} = {v_i})$ and $P_{i,\bar C}^{(L)}: = \Pr ({w_L} \notin C|{w_0} = {v_i})$, which correspond to the probability within and outside the community $C$ after the $L$-step RW starting from $v_i$, respectively.
First, RWs initialized at nodes inside $C$ approach a similar within-community distribution before substantially escaping the community. Namely, there exists a reference distribution ${\upsilon _C}$ within $C$ such that, for each $v_i \in C$,
\begin{equation}\label{Eq:Cond-1}
    ||P_{i,C}^{(L)} - {\upsilon _C}|{|_1} \le {\epsilon _{{\rm{in}}}}.
\end{equation}
Second, the probability of escaping $C$ within the considered diffusion scale is small. Namely, for each $v_i \in C$,
\begin{equation}\label{Eq:Cond-2}
    ||P_{i,\bar C}^{(L)}||_1 \le {\epsilon _{{\rm{out}}}}.
\end{equation}

\begin{proposition}[Within-community diffusion similarity]\label{Th:RW}
    Under the assumptions defined in (\ref{Eq:Cond-1}) and (\ref{Eq:Cond-2}), for any two nodes $v_i, v_j \in C$,
    \begin{equation}
        ||P_{i,C}^{(L)} - P_{j,C}^{(L)}|{|_1} \le 2({\epsilon _{{\rm{in}}}} + {\epsilon _{{\rm{out}}}}).
    \end{equation}
\end{proposition}
\begin{proof}
    On the one hand, we have $||P_{i,C}^{(L)} - {\upsilon _C}|{|_1} \le {\epsilon _{{\rm{in}}}}$ and $||P_{j,C}^{(L)} - {\upsilon _C}|{|_1} \le {\epsilon _{{\rm{in}}}}$ for any nodes $v_i, v_j \in C$. By the triangle inequality, we can obtain
    \begin{equation}\label{Eq:AuxInq-1}
        ||P_{i,C}^{(L)} - P_{j,C}^{(L)}|{|_1} = ||(P_{i,C}^{(L)} - {\upsilon _C}) - (P_{j,C}^{(L)} - {\upsilon _C})|{|_1} \le ||P_{i,C}^{(L)} - {\upsilon _C}|{|_1} + ||P_{j,C}^{(L)} - {\upsilon _C}|{|_1} \le 2{\epsilon _{{\rm{in}}}}.
    \end{equation}
    On the other hand, we also have $||P_{i,\bar C}^{(L)}|| \le {\epsilon _{{\rm{out}}}}$ and $||P_{j,\bar C}^{(L)}|| \le {\epsilon _{{\rm{out}}}}$ for any nodes $v_i, v_i \in C$. Also by the triangle inequality, we further have the following derivation
    \begin{equation}\label{Eq:AuxInq-2}
        ||P_{i,\bar C}^{(L)} - P_{j,\bar C}^{(L)}|{|_1} \le ||P_{i,\bar C}^{(L)}|| + ||P_{j,\bar C}^{(L)}|| \le 2{\epsilon _{{\rm{out}}}}
    \end{equation}
    By merging (\ref{Eq:AuxInq-1}) and (\ref{Eq:AuxInq-2}), we can finally obtain
    \begin{equation}
        ||P_i^{(L)} - P_j^{(L)}|{|_1} = ||(P_{i,C}^{(L)} + P_{i,\bar C}^{(L)}) - (P_{j,C}^{(L)} + P_{j,\bar C}^{(L)})|{|_1} \le ||P_{i,C}^{(L)} - P_{j,C}^{(L)}|{|_1} + ||P_{i,\bar C}^{(L)} - P_{j,\bar C}^{(L)}|{|_1} \le 2({\epsilon _{{\rm{in}}}} + {\epsilon _{{\rm{out}}}}).
    \end{equation}
    This completes the proof.
\end{proof}
Since $||{\bf{x}}||_2 \le ||{\bf{x}}||_1$, we also have $||P_i^{(L)} - P_j^{(L)}|{|_2} \le 2({\epsilon _{{\rm{in}}}} + {\epsilon _{{\rm{out}}}})$. Combing Proposition~\ref{Th:RW} with Lemma~\ref{Th:Rand-Proj} gives a direct result in the embedding space.
\begin{corollary}[Within-community proximity after random projection]
    For PI-HIST~(R), ${\bf{z}}_i = P_i^{L}{\bf{\Theta}}$. Under the two assumptions of Proposition~\ref{Th:RW}, with a high probability, 
    \begin{equation}
        ||{{\bf{z}}_i} - {{\bf{z}}_j}|{|_2} \le 2\sqrt {1 + \pi } ({\epsilon _{{\rm{in}}}} + {\epsilon _{{\rm{out}}}})
    \end{equation}
    for all nodes $v_i, v_j \in C$.
\end{corollary}
In summary, nodes within the same metastable community have similar diffusion profiles and also remain close after Gaussian random projection.
This provides a theoretical mechanism through which PI-HIST~(R) preserves node-position information.
PI-HIST~(R) additionally contains a skip connection controlled by $\varepsilon$, which can be formulated as
\begin{equation*}
    {\bf{Z}}^{(l)} = [\varepsilon {\bf{I}} + (1 - \varepsilon )({{\bf{D}}^{ - 1}} {\bf{A}})] {\bf{Z}}^{(l-1)}.
\end{equation*}
Let ${\bf{Q}} := \varepsilon {\bf{I}} + (1 - \varepsilon )({{\bf{D}}^{ - 1}} {\bf{A}})$. PI-HIST~(R)'s $L$-layer linear core is ${\bf{Z}}^{(L)} = {\bf{Q}}^L{\bf{\Theta}}$. The above reasoning can also be applied to the lazy-walk operator ${\bf{Q}}^L$, which corresponds more directly to PI-HIST~(R).

\subsection{AW-induced operators preserve structural-role information}
For PI-HIST~(A), the following Theorem~\ref{Th:AW} proved by Micali et al.~\cite{micali2016reconstructing} indicates that AW statistics can be used to reconstruct nodes' ego-nets and thus have the potential to characterize node identities.
\begin{theorem}[AW statistics for ego-net reconstruction]\label{Th:AW}
    Let $G_L (v_i)$ be subgraph rooted at a node $v_i$ (alternatively called $v_i$'s ego-net) induced by nodes $\cup _{l = 1}^L{N_l}({v_i})$ with a distance less than or equal to $L$ from $v_i$. Let ${\bf{q}}_l (v_i)$ denote the distribution of AWs with respect to the sampled length-$l$ RWs originating at $v_i$. One can reconstruct $G_L(v_i)$ in time $O(\tilde N^2)$ with $O(\tilde N^2)$ access to the collection of AW distributions $[{\bf{q}}_1(v_i), \cdots, {\bf{q}}_r(v_i)]$, where $r = O(\tilde M)$; $\tilde N$ and $\tilde M$ represent the number of nodes and edges in $G_L (v_i)$.
\end{theorem}
Since a length-$L$ AW contains its shorter prefixes, we can derive ${\bf{q}}_l (v_i)$ using ${\bf{q}}_L(v_i)$, with $l < L$. For instance, AW $(0, 1, 2, 0)$ provides information about a shorter AW $(0, 1, 2)$. Therefore, ${\bf{q}}_L(v_i)$ can be used to characterized the ego-net $G_L (v_i)$ and thus $v_i$'s identity (or structural role). Two nodes $v_i$ and $v_j$ with similar structural role are expected to have similar distributions ${\bf{q}}_L(v_i)$ and ${\bf{q}}_L(v_j)$.

PI-HIST~(A) does not directly use $\{ {\bf{q}}_L(v_i) | v_i \in V\}$ as the embedding. Instead, it first estimates $\{ {\bf{q}}_L(v_i) | v_i \in V\}$ based on a finite number of RWs and coverts them into an AW-induced hierarchical operator.
For simplicity, let ${\bf{q}}_i := {\bf{q}}_L (v_i) \in \mathbb{R}^{\eta_L}$ and ${\bf{\hat q}}_i$ denote the estimation of ${\bf{q}}_i$.
According to (\ref{Eq:B}), for each layer $l$, there exists a fixed linear mapping ${\bf{T}}_l$, determined only based on AW encoding, such that ${\bf{B}}_{i,:,:}^{(l)} = {\bf{T}}_l {\bf{q}}_i$. Given two nodes $v_i$ and $v_j$, we then have
\begin{equation}\label{Eq:AW-B-Dif}
    ||{\bf{B}}_{i,:,:}^{(l)} - {\bf{B}}_{j,:,:}^{(l)}|| \le ||{{\bf{T}}_l}|| \cdot ||{{\bf{\hat q}}_i} - {{\bf{\hat q}}_j}||,
\end{equation}
for any compatible matrix or vector norm.
Let the linear propagation operator at the $l$-th layer be ${\bf{\bar B}}_i^{(l)}: = \varepsilon I + (1 - \varepsilon ){\bf{B}}_{i,:,:}^{(l)}$. Ignoring the optional nonlinear activation and normalization in PI-HIST~(A), the hierarchical propagation related to $v_i$ can be formulated as ${{{\bf{\bar B}}}_i}: = {\bf{\bar B}}_i^{(L)}{\bf{\bar B}}_i^{(L - 1)} \cdots {\bf{\bar B}}_i^{(1)}$. Let ${\bf{e}}_*$ select the single output unit in the final layer of PI-HIST~(A). The resulting high-dimensional AW-induced structural signature of $v_i$ can then be formulated as 
\begin{equation}\label{Eq:AW-x}
    {\bf{x}}_i = {\bf{e}}_*^T {\bf{\bar B}}_i.
\end{equation}
The linearized PI-HIST~(A) embedding can therefore be represented as
\begin{equation}\label{Eq:AW-z}
    {\bf{z}}_i = {\bf{x}}_i {\bf{\Theta}}.
\end{equation}
Assume that $||{\bf{\bar B}}_i^{(l)}|{|_2} \le \kappa $ for each node $v_i$ and layer $l$.
By the telescoping identity for products of matrices, we can obtain
\begin{equation*}
    {{{\bf{\bar B}}}_i} - {{{\bf{\bar B}}}_j} = \sum\limits_{l = 1}^L {{\bf{\bar B}}_i^{(L)} \cdots {\bf{\bar B}}_i^{(l + 1)} \cdot ({\bf{\bar B}}_i^{(l)} - {\bf{\bar B}}_j^{(l)})}  \cdot {\bf{\bar B}}_j^{(l - 1)} \cdots {\bf{\bar B}}_j^{(1)} \Rightarrow
\end{equation*}
\begin{equation}
    ||{{{\bf{\bar B}}}_i} - {{{\bf{\bar B}}}_j}|{|_2} \le {\kappa ^{L - 1}}{\sum\limits_{l = 1}^L {||{\bf{\bar B}}_i^{(l)} - {\bf{\bar B}}_j^{(l)}||} _2}.
\end{equation}
Since ${\bf{\bar B}}_i^{(l)} - {\bf{\bar B}}_j^{(l)} = (1 - \varepsilon) ({\bf{B}}_{i,:,:}^{(l)} - {\bf{B}}_{j,:,:}^{(l)})$, we further have
\begin{equation}\label{Eq:AW-B}
    ||{{{\bf{\bar B}}}_i} - {{{\bf{\bar B}}}_j}|{|_2} \le (1 - \varepsilon){\kappa ^{L - 1}}{\sum\limits_{l = 1}^L {||{\bf{\bar B}}_{i,:,:}^{(l)} - {\bf{\bar B}}_{j,:,:}^{(l)}||} _2}.
\end{equation}
Combining (\ref{Eq:AW-B-Dif}), (\ref{Eq:AW-x}), and (\ref{Eq:AW-B}) gives
\begin{equation}\label{Eq:AW-Dif}
    ||{{\bf{x}}_i} - {{\bf{x}}_j}|{|_2} \le {c_L}||{{\bf{\hat q}}_i} - {{\bf{\hat q}}_j}|{|_2},
\end{equation}
with $c_L := (1 - \varepsilon ){\kappa ^{L - 1}}\sum\nolimits_{l = 1} {||{{\bf{T}}_l}|{|_2}}$.
By further combing Lemma~\ref{Th:Rand-Proj} with (\ref{Eq:AW-z}) and (\ref{Eq:AW-Dif}), we can derive that with a high probability,
\begin{equation}
    ||{{\bf{z}}_i} - {{\bf{z}}_j}|{|_2} \le c_L \sqrt {1 + \pi } ||{\bf{\hat q}}_i - {\bf{\hat q}}_j||_2.
\end{equation}
Let $\tau_i := ||{\bf{\hat q}}_i - {\bf{q}}_i ||_2$ be the estimation error of ${\bf{q}}_i$.
By the triangle inequality, we can obtain
\begin{equation}
    ||{{\bf{\hat q}}_i} - {{\bf{\hat q}}_j}|{|_2} = ||({{\bf{\hat q}}_i} - {{\bf{q}}_i}) - ({{\bf{\hat q}}_j} - {{\bf{q}}_j}) + ({{\bf{q}}_i} - {{\bf{q}}_j})|{|_2} \le ||{{\bf{\hat q}}_i} - {{\bf{q}}_i}|{|_2} + ||{{\bf{\hat q}}_j} - {{\bf{q}}_j}|{|_2} + ||{{\bf{q}}_i} - {{\bf{q}}_j}|{|_2} = ||{{\bf{q}}_i} - {{\bf{q}}_j}|{|_2} + {\tau_i} + {\tau_j}.
\end{equation}
Therefore, the difference between embeddings ${\bf{z}}_i$ and ${\bf{z}}_j$ for any two nodes $v_i, v_j \in V$ becomes
\begin{equation}
    ||{{\bf{z}}_i} - {{\bf{z}}_j}|{|_2} \le {c_L}\sqrt {1 + \pi } (||{{\bf{q}}_i} - {{\bf{q}}_j}|| + {\tau _i} + {\tau _j}).
\end{equation}
Together with Theorem~\ref{Th:AW}, this result provides a theoretical mechanism linking AWs to identity embeddings. AW statistics retain information about nodes' ego-net structures, while PI-HIST~(A) transforms these statistics into low-dimensional embeddings in a stable manner.
Particularly, a pair of nodes $(v_i, v_j)$ with similar structural roles are expected to have similar AW statistics $({\bf{q}}_i, {\bf{q}}_j)$ and thus close embeddings $({\bf{z}}_i, {\bf{z}}_j)$.
As the number of samples RWs increases, $\{ {\bf{\hat q}}_i \}$ converge to $\{ {\bf{q}}_i \}$ and thus $\{ \tau_i \}$ decrease.
This also provides a theoretical interpretation about the trade-off between AW sampling cost and embedding stability.

\section{Detailed Experiment Settings}\label{App:Set}

\subsection{Datasets}\label{App:Set-Data}

\begin{table}[t]\footnotesize 
\centering
\caption{Summary of $8$ public real-world graph datasets, where $N$, $M$, and $C$ denote the numbers of nodes, edges, and classes; the ground-truth of each dataset is with either node identities or node positions.}
\label{Tab:Data}
\begin{tabular}{l|l|l|l|c|l}
\hline
\textbf{Datasets} & $N$ & $M$ & $C$ & \multicolumn{1}{l|}{\textbf{Labels}} & \textbf{Scenarios} \\ \hline
Europe \cite{ribeiro2017struc2vec} & 399 & 5,993 & 4 & \multirow{4}{*}{\begin{tabular}[c]{@{}c@{}}Node Identities\\(First Type)\end{tabular}} & Europe Air-Traffic Network \\
USA \cite{ribeiro2017struc2vec} & 1,186 & 13,597 & 4 &  & USA Air-Traffic Network \\
Actor \cite{jiao2021survey} & 7,592 & 26,553 & 4 &  & Wikipedia Actor Co-occurrence and Influences \\
Film \cite{jiao2021survey} & 26,851 & 122,234 & 4 &  & Wikipedia Film Entity Relations and Roles \\ \hline
PPI \cite{grover2016node2vec} & 3,852 & 37,841 & 50 & \multirow{4}{*}{\begin{tabular}[c]{@{}c@{}}Node Positions\\(Second Type)\end{tabular}} & Homo Sapiens Protein-Protein Interactions \\
BlogCatalog \cite{grover2016node2vec} & 10,312 & 333,983 & 39 &  & BlogCatalog Online Social Friend Relations \\
DBLP \cite{yang2015defining} & 317,080 & 1,049,866 & 100 &  & DBLP Academic Author Collaborations \\
Amazon \cite{yang2015defining} & 334,863 & 925,872 & 100 &  & Amazon Product Co-purchasing Relations \\ \hline
\end{tabular}
\end{table}

Details about the $8$ public real-world graph datasets used in our evaluations (Fig.~\ref{Fig:NC_Eva}a) are summarized in Supplementary Table~\ref{Tab:Data}, where $N$, $M$, and $C$ represent the number of nodes, edges, and classes; the first and last $4$ datasets contain node identity and position ground-truth, respectively.
These datasets cover diverse scenarios and scales from thousands to millions of edges.

For the first type of datasets, Europe\footnote{\url{https://github.com/leoribeiro/struc2vec/blob/master/graph/europe-airports.edgelist}} and USA\footnote{\url{https://github.com/leoribeiro/struc2vec/blob/master/graph/usa-airports.edgelist}} are two air-traffic networks in Europe and USA, with airports and routes between these airports abstracted as nodes and edges. Nodes are labeled based on the activity level of airports, which usually correspond to different roles, such as hub and non-hub, in the network.
Actor\footnote{\url{https://github.com/cspjiao/RONE/blob/main/dataset/clf/actor.edge}} and Film\footnote{\url{https://github.com/cspjiao/RONE/blob/main/dataset/clf/film.edge}}, constructed by Jiao et al.~\cite{jiao2021survey} using the dump of Wikipedia, are graphs recording the co-occurrence relations between actors and entities about films. Nodes in Actor and Film are labeled according to actors' influences and film entity types (e.g., film, director, and actor) relevant to node identities, respectively.

For the second type of datasets, PPI\footnote{\url{https://snap.stanford.edu/node2vec/Homo_sapiens.mat}} is a graph describing protein-protein interactions of Homo Sapiens, in which node labels are derived from hallmark gene sets and correspond to communities with consistent biological states.
BlogCatalog\footnote{\url{http://leitang.net/code/social-dimension/data/blogcatalog.mat}} is constructed from the social friend relations between bloggers on the BlogCatalog website, with node labels corresponding to several common interests of bloggers.
DBLP\footnote{\url{https://snap.stanford.edu/data/com-DBLP.html}} describes academic collaborations between paper authors, where each publication venue (famous journal or conference of a specific research field) defines an individual ground-truth community.
Amazon\footnote{\url{https://snap.stanford.edu/data/com-Amazon.html}} is collected by crawling the product co-purchasing relations
from the Amazon website, where each ground-truth community is defined as a unique product category on Amazon.

\subsection{Baselines}\label{App:Set-Base}
As summarized in Supplementary Table~\ref{Tab:Meth}, we compare PI-HIST with 18 classic and state-of-the-art unsupervised embedding baselines covering various categories, where P and I represent that a method is claimed to be position- and identity-preserving, respectively.

\begin{table}[t]\footnotesize 
\centering
\caption{Summary of unsupervised embedding methods to be evaluated, where P and I denote that a method is claimed to be position- and identity-preserving.}
\label{Tab:Meth}
\begin{tabular}{l|l|l|l|l}
\hline
\textbf{Methods} (Abbreviations) & \textbf{Publication Venues} & \begin{tabular}[c]{@{}l@{}}\textbf{Emb.}\\\textbf{Type}\end{tabular} & \begin{tabular}[c]{@{}l@{}}\textbf{Efficient}/\\ \textbf{Scalable}\end{tabular} & \begin{tabular}[c]{@{}l@{}}\textbf{GNN}\\\textbf{SSL}\end{tabular} \\ \hline
node2vec (n2v) & KDD 2016 \cite{grover2016node2vec} & P &  &  \\
PhUSION (PhN) (P) & SDM 2021 \cite{zhu2021node} & P &  &  \\
PaCEr (P) & Web Conf 2024 \cite{yan2024pacer} & P &  &  \\ \hline
struc2vec (s2v) & KDD 2017 \cite{ribeiro2017struc2vec} & I &  &  \\
PhUSION (PhN) (I) & SDM 2021 \cite{zhu2021node} & I &  &  \\
PaCEr (I) & Web Conf 2024 \cite{yan2024pacer} & I &  &  \\ \hline
RandNE & ICDM 2018 \cite{zhang2018billion} &  & \checkmark &  \\
LouvainNE (LvnNE) & WSDM 2020 \cite{bhowmick2020louvainne} & P & \checkmark &  \\
SketchNE (SktNE) & TKDE 2023 \cite{xie2023sketchne} & P & \checkmark &  \\
node2binary (n2b) & Web Conf 2025 \cite{talukder2025node2binary} & P & \checkmark &  \\ \hline
S3GC & NeurIPS 2022 \cite{devvrit2022s3gc} & P & \checkmark & \checkmark \\
MAGI & KDD 2024 \cite{liu2024revisiting} & P & \checkmark & \checkmark \\
GraLSP (GLSP) & AAAI 2020 \cite{jin2020gralsp} & I &  & \checkmark \\
DRSR & TCSS 2024 \cite{zhang2024disentangled} & I &  & \checkmark \\
DGI & ICLR 2019 \cite{velivckovicdeep} &  & \checkmark & \checkmark \\
GraphMAE2 (GMAE2) & KDD 2023 \cite{hou2023graphmae2} &  & \checkmark & \checkmark \\
GGD & NeurIPS 2022 \cite{zheng2022rethinking} &  & \checkmark & \checkmark \\
E2Neg & AAAI 2025 \cite{huang2025does} &  & \checkmark & \checkmark \\ \hline
\textbf{PI-HIST}~(R) & \textbf{Ours} & P & \checkmark &  \\
\textbf{PI-HIST}~(A) & \textbf{Ours} & I & \checkmark &  \\ \hline
\end{tabular}
\end{table}

Node2vec\footnote{\url{https://github.com/aditya-grover/node2vec}} and struc2vec\footnote{\url{https://github.com/sebkaz/struc2vec}} are classic position and identity embedding approaches.
PhUSION\footnote{\url{https://github.com/GemsLab/PhUSION}} and PaCEr are comprehensive methods investigating correlations between node positions and identities from different perspectives. They contain multiple variants for different types of embeddings. We use postfixes (P) and (I) to denote the corresponding variants for position and identity embeddings.

RandNE\footnote{\url{https://github.com/ZW-ZHANG/RandNE}}, LouvainNE\footnote{\url{https://github.com/maxdan94/LouvainNE}}, SketchNE\footnote{\url{https://github.com/THU-numbda/SketchNE}}, and node2binary\footnote{\url{https://github.com/ntalukde/node2binary}} are typical efficient embedding approaches using fast dimension reduction, including random projection \cite{arriaga2006algorithmic} and randomized matrix factorization (MF) \cite{martinsson2011randomized}.
LouvainNE, SketchNE, and node2binary should be position-preserving methods, since their embedding inference procedures are based on approximate position embedding objectives \cite{qiu2019netsmf} (e.g., sparse MF of DeepWalk \cite{perozzi2014deepwalk}) or classic community detection algorithms (e.g., Louvain \cite{blondel2008fast} and Leiden \cite{traag2019louvain}).

The $18$ baselines also cover GNN-based self-supervised learning (SSL), a representative state-of-the-art graph learning technique, including S3GC\footnote{\url{https://github.com/devvrit/S3GC}}, MAGI\footnote{\url{https://github.com/EdisonLeeeee/MAGI}}, GraLSP\footnote{\url{https://github.com/KL4805/GraLSP}}, DRSR\footnote{\url{https://github.com/cswzhang/DRSR}}, DGI\footnote{\url{https://github.com/PetarV-/DGI}}, GraphMAE2\footnote{\url{https://github.com/THUDM/GraphMAE2}}, GGD\footnote{\url{https://github.com/zyzisastudyreallyhardguy/Graph-Group-Discrimination}}, and E2Neg\footnote{\url{https://github.com/hedongxiao-tju/E2Neg}}.
Particularly, S3GC and MAGI can capture node positions via community-preserving graph contrastive objectives, while GraLSP and DRSR are claimed to be identity-preserving based on additional feature augmentation. Moreover, S3GC, MAGI, DGI, GraphMAE2, GGD, and E2Neg are scalable methods that can handle the embedding inference on large-scale graphs.

Note that most of these GNN-based approaches are originally designed for attributed graphs, while this article assumes that graph attributes are unavailable. To enable these baselines to support the embedding without attributes, we try $3$ commonly used standard settings that use (\romannumeral1) random noise, (\romannumeral2) constant, and (\romannumeral3) one-hot degree encoding as auxiliary attribute inputs, with the best inference quality reported.

\subsection{Inference Tasks and Evaluation Protocols}\label{App:Set-Task}
This article adopts $7$ inference tasks spanning node, edge, and graph levels to validate PI-HIST's effectiveness.

\subsubsection{Node-Level Tasks}

As summarized in Fig.~\ref{Fig:NC_Eva}a, we first measure the embedding quality of PI-HIST~(R) and (A) using $4$ node-level (position- and identity-related) tasks to validate whether they can effectively encode node positions and identities. In addition to embedding quality, we also evaluate their efficiency to demonstrate the potential advantage brought by PI-HIST's simple design, investigating the trade-off between quality and efficiency.

\textbf{Node Position and Identity Classification}.
For the first (or second) type of datasets with identity (or position) ground-truth, the supervised node classification can be used to evaluate the quality of identity (or position) embeddings.
We follow the standard evaluation pipeline\footnote{\url{https://github.com/THU-numbda/SketchNE/blob/master/example/predict.py}} of Xie et al.~\cite{xie2023sketchne} to facilitate this setup, where Logistic regression is the downstream classifier taking the derived embeddings as inputs; macro and micro F1 scores are adopted as quality metrics.
On each dataset, $20$\% and $10$\% of labeled nodes are randomly sampled to form the training and validation sets, with the remaining $70$\% nodes as the test set.
We repeat this procedure, which includes data split and downstream evaluation, $10$ times and report the mean and standard deviation of each quality metric. In this setting, larger average F1 scores indicate better embedding quality.

\textbf{Community Detection}.
For the first type of datasets without identity ground-truth, community detection is a straightforward unsupervised position-related task.
We employ $K$Means as the downstream clustering module and adopt modularity \cite{newman2006modularity}, a widely used metric, to measure the quality of position embeddings.
The evaluation procedure is also repeated $10$ times, with the mean and standard deviation of the quality metric reported. Hence, larger average modularity means better embedding quality.

\textbf{Node Identity Clustering}.
To enable the evaluation of identity embedding on the second type of datasets without identity ground-truth, we define node identity clustering, an unsupervised identity-related task, extended from the classic spectral clustering algorithm~\cite{von2007tutorial}.
Since high-order degree statistics can characterize node identities, this task first constructs an auxiliary similarity graph $G_D = (V, E_D)$, which is the first step of spectral clustering, using one-hot degree encodings of nodes' high-order neighbors.
$G_D$ has the same node set $V$ with the original graph $G = (V, E)$ but a different edge set $E_D$.
Let $N_l(v_i)$ be the set of neighbors with exact distance $l$ from $v_i$ and $\delta_l (v_i)$ be the one-hot degree encoding of $N_l(v_i)$. We first extract $\delta_{1:r} (v_i) := [\delta_1 (v_i), \cdots, \delta_r (v_i)]$ for each node $v_i$, with $r$ as the order of neighbors. The K-D Tree algorithm \cite{bentley1975multidimensional} is then applied to $\{ \delta_{1:r} (v_i) : v_i \in V\}$ to find top-$10$ nodes with the highest similarity for each node $v_i$, which constructs $G_D$.
By jointly considering the effectiveness and computational overhead of building $G_D$, we set $r \in ( 5, 5, 3, 3 )$ for PPI, BlogCatalog, DBLP, and Amazon, respectively.
Given embeddings $\{ f(v_i) \}$ derived from $G$, we employ $K$Means as the downstream clustering module to get a clustering result $C := (C_1, C_2, \cdots, C_K)$, with $C_s \subseteq V$ as the set of members in the $s$-th cluster and ${C_s} \cap {C_t} = \emptyset$ ($\forall s \ne t$).
We use average conductance \cite{mizutani2021improved}, a classic objective of spectral clustering, to measure the clustering quality based on $G_D$, which equivalently evaluates the ability of $\{ f(v_i) \}$ to encode node identities.
Given the clustering result $C = (C_1, \cdots, C_K)$ and auxiliary similarity graph $G_D = (V, E_D)$, the average conductance is defined as
\begin{equation}
    {\mathop{\rm Cond}\nolimits} (C;{G_D}): = \frac{1}{K}\sum\limits_{s = 1}^K \min \{ {\frac{{|{E_D}({C_s},V\backslash {C_s})|}}{{\sum\nolimits_{{v_i} \in {C_s}} {{\kappa _i}} }}}, 1 \},
\end{equation}
where ${E_D}({C_s},V\backslash {C_s}): = \{ ({v_i},{v_j}) \in {E_D}:{v_i} \in {C_s},{v_j} \in V\backslash {C_s}\}$ is the set of edges in $G_D$ across a cluster $C_s$ and its complement; $\kappa _i$ denotes the degree of node $v_i$ in $G_D$.
For each cluster $C_s$, this metric computes the ratio between the number of edges in $G_D$ across $C_s$ and the number of internal edges within $C_s$. When a cluster $C_s$ is empty, we simply set this ratio to $1$ to avoid the divide-by-zero exception.
Conductance is a lower-is-better metric, with a smaller value corresponding to better clustering quality.
Similar to community detection, we also repeat the aforementioned evaluation procedure $10$ times, with the mean and standard deviation of the quality metric reported.

\textbf{Efficiency Evaluation}. In addition to embedding quality, we also measure the effective computation time (sec) of embedding inference for efficiency evaluation.
Concretely, we report the mean and standard deviation of embedding inference time over $10$ independent runs, following $5$ warm-up runs. Shorter inference time corresponds to better efficiency.
On large-scale graphs, we define that a method encounters the OOT exception, if it fails to derive feasible embeddings within $1 \times 10^4$ seconds.

\subsubsection{Edge-Level Tasks}
As highlighted in Fig.~\ref{Fig:GRLP_Eva}a, we use link prediction and graph reconstruction, which are neither typically position- nor identity-dominant edge-level tasks, to validate whether node positions and identities are complementary topology properties and may together improve the inference quality.
For both tasks, we follow the standard evaluation pipeline\footnote{\url{https://github.com/AnryYang/NRP-code/blob/master/eval/eval_linkpred.py}} of Yang et al.~\cite{yang2020homogeneous}, with logistic regression as the downstream classifier and AUC as the quality metric. A larger AUC value implies better inference quality.

\textbf{Link Prediction}.
Given the observed graph topology, link prediction aims to determine whether there are potential edges between some non-adjacent node pairs.
We first randomly removed $20$\% of the edges $E$ in a given graph $G$, with the remaining $80$\% edges constituting the set of positive training samples $E_{\rm{trn}}$. Specifically, we let $E_{\rm{trn}}$ preserve a spanning tree of $E$, which can be obtained by a depth-first search on $E$. It ensures that the training topology induced by $E_{\rm{trn}}$ is connected. Each method is allowed to derive its node-wise embeddings $\{ f(v_i) \}$ only based on this connected topology.
We further generate another negative sample set $N_{\rm{trn}}$ by sampling non-adjacent node pairs equal in number to $E_{\rm{trn}}$.
Finally, $S_{\rm{trn}} := E_{\rm{trn}} \cup N_{\rm{trn}}$ is used as the training set to train the downstream logistic regression. For each node pair $(v_i, v_j) \in S_{\rm{trn}}$, we get node-pair embedding $e (v_i, v_j) := [f(v_i) || f(v_j)]$ by concatenating corresponding node embeddings. The downstream classifier is trained using inputs $\{ e (v_i, v_j) \}$ and labels $\{ l_{ij} \}$, where $l_{ij} = 1$ if $(v_i, v_j) \in E_{\rm{trn}}$ and $l_{ij} = 0$ otherwise.

The removed $20$\% edges $E \backslash E_{\rm{trn}}$ are split evenly to form positive samples of the validation and test sets, denoted by $E_{\rm{val}}$ and $E_{\rm{tst}}$.
Similar to $S_{\rm{trn}}$, non-adjacent node pairs matching sizes of $E_{\rm{val}}$ and $E_{\rm{tst}}$ are also sampled to build negative sample sets $N_{\rm{val}}$ and $N_{\rm{tst}}$. They are further used to construct the validation and test sets, with definitions of $S_{\rm{val}} := E_{\rm{val}} \cup N_{\rm{val}}$ and $S_{\rm{tst}} := E_{\rm{tst}} \cup N_{\rm{tst}}$.
In this setting, we adopt $80$/$10$/$10$\% for the split of training, validation, and test sets.
We repeat the aforementioned procedure, including data split, training, and inference, $5$ times, with the mean and standard deviation of the quality metric reported.

\textbf{Graph Reconstruction}.
Consistent with the motivation of embedding, graph reconstruction measures whether the derived low-dimensional vector representations $\{ f(v_i) \}$ can recover the original graph topology $G = (V, E)$.
Let $N := |N|$ be the number of nodes in $G$. Directly computing the reconstruction scores for all $N^2$ node pairs may be intractable for large-scale graphs. Instead, we randomly sample a set $S \subset V \times V$ containing $\mu N^2$ node pairs, where we set the ratio $\mu \in ( 10^{-1}, 10^{-1}, 10^{-1}, 10^{-2}, 10^{-2}, 10^{-3}, 10^{-5}, 10^{-5})$ for Europe, USA, PPI, Actor, BlogCatalog, Film, DBLP, and Amazon with progressively growing scale.
Different from link prediction, each method should derive its node-wise embeddings $\{ f(v_i) \}$ based on the full graph topology $(V, E)$.
For each node pair $(v_i, v_j) \in S$, we derive a node pair embedding $e (v_i, v_j) := [f(v_i) || f(v_j)]$ and a corresponding label $l_{ij}$, where $l_{ij} = 1$ if $(v_i, v_j) \in E$ and $l_{ij} = 0$ otherwise.
We then train the downstream classifier using input-label pairs $\{ (e(v_i, v_j), l_{ij}) \}$ and measure the ability to recover the original topology based on classifier scores with respect to $S$.
This evaluation procedure is repeated $10$ times, with the mean and standard deviation of the quality metric reported.

\subsubsection{Graph-Level Tasks}
We also investigate whether the combination of position and identity embeddings can improve the inference of some graph-level tasks, compared with only using one type of embeddings.
Motivated by the case studies of Milo et al.~\cite{milo2002network} and Dong et al.~\cite{ugander2012structural}, we consider the graph superfamily identification on both real datasets and synthetic benchmarks.

\textbf{Graph Superfamily Identification}.
This task evaluates whether the given graph-level features or embeddings can group structurally similar graphs into a common superfamily and separate those from different superfamilies.
For a case study, we visualize the correlation coefficient matrix derived from the graph-level embeddings $\{ g(G) \}$ of a method.
One can qualitatively assess inference quality by examining whether entries in main-diagonal blocks corresponding to ground-truth superfamilies have notably high correlation values.

For the $8$ real-world graphs used in our evaluation of node- and edge-level tasks, we divide them into $4$ rough superfamilies according to their scenarios (cf. Supplementary Table~\ref{Tab:Data}) and rearrange the layout order of these graphs based on this partition (Fig.~\ref{Fig:GSI_Eva}a) when visualizing the correlation matrix.
Concretely, Europe and USA are grouped into a common superfamily as they are both air traffic networks, while Film and Actor constitute a separate superfamily about movies.
We further cluster BlogCatalog, DBLP, and Amazon into another separate superfamily, due to their common application domains regarding human interactions.
PPI alone forms an independent superfamily since it is the only dataset related to proteins.

As highlighted in Fig.~\ref{Fig:GSI_Eva}b, we also adopt the classic LFR benchmark \cite{lancichinetti2008benchmark} to facilitates the superfamily identification on synthetic graphs. This benchmark uses a set of parameters $(N, \theta, k_{\rm{avg}}, k_{\max}, c_{\min}, c_{\max})$ to generate each synthetic graph, where $N$ is the number of nodes; $\theta$ is the ratio between the external degree and total degree of each node $v_i$ with respect to the cluster $v_i$ belongs to; $k_{\rm{avg}}$ and $k_{\max}$ denote the average and maximum node degrees; $c_{\min}$ and $c_{\max}$ represent the minimum and maximum community sizes.
Concretely, we fix $(N, \theta, k_{\max}, c_{\max}) = (1000, 0.1, 50, 500)$ and set $(k_{\rm{avg}}, c_{\min}) = (10, 20)$, $(10, 40)$, $(10, 60)$, $(10, 80)$, $(10, 100)$, $(9, 100)$, $(8, 100)$, $(7, 100)$, and $(6, 100)$, which in sequence generate $9$ synthetic graphs $(G_1, G_2, \cdots, G_9)$.
Note that node positions and identities are related to community structures and (high-order) degrees.
Under the above parameter settings, $(G_1, \cdots, G_5)$ and $(G_5, \cdots, G_9)$ mimic variations in community size and degree, thus correspondingly simulating gradual changes of node positions and identities.
Based on the layout order of synthetic graphs (i.e., from $G_1$ to $G_9$), entries within a corelation matrix are expected to decrease as they move away from the main diagonal (i.e., along the direction of purple arrows in Fig.~\ref{Fig:GSI_Eva}d and Supplementary Fig.~\ref{Fig:GSI_Eva_Apx}b), with a larger magnitude of decrease indicating a better identification result.
As the generation of synthetic graphs may not be stable, we repeat the generation procedure of $(G_1, \cdots, G_9)$ $10$ times and report the averaged correlation matrix over these independent runs.

Furthermore, we follow the evaluation pipeline of Dong et al.~\cite{dong2017structural} to compare PI-HIST with four baselines that can support graph superfamily identification, including BoD, BoHD, SSP \cite{milo2004superfamilies}, and CNS \cite{dong2017structural}.
BoD and BoHD are simple yet classic graph-level features. They characterize a graph using a fixed-dimensional histogram vector obtained by counting the frequency of each distinct (high-order) degree. For BoHD, we set the neighbor order as $r = 5$ for its degree counting.
SSP and CNS are sophisticated graph-level features about the profile of motifs \cite{milo2002network} (over the null model) and signature of structural diversity \cite{ugander2012structural}.
For SSP, we set the motif size to $4$ and thus consider all the $6$ (undirected) motifs induced by four nodes (cf. examples in Fig.~3 of \cite{milo2002network}).
Same as the settings in \cite{ugander2012structural}, we consider all the $17$ cases for structural diversity (cf. examples in Fig.~2 of \cite{ugander2012structural}) induced by two-, three-, and four-node common neighbors.
As official implementations of SSP and CNS are unavailable and directly computing exact statistics for motifs and structural diversity is intractable, we implement alternative approximate algorithms for them using \texttt{Numba}.
Our implementations estimate motif and structural diversity statistics by sampling only a small ratio of node combinations, where the computationally expensive sampling-and-estimation procedure can be sped up using GPUs.

\subsection{Model Configuration and Hyper-parameter Settings}\label{App:Set-Param}

Almost all the methods evaluated in our experiments contain tunable hyper-parameters. For instance, node2vec introduces $\{ p, q\}$ to balance the width- and depth-first search in RW sampling \cite{grover2016node2vec}.
In PI-HIST, $\{ \varepsilon, L\}$ are hyper-parameters regarding the single FFP for embedding derivation. One can also select settings of the nonlinearity activation and normalization in PI-HIST using a strategy similar to hyper-parameter tuning.
On each dataset, we set the same embedding dimensionality $d$ for all methods, except for those with a fixed default setting, such as PhUSION (I) with $d=250$ for all cases.
For supervised tasks (e.g., node classification and link prediction), we tune hyper-parameters and other optional setups (e.g., different augmented feature inputs of GNN-based approaches mentioned in Supplementary Section~\ref{App:Set-Base}) based on quality metrics on the validation set.
For unsupervised tasks (e.g., community detection and node identity clustering), we determine corresponding settings according to unsupervised metrics (e.g., modularity and conductance).
For case studies of graph superfamily identification, we manually select the setting with the fewest failure cases in the visualized correlation matrix.

Our node-level evaluations(Fig.~\ref{Fig:NC_Eva} and Supplementary Fig.~\ref{Fig:NC_Eva_Apx}) involve both position- and identity-related tasks, such as node identity classification and community detection on the first type of datasets.
On each dataset, we tune hyper-parameters and optional settings of position- and identity-preserving approaches (cf. the summary of baselines in Supplementary Table~\ref{Tab:Meth}) based on the corresponding position- and identity-related tasks, respectively.
For the rest baselines, which are uncertain to encode which property, we use supervised tasks (i.e., identity and position classification on the first and second types of datasets) to select their settings.

Following the tuning strategy described in `Methods' section, detailed model configurations and hyper-parameters settings for all the experimental evaluation cases are summarized in Supplementary Tables~\ref{Tab:Param} and \ref{Tab:Param-RA}, where $d$ is the embedding dimensionality; $L$ represents the RW/AW length; $\varepsilon$ is the aggregation weight of FFP, while $\alpha$ is the fusion weight when combining position and identity embeddings; $n$ and $\hat n$ are the numbers of RWs sampled per node in total and in each run with $\hat n \le n$; $\psi_{\rm act} (\cdot)$ and $\psi_{\rm norm} (\cdot)$ are optional nonlinear activation and normalization after FFP; $\gamma_{\rm norm} (\cdot)$ denotes the optional normalization of embedding fusion.
When applying PI-HIST~(A) to the superfamily identification over real-world graphs, we set $(n, \hat n)$ to (5e4, 5e4), (5e4, 5e4), (5e4, 5e4), (5e4, 5e4), (5e4, 2.5e4), (5e4, 1e4), (5e4, 5e2), and (5e4, 5e2) for Europe, USA, PPI, Actor, BlogCatalog, Film, DBLP, and Amazon, respectively.

\begin{table}[]\tiny
\centering
\caption{Detailed parameter settings and configurations of PI-HIST~(R) and (A) for all the experimental evaluation cases, where $d$ is the embedding dimensionality; $L$ denotes the RW/AW length; $\varepsilon$ is the aggregation weight for one FFP; $\psi_{\rm act} (\cdot)$ and $\psi_{\rm norm} (\cdot)$ represent optional nonlinear activation function and normalization operation; $n$ and $\hat n$ are the numbers of RWs sampled per node in total and in each run ($\hat n \le n$).}
\label{Tab:Param}
\begin{tabular}{c|c|l|l|l|l|l|l|l}
\hline
\textbf{Tasks} & \multicolumn{1}{l|}{\textbf{Variants}} & \textbf{Datasets} & $d$ & $L$ & $\varepsilon$ & $\psi_{\rm{act}} (\cdot)$ & $\psi_{\rm{norm}} (\cdot)$ & ($n$, $\hat n$) \\ \hline
\multirow{2}{*}{\textbf{Case Study in Fig.~1}} & \textbf{PI-HIST}~(R) & \textbf{Karate Club} & 16 & 8 & 0.1 & tanh & col z-norm & N/A \\ \cline{2-9}
 & \textbf{PI-HIST}~(A) & \textbf{Karate Club} & 16 & 8 & 0.9 & tanh & col z-norm & 1e3,1e3 \\ \hline
\multirow{16}{*}{\textbf{Node-Level Tasks}} & \multirow{8}{*}{\textbf{PI-HIST}~(R)} & \textbf{Europe} & 64 & 5 & 0.3 & $\emptyset$ & row l2-norm & N/A \\
 &  & \textbf{USA} & 64 & 10 & 0.7 & $\emptyset$ & row l2-norm & N/A \\
 &  & \textbf{Actor} & 256 & 15 & 0.3 & $\emptyset$ & row l2-norm & N/A \\
 &  & \textbf{Film} & 256 & 12 & 0.1 & ReLU & row l2-norm & N/A \\
 &  & \textbf{PPI} & 256 & 10 & 0.2 & tanh & $\emptyset$ & N/A \\
 &  & \textbf{BlogCatalog} & 512 & 6 & 0.1 & $\emptyset$ & $\emptyset$ & N/A \\
 &  & \textbf{DBLP} & 256 & 20 & 0.0 & tanh & col z-norm & N/A \\
 &  & \textbf{Amazon} & 128 & 20 & 0.1 & tanh & $\emptyset$ & N/A \\ \cline{2-9} 
 & \multirow{8}{*}{\textbf{PI-HIST}~(A)} & \textbf{Europe} & 64 & 5 & 0.9 & $\emptyset$ & $\emptyset$ & 5e4,5e4 \\
 &  & \textbf{USA} & 64 & 5 & 0.3 & tanh & col z-norm & 5e4,5e4 \\
 &  & \textbf{Actor} & 256 & 5 & 0.9 & ReLU & $\emptyset$ & 5e4,5e4 \\
 &  & \textbf{Film} & 256 & 7 & 0.4 & ReLU & col z-norm & 5e4,1e4 \\
 &  & \textbf{PPI} & 256 & 8 & 0.6 & sigmoid & row l2-norm & 5e4,5e4 \\
 &  & \textbf{BlogCatalog} & 512 & 8 & 0.7 & ReLU & $\emptyset$ & 5e4,2.5e4 \\
 &  & \textbf{DBLP} & 256 & 8 & 0.9 & ReLU & col z-norm & 1e4,5e2 \\
 &  & \textbf{Amazon} & 128 & 7 & 0.1 & ReLU & $\emptyset$ & 1e4,5e2 \\ \hline
\multirow{16}{*}{\textbf{Link Prediction}} & \multirow{8}{*}{\textbf{PI-HIST}~(R)} & \textbf{Europe} & 64 & 5 & 0.0 & ReLU & col z-norm & N/A \\
 &  & \textbf{USA} & 64 & 8 & 0.9 & ReLU & col z-norm & N/A \\
 &  & \textbf{PPI} & 256 & 5 & 0.6 & ReLU & col z-norm & N/A \\
 &  & \textbf{Actor} & 256 & 5 & 0.7 & ReLU & col z-norm & N/A \\
 &  & \textbf{BlogCatalog} & 512 & 5 & 0.0 & ReLU & col z-norm & N/A \\
 &  & \textbf{Film} & 256 & 5 & 0.6 & ReLU & col z-norm & N/A \\
 &  & \textbf{DBLP} & 256 & 5 & 0.7 & exp & col z-norm & N/A \\
 &  & \textbf{Amazon} & 128 & 5 & 0.6 & exp & col z-norm & N/A \\ \cline{2-9} 
 & \multirow{8}{*}{\textbf{PI-HIST}~(A)} & \textbf{Europe} & 64 & 7 & 0.2 & sigmoid & col z-norm & 5e4,5e4 \\
 &  & \textbf{USA} & 64 & 8 & 0.9 & ReLU & col z-norm & 5e4,5e4 \\
 &  & \textbf{PPI} & 256 & 8 & 0.5 & tanh & col z-norm & 5e4,5e4 \\
 &  & \textbf{Actor} & 256 & 6 & 0.3 & ReLU & col z-norm & 5e4,5e4 \\
 &  & \textbf{BlogCatalog} & 512 & 5 & 0.3 & ReLU & col z-norm & 5e4,5e4 \\
 &  & \textbf{Film} & 256 & 6 & 0.4 & ReLU & col z-norm & 5e4,1e4 \\
 &  & \textbf{DBLP} & 256 & 5 & 0.8 & ReLU & col z-norm & 1e4,1e3 \\
 &  & \textbf{Amazon} & 128 & 5 & 0.6 & exp & col z-norm & 1e4,1e3 \\ \hline
\multirow{16}{*}{\textbf{Graph Reconstruction}} & \multirow{8}{*}{\textbf{PI-HIST}~(R)} & \textbf{Europe} & 64 & 5 & 0.0 & ReLU & col z-norm & N/A \\
 &  & \textbf{USA} & 64 & 5 & 0.5 & ReLU & col z-norm & N/A \\
 &  & \textbf{PPI} & 256 & 7 & 0.7 & exp & col z-norm & N/A \\
 &  & \textbf{Actor} & 256 & 6 & 0.8 & exp & col z-norm & N/A \\
 &  & \textbf{BlogCatalog} & 512 & 5 & 0.0 & tanh & col z-norm & N/A \\
 &  & \textbf{Film} & 256 & 5 & 0.8 & exp & col z-norm & N/A \\
 &  & \textbf{DBLP} & 256 & 5 & 0.5 & ReLU & col z-norm & N/A \\
 &  & \textbf{Amazon} & 128 & 6 & 0.0 & tanh & col z-norm & N/A \\ \cline{2-9} 
 & \multirow{8}{*}{\textbf{PI-HIST}~(A)} & \textbf{Europe} & 64 & 8 & 0.2 & ReLU & col z-norm & 5e4,5e4 \\
 &  & \textbf{USA} & 64 & 6 & 0.3 & ReLU & col z-norm & 5e4,5e4 \\
 &  & \textbf{PPI} & 256 & 8 & 0.3 & ReLU & col z-norm & 5e4,5e4 \\
 &  & \textbf{Actor} & 256 & 7 & 0.3 & ReLU & col z-norm & 5e4,5e4 \\
 &  & \textbf{BlogCatalog} & 512 & 5 & 0.3 & ReLU & col z-norm & 5e4,5e4 \\
 &  & \textbf{Film} & 256 & 5 & 0.9 & ReLU & col z-norm & 5e4,1e4 \\
 &  & \textbf{DBLP} & 256 & 6 & 0.2 & ReLU & col z-norm & 1e4,1e3 \\
 &  & \textbf{Amazon} & 128 & 7 & 0.2 & ReLU & col z-norm & 1e4,1e3 \\ \hline
\multirow{4}{*}{\textbf{Graph Superfamily Identification}} & \multirow{2}{*}{\textbf{PI-HIST}~(R)} & \textbf{Syn}.~\textbf{Graphs} & 64 & 5 & 0.8 & ReLU & $\emptyset$ & N/A \\
 &  & \textbf{Real}~\textbf{Graphs} & 64 & 5 & 0.9 & ReLU & $\emptyset$ & N/A \\ \cline{2-9} 
 & \multirow{2}{*}{\textbf{PI-HIST}~(A)} & \textbf{Syn}.~\textbf{Graphs} & 64 & 6 & 0.5 & ReLU & $\emptyset$ & 1e4,1e4 \\
 &  & \textbf{Real}~\textbf{Graphs} & 64 & 7 & 0.8 & ReLU & $\emptyset$ & 5e4/2e4,5e4/2.5e4/1e4/5e2 \\ \hline
\end{tabular}
\end{table}

\begin{table}[]\tiny
\centering
\caption{Detailed parameter settings of PI-HIST~(R\&A) for cases combining both variants of PI-HIST, where $\alpha$ is the fusion weight; $\gamma_{\rm{norm}} (\cdot)$ is the optional normalization operation.}
\label{Tab:Param-RA}
\begin{tabular}{c|l|l|l}
\hline
\multicolumn{1}{c|}{\textbf{Tasks}} & \textbf{Datasets} & $\alpha$ & $\gamma_{\rm norm} (\cdot)$ \\ \hline
\multirow{8}{*}{\textbf{Link Prediction}} & \textbf{Europe} & 0.9 & row l2-norm + col z-norm \\
 & \textbf{USA} & 0.9 & row l2-norm + col z-norm \\
 & \textbf{PPI} & 0.6 & row l2-norm + col z-norm \\
 & \textbf{Actor} & 0.8 & row l2-norm + col z-norm \\
 & \textbf{BlogCatalog} & 0.7 & row l2-norm + col z-norm \\
 & \textbf{Film} & 0.9 & row l2-norm + col z-norm \\
 & \textbf{DBLP} & 0.8 & row l2-norm + col z-norm \\
 & \textbf{Amazon} & 0.9 & row l2-norm + col z-norm \\ \hline
\multirow{8}{*}{\textbf{Graph Reconstruction}} & \textbf{Europe} & 0.6 & row z-norm \\
 & \textbf{USA} & 0.5 & row z-norm \\
 & \textbf{PPI} & 0.3 & row z-norm \\
 & \textbf{Actor} & 0.2 & row z-norm \\
 & \textbf{BlogCatalog} & 0.7 & row z-norm \\
 & \textbf{Film} & 0.9 & row z-norm \\
 & \textbf{DBLP} & 0.3 & row z-norm \\
 & \textbf{Amazon} & 0.5 & row z-norm \\ \hline
\multirow{2}{*}{\textbf{Graph Superfamily Identification}} & \textbf{Synthetic Graphs} & 0.6 & row l2-norm \\
 & \textbf{Real Graphs} & 0.9 & $\emptyset$ \\ \hline
\end{tabular}
\end{table}

\subsection{Experimental Environment}\label{App:Set-Env}
We use \texttt{PyTorch}, \texttt{PyTorch Cluster}, and \texttt{Numba} to implement PI-HIST, while the open-source official implementations of other baselines are adopted for our experimental comparison (cf. Supplementary Section~\ref{App:Set-Base} for their sources).
All the experiments were conducted on a rented cloud server running Ubuntu 22.04, equipped with an AMD EPYC 9K84 96-Core CPU, $150$GB main memory, and a single H20 GPU ($96$GB GPU memory).
Under this experimental environment, PI-HIST and all other learning-based approaches with official implementations in  \texttt{PyTorch} or \texttt{TensorFlow} are run on the GPU.

\section{Detailed Experiment Results}\label{App:Res}

Corresponding to the visualization results shown in Fig.~\ref{Fig:NC_Eva} and \ref{Fig:NC_Eva_Apx}, numerical details of node-level task evaluation on Europe, USA, Actor, Film, PPI, BlogCatalog, DBLP, and Amazon in terms of $m\pm s$ are depicted in Supplementary Tables~\ref{Tab:NC-Europe}, \ref{Tab:NC-USA}, \ref{Tab:NC-Actor}, \ref{Tab:NC-Film}, \ref{Tab:NC-PPI}, \ref{Tab:NC-BlogCatalog}, \ref{Tab:NC-DBLP}, and \ref{Tab:NC-Amazon}, with $m$ and $s$ as the mean and standard deviation of a metric over multiple independent runs.
Table~\ref{Tab:Time-A} further provides details about the inference time with respect to different steps of PI-HIST (A), where RW, AW-MAP, AW-IDX, AW-CNT, CAT, HIER, and FFP refer to (\romannumeral1) RW sampling, (\romannumeral2) AW mapping, (\romannumeral3) AW indexing, (\romannumeral4) AW counting, (\romannumeral5) statistic gathering, (\romannumeral6) hierarchical structure extraction, and (\romannumeral7) one training-free FFP, respectively.

Numerical details of the two edge-level tasks (i.e., link prediction and graph reconstruction) in terms of $m\pm s$, which correspond to the visualization results in Fig.~\ref{Fig:GRLP_Eva} and \ref{Fig:GRLP_Eva_Apx}, are shown in Supplementary Tables~\ref{Tab:LP} and \ref{Tab:GR}, where Val and Test denote outcomes on the validation and test sets, respectively.
Supplementary Table~\ref{Tab:RA-alpha} gives detailed results of PI-HIST (R\&A) (highlighted by purple lines in Fig.~\ref{Fig:GRLP_Eva} and \ref{Fig:GRLP_Eva_Apx}) with respect to $\alpha \in \{ 0.0, 0.1, \cdots, 1.0\}$.

\begin{table}[]\tiny
\centering
\caption{Detailed results of node-level tasks on Europe with node identity ground-truth.}
\label{Tab:NC-Europe}
\begin{tabular}{l|l|ll|ll|l}
\hline
\multirow{2}{*}{} & \textbf{Time} & \multicolumn{2}{c|}{\textbf{Validation Set}(\%)} & \multicolumn{2}{c|}{\textbf{Test Set}(\%)} & \textbf{Modularity} \\ \cline{3-6}
 & (sec)$\downarrow$ & \textbf{Micro~F1}$\uparrow$ & \textbf{Macro~F1}$\uparrow$ & \textbf{Micro~F1}$\uparrow$ & \textbf{Macro~F1}$\uparrow$ & (\%)$\uparrow$ \\ \hline
node2vec & 7.8438$\pm$0.3677 & 30.26$\pm$7.50 & 26.92$\pm$7.69 & 34.66$\pm$2.22 & 33.51$\pm$3.04 & 12.74$\pm$4.88 \\
PhUSION~(P) & 0.2892$\pm$0.0078 & 30.77$\pm$9.03 & 29.17$\pm$8.70 & 32.10$\pm$2.64 & 30.89$\pm$2.57 & 10.55$\pm$3.22 \\
PaCEr~(P) & 3.6484$\pm$0.2124 & 36.15$\pm$6.73 & 34.16$\pm$6.13 & 39.72$\pm$1.50 & 39.43$\pm$1.60 & 9.72$\pm$1.66 \\ \hline
struc2vec & 7.8520$\pm$0.1582 & 50.77$\pm$6.36 & 47.39$\pm$6.83 & 52.74$\pm$3.26 & 52.33$\pm$3.74 & -5.39$\pm$1.68 \\
PhUSION~(I) & 1.7452$\pm$0.0150 & 47.44$\pm$7.28 & 42.96$\pm$7.21 & 51.21$\pm$2.44 & 49.02$\pm$3.87 & -5.33$\pm$1.13 \\ 
PaCEr~(I) & 0.5868$\pm$0.0038 & 54.36$\pm$7.84 & 52.60$\pm$8.69 & 55.05$\pm$3.16 & 54.42$\pm$2.75 & -4.43$\pm$0.62 \\ \hline
RandNE & 0.0065$\pm$0.0003 & 49.74$\pm$5.76 & 48.00$\pm$5.77 & 49.47$\pm$2.15 & 48.67$\pm$2.05 & -1.58$\pm$0.88 \\
LouvainNE & 0.0106$\pm$0.0001 & 24.62$\pm$11.76 & 23.21$\pm$10.38 & 27.01$\pm$3.14 & 26.51$\pm$3.39 & 18.71$\pm$0.74 \\
SketchNE & 0.7962$\pm$0.0662 & 38.97$\pm$8.57 & 35.85$\pm$8.74 & 37.47$\pm$2.68 & 34.21$\pm$3.94 & 13.71$\pm$4.29 \\
node2binary & 4.6615$\pm$0.1551 & 28.21$\pm$5.38 & 26.24$\pm$5.32 & 32.63$\pm$1.84 & 31.67$\pm$2.00 & 14.35$\pm$3.45 \\ \hline
S3GC & 34.2107$\pm$0.6347 & 36.67$\pm$8.58 & 32.79$\pm$8.90 & 38.29$\pm$2.41 & 34.68$\pm$3.24 & 11.31$\pm$0.14 \\
MAGI & 1.2443$\pm$0.0297 & 38.72$\pm$9.06 & 32.28$\pm$10.07 & 38.93$\pm$3.60 & 33.97$\pm$5.31 & 13.65$\pm$1.70 \\
GraLSP & 28.9134$\pm$3.5807 & 41.03$\pm$7.34 & 38.69$\pm$7.66 & 45.05$\pm$2.76 & 43.45$\pm$2.82 & 0.43$\pm$0.55 \\
DRSR & 7.1868$\pm$0.0299 & 50.77$\pm$4.70 & 48.05$\pm$5.44 & 53.13$\pm$2.98 & 51.95$\pm$2.90 & -4.15$\pm$0.19 \\
DGI & 0.9755$\pm$0.2283 & 46.15$\pm$7.69 & 40.10$\pm$8.21 & 47.69$\pm$2.88 & 45.35$\pm$4.42 & 0.13$\pm$3.29 \\
GraphMAE2 & 1.6811$\pm$0.4101 & 40.77$\pm$9.49 & 38.23$\pm$8.35 & 45.20$\pm$1.90 & 43.87$\pm$2.09 & 0.80$\pm$0.35 \\
GGD & 1.4241$\pm$0.0419 & 54.10$\pm$9.42 & 51.20$\pm$10.41 & 54.91$\pm$2.12 & 53.45$\pm$3.36 & -5.56$\pm$0.06 \\
E2Neg & 6.5305$\pm$0.2581 & 53.08$\pm$7.87 & 50.61$\pm$5.61 & 50.89$\pm$2.67 & 50.85$\pm$2.51 & -2.38$\pm$4.90 \\ \hline
\textbf{PI-HIST}~(R) & 0.0018$\pm$0.0000 & 32.31$\pm$7.88 & 30.28$\pm$7.98 & 33.17$\pm$4.21 & 31.20$\pm$4.70 & 18.68$\pm$0.41 \\
\textbf{PI-HIST}~(A) & 0.0891$\pm$0.0021 & 56.67$\pm$9.62 & 54.32$\pm$10.03 & 56.65$\pm$2.49 & 55.54$\pm$2.84 & -5.33$\pm$5.40 \\ \hline
\end{tabular}
\end{table}

\begin{table}[]\tiny
\centering
\caption{Detailed results of node-level tasks on USA with node identity ground-truth.}
\label{Tab:NC-USA}
\begin{tabular}{l|l|ll|ll|l}
\hline
\multirow{2}{*}{} & \textbf{Time} & \multicolumn{2}{c|}{\textbf{Validation Set}(\%)} & \multicolumn{2}{c|}{\textbf{Test Set}(\%)} & \textbf{Modularity} \\ \cline{3-6}
 & (sec)$\downarrow$ & \textbf{Micro~F1}$\uparrow$ & \textbf{Macro~F1}$\uparrow$ & \textbf{Micro~F1}$\uparrow$ & \textbf{Macro~F1}$\uparrow$ & (\%)$\uparrow$ \\ \hline
node2vec & 26.8722$\pm$0.7963 & 50.51$\pm$4.31 & 49.14$\pm$4.77 & 49.33$\pm$1.42 & 48.61$\pm$1.60 & 26.94$\pm$2.73 \\
PhUSION~(P) & 1.7797$\pm$0.0207 & 49.49$\pm$2.37 & 48.05$\pm$2.82 & 50.26$\pm$1.42 & 49.25$\pm$1.70 & 24.24$\pm$1.88 \\
PaCEr~(P) & 7.3735$\pm$0.0828 & 51.95$\pm$3.68 & 51.61$\pm$3.63 & 50.88$\pm$1.35 & 50.74$\pm$1.24 & 22.46$\pm$0.44 \\ \hline
struc2vec & 32.6666$\pm$0.5896 & 55.85$\pm$5.23 & 54.62$\pm$4.36 & 55.46$\pm$0.97 & 54.60$\pm$1.18 & 4.21$\pm$2.71 \\
PhUSION~(I) & 12.3627$\pm$0.1339 & 57.88$\pm$2.91 & 56.39$\pm$3.27 & 59.00$\pm$1.57 & 58.06$\pm$2.03 & 12.74$\pm$0.05 \\ 
PaCEr~(I) & 4.1170$\pm$0.0230 & 56.10$\pm$4.28 & 54.00$\pm$5.51 & 56.53$\pm$1.83 & 55.13$\pm$1.84 & -0.34$\pm$1.08 \\ \hline
RandNE & 0.0140$\pm$0.0007 & 54.32$\pm$3.02 & 53.19$\pm$3.06 & 54.21$\pm$2.17 & 53.32$\pm$2.72 & 11.37$\pm$0.78 \\
LouvainNE & 0.0231$\pm$0.0009 & 45.68$\pm$3.50 & 44.00$\pm$3.78 & 45.25$\pm$1.68 & 43.94$\pm$2.80 & 27.80$\pm$0.34 \\
SketchNE & 0.8249$\pm$0.0607 & 54.07$\pm$3.41 & 52.16$\pm$4.30 & 53.03$\pm$1.33 & 51.81$\pm$1.83 & 27.94$\pm$2.32 \\
node2binary & 14.1784$\pm$1.6952 & 45.76$\pm$4.29 & 44.49$\pm$3.81 & 46.98$\pm$1.52 & 46.21$\pm$2.12 & 24.26$\pm$0.88 \\ \hline
S3GC & 29.1119$\pm$0.7194 & 53.56$\pm$2.78 & 51.35$\pm$3.95 & 51.84$\pm$1.89 & 49.97$\pm$2.26 & 26.78$\pm$0.08 \\
MAGI & 1.5214$\pm$0.0277 & 49.66$\pm$3.46 & 46.72$\pm$3.53 & 48.06$\pm$2.00 & 45.38$\pm$2.94 & 24.15$\pm$0.02 \\
GraLSP & 28.9351$\pm$3.7649 & 56.36$\pm$5.64 & 54.90$\pm$5.17 & 56.33$\pm$1.75 & 55.41$\pm$1.98 & 4.66$\pm$0.36 \\
DRSR & 7.8639$\pm$0.0938 & 60.17$\pm$4.82 & 58.25$\pm$4.83 & 60.77$\pm$1.87 & 59.66$\pm$1.69 & 11.74$\pm$0.69 \\
DGI & 1.4196$\pm$0.3677 & 56.86$\pm$3.04 & 55.33$\pm$1.87 & 56.46$\pm$1.44 & 55.01$\pm$1.65 & 11.28$\pm$3.01 \\
GraphMAE2 & 1.6050$\pm$0.4592 & 51.44$\pm$6.08 & 49.19$\pm$6.53 & 51.64$\pm$2.41 & 49.72$\pm$2.54 & 2.58$\pm$0.63 \\
GGD & 1.3626$\pm$0.0025 & 56.10$\pm$3.99 & 55.67$\pm$3.77 & 56.77$\pm$1.67 & 56.50$\pm$1.66 & 11.34$\pm$6.20 \\
E2Neg & 6.5868$\pm$0.1357 & 55.93$\pm$3.75 & 55.51$\pm$3.79 & 57.21$\pm$1.29 & 56.96$\pm$1.23 & 8.34$\pm$3.12 \\ \hline
\textbf{PI-HIST}~(R) & 0.0058$\pm$0.0042 & 50.51$\pm$2.25 & 48.47$\pm$3.27 & 50.36$\pm$2.30 & 48.93$\pm$2.77 & 27.22$\pm$1.17 \\
\textbf{PI-HIST}~(A) & 0.1313$\pm$0.0064 & 61.86$\pm$3.13 & 59.85$\pm$4.34 & 60.23$\pm$1.13 & 58.66$\pm$1.41 & 23.61$\pm$0.53 \\ \hline
\end{tabular}
\end{table}

\begin{table}[]\tiny
\centering
\caption{Detailed results of node-level tasks on Actor with node identity ground-truth.}
\label{Tab:NC-Actor}
\begin{tabular}{l|l|ll|ll|l}
\hline
\multirow{2}{*}{} & \textbf{Time} & \multicolumn{2}{c|}{\textbf{Validation Set}(\%)} & \multicolumn{2}{c|}{\textbf{Test Set}(\%)} & \textbf{Modularity} \\ \cline{3-6}
 & (sec)$\downarrow$ & \textbf{Micro~F1}$\uparrow$ & \textbf{Macro~F1}$\uparrow$ & \textbf{Micro~F1}$\uparrow$ & \textbf{Macro~F1}$\uparrow$ & (\%)$\uparrow$ \\ \hline
node2vec & 359.6592$\pm$11.4928 & 28.96$\pm$1.99 & 25.44$\pm$1.61 & 28.81$\pm$0.65 & 25.24$\pm$0.55 & 46.23$\pm$0.66 \\
PhUSION~(P) & 72.2466$\pm$0.9959 & 31.71$\pm$2.58 & 27.28$\pm$2.06 & 31.54$\pm$0.40 & 27.11$\pm$0.33 & 46.30$\pm$0.83 \\
PaCEr~(P) & 172.2412$\pm$0.4037 & 35.13$\pm$1.67 & 31.04$\pm$1.46 & 34.75$\pm$0.50 & 30.97$\pm$0.43 & 46.44$\pm$0.34 \\ \hline
struc2vec & 292.1057$\pm$7.5681 & 39.93$\pm$1.60 & 33.44$\pm$1.46 & 40.03$\pm$0.35 & 33.78$\pm$0.42 & 0.80$\pm$0.70 \\
PhUSION~(I) & 476.5349$\pm$3.7501 & 40.40$\pm$1.53 & 32.93$\pm$1.32 & 40.82$\pm$0.69 & 33.49$\pm$1.15 & 6.62$\pm$0.80 \\ 
PaCEr~(I) & 150.3713$\pm$0.4370 & 41.50$\pm$1.88 & 32.52$\pm$2.07 & 41.77$\pm$0.63 & 32.94$\pm$1.64 & -4.82$\pm$0.01 \\ \hline
RandNE & 3.7662$\pm$0.1607 & 30.30$\pm$1.40 & 24.82$\pm$0.81 & 30.19$\pm$1.11 & 24.66$\pm$0.95 & 7.05$\pm$2.71 \\
LouvainNE & 0.3428$\pm$0.0045 & 28.23$\pm$1.95 & 21.92$\pm$2.28 & 27.82$\pm$0.88 & 21.58$\pm$0.79 & 46.28$\pm$1.28 \\
SketchNE & 1.3812$\pm$0.0536 & 31.20$\pm$1.76 & 25.43$\pm$1.27 & 31.51$\pm$0.61 & 25.83$\pm$0.49 & 46.16$\pm$0.41 \\
node2binary & 128.7087$\pm$25.8684 & 27.25$\pm$1.74 & 23.04$\pm$1.38 & 26.51$\pm$0.58 & 22.33$\pm$0.62 & 28.06$\pm$0.92 \\ \hline
S3GC & 32.2573$\pm$0.5685 & 31.37$\pm$1.62 & 25.62$\pm$1.16 & 31.61$\pm$0.44 & 25.92$\pm$0.39 & 44.25$\pm$0.01 \\
MAGI & 21.7904$\pm$0.5835 & 37.02$\pm$1.59 & 29.36$\pm$1.85 & 37.72$\pm$0.68 & 30.05$\pm$1.29 & 33.00$\pm$0.00 \\
GraLSP & 249.4036$\pm$43.8074 & 41.36$\pm$2.54 & 33.07$\pm$2.54 & 41.55$\pm$0.38 & 33.35$\pm$0.84 & 4.58$\pm$0.06 \\
DRSR & 11.6590$\pm$0.1086 & 43.29$\pm$1.55 & 35.31$\pm$1.07 & 43.72$\pm$0.46 & 35.78$\pm$0.45 & 10.90$\pm$1.70 \\
DGI & 2.3272$\pm$0.3494 & 42.37$\pm$2.18 & 33.77$\pm$1.64 & 42.14$\pm$0.59 & 33.77$\pm$0.76 & 5.69$\pm$1.11 \\
GraphMAE2 & 4.7755$\pm$2.6565 & 39.22$\pm$1.57 & 31.50$\pm$0.97 & 39.13$\pm$0.59 & 31.47$\pm$1.01 & 0.63$\pm$0.47 \\
GGD & 1.6110$\pm$0.0199 & 40.58$\pm$1.87 & 34.50$\pm$1.52 & 40.33$\pm$0.42 & 34.07$\pm$0.32 & 5.42$\pm$2.13 \\
E2Neg & 7.2887$\pm$0.1599 & 41.29$\pm$2.09 & 33.87$\pm$1.30 & 41.00$\pm$0.90 & 33.82$\pm$0.61 & 1.55$\pm$1.09 \\ \hline
\textbf{PI-HIST}~(R) & 0.0078$\pm$0.0001 & 29.67$\pm$1.36 & 23.95$\pm$1.15 & 29.38$\pm$0.62 & 23.69$\pm$0.61 & 45.16$\pm$0.01 \\
\textbf{PI-HIST}~(A) & 0.6730$\pm$0.0099 & 42.75$\pm$0.99 & 34.51$\pm$0.97 & 43.39$\pm$0.57 & 35.18$\pm$0.76 & 2.28$\pm$0.00 \\ \hline
\end{tabular}
\end{table}

\begin{table}[]\tiny
\centering
\caption{Detailed results of node-level tasks on Film with node identity ground-truth.}
\label{Tab:NC-Film}
\begin{tabular}{l|l|ll|ll|l}
\hline
\multirow{2}{*}{} & \textbf{Time} & \multicolumn{2}{c|}{\textbf{Validation Set}(\%)} & \multicolumn{2}{c|}{\textbf{Test Set}(\%)} & \textbf{Modularity} \\ \cline{3-6}
 & (sec)$\downarrow$ & \textbf{Micro~F1}$\uparrow$ & \textbf{Macro~F1}$\uparrow$ & \textbf{Micro~F1}$\uparrow$ & \textbf{Macro~F1}$\uparrow$ & (\%)$\uparrow$ \\ \hline
node2vec & 359.6592$\pm$11.4928 & 28.96$\pm$1.99 & 25.44$\pm$1.61 & 28.81$\pm$0.65 & 25.24$\pm$0.55 & 46.23$\pm$0.66 \\
PhUSION~(P) & 72.2466$\pm$0.9959 & 31.71$\pm$2.58 & 27.28$\pm$2.06 & 31.54$\pm$0.40 & 27.11$\pm$0.33 & 46.30$\pm$0.83 \\
PaCEr~(P) & 172.2412$\pm$0.4037 & 35.13$\pm$1.67 & 31.04$\pm$1.46 & 34.75$\pm$0.50 & 30.97$\pm$0.43 & 46.44$\pm$0.34 \\ \hline
struc2vec & 292.1057$\pm$7.5681 & 39.93$\pm$1.60 & 33.44$\pm$1.46 & 40.03$\pm$0.35 & 33.78$\pm$0.42 & 0.80$\pm$0.70 \\
PhUSION~(I) & 476.5349$\pm$3.7501 & 40.40$\pm$1.53 & 32.93$\pm$1.32 & 40.82$\pm$0.69 & 33.49$\pm$1.15 & 6.62$\pm$0.80 \\ 
PaCEr~(I) & 150.3713$\pm$0.4370 & 41.50$\pm$1.88 & 32.52$\pm$2.07 & 41.77$\pm$0.63 & 32.94$\pm$1.64 & -4.82$\pm$0.01 \\ \hline
RandNE & 3.7662$\pm$0.1607 & 30.30$\pm$1.40 & 24.82$\pm$0.81 & 30.19$\pm$1.11 & 24.66$\pm$0.95 & 7.05$\pm$2.71 \\
LouvainNE & 0.3428$\pm$0.0045 & 28.23$\pm$1.95 & 21.92$\pm$2.28 & 27.82$\pm$0.88 & 21.58$\pm$0.79 & 46.28$\pm$1.28 \\
SketchNE & 1.3812$\pm$0.0536 & 31.20$\pm$1.76 & 25.43$\pm$1.27 & 31.51$\pm$0.61 & 25.83$\pm$0.49 & 46.16$\pm$0.41 \\
node2binary & 128.7087$\pm$25.8684 & 27.25$\pm$1.74 & 23.04$\pm$1.38 & 26.51$\pm$0.58 & 22.33$\pm$0.62 & 28.06$\pm$0.92 \\ \hline
S3GC & 32.2573$\pm$0.5685 & 31.37$\pm$1.62 & 25.62$\pm$1.16 & 31.61$\pm$0.44 & 25.92$\pm$0.39 & 44.25$\pm$0.01 \\
MAGI & 21.7904$\pm$0.5835 & 37.02$\pm$1.59 & 29.36$\pm$1.85 & 37.72$\pm$0.68 & 30.05$\pm$1.29 & 33.00$\pm$0.00 \\
GraLSP & 249.4036$\pm$43.8074 & 41.36$\pm$2.54 & 33.07$\pm$2.54 & 41.55$\pm$0.38 & 33.35$\pm$0.84 & 4.58$\pm$0.06 \\
DRSR & 11.6590$\pm$0.1086 & 43.29$\pm$1.55 & 35.31$\pm$1.07 & 43.72$\pm$0.46 & 35.78$\pm$0.45 & 10.90$\pm$1.70 \\
DGI & 2.3272$\pm$0.3494 & 42.37$\pm$2.18 & 33.77$\pm$1.64 & 42.14$\pm$0.59 & 33.77$\pm$0.76 & 5.69$\pm$1.11 \\
GraphMAE2 & 4.7755$\pm$2.6565 & 39.22$\pm$1.57 & 31.50$\pm$0.97 & 39.13$\pm$0.59 & 31.47$\pm$1.01 & 0.63$\pm$0.47 \\
GGD & 1.6110$\pm$0.0199 & 40.58$\pm$1.87 & 34.50$\pm$1.52 & 40.33$\pm$0.42 & 34.07$\pm$0.32 & 5.42$\pm$2.13 \\
E2Neg & 7.2887$\pm$0.1599 & 41.29$\pm$2.09 & 33.87$\pm$1.30 & 41.00$\pm$0.90 & 33.82$\pm$0.61 & 1.55$\pm$1.09 \\ \hline
\textbf{PI-HIST}~(R) & 0.0078$\pm$0.0001 & 29.67$\pm$1.36 & 23.95$\pm$1.15 & 29.38$\pm$0.62 & 23.69$\pm$0.61 & 45.16$\pm$0.01 \\
\textbf{PI-HIST}~(A) & 0.6730$\pm$0.0099 & 42.75$\pm$0.99 & 34.51$\pm$0.97 & 43.39$\pm$0.57 & 35.18$\pm$0.76 & 2.28$\pm$0.00 \\ \hline
\end{tabular}
\end{table}

\begin{table}[]\tiny
\centering
\caption{Detailed results of node-level tasks on PPI with node position ground-truth.}
\label{Tab:NC-PPI}
\begin{tabular}{l|l|ll|ll|l}
\hline
\multirow{2}{*}{} & \textbf{Time} & \multicolumn{2}{c|}{\textbf{Validation Set}(\%)} & \multicolumn{2}{c|}{\textbf{Test Set}(\%)} & \textbf{Conductance} \\ \cline{3-6}
 & (sec)$\downarrow$ & \textbf{Micro~F1}$\uparrow$ & \textbf{Macro~F1}$\uparrow$ & \textbf{Micro~F1}$\uparrow$ & \textbf{Macro~F1}$\uparrow$ & (\%)$\downarrow$ \\ \hline
node2vec & 178.8106$\pm$6.8075 & 18.73$\pm$1.42 & 15.48$\pm$0.80 & 18.19$\pm$0.38 & 15.73$\pm$0.46 & 77.58$\pm$2.53 \\
PhUSION~(P) & 22.1798$\pm$0.3195 & 20.55$\pm$0.97 & 16.39$\pm$1.02 & 20.15$\pm$0.41 & 16.79$\pm$0.53 & 77.47$\pm$2.81 \\
PaCEr~(P) & 64.1417$\pm$0.1399 & 15.10$\pm$0.83 & 12.87$\pm$1.06 & 15.49$\pm$0.40 & 13.67$\pm$0.42 & 73.30$\pm$1.42 \\ \hline
struc2vec & 157.3473$\pm$6.2590 & 7.76$\pm$0.67 & 6.01$\pm$0.46 & 8.10$\pm$0.41 & 6.40$\pm$0.33 & 71.94$\pm$4.85 \\
PhUSION~(I) & 143.6224$\pm$1.1783 & 9.46$\pm$0.50 & 4.38$\pm$0.39 & 9.22$\pm$0.28 & 4.21$\pm$0.28 & 56.14$\pm$3.56 \\ 
PaCEr~(I) & 56.3766$\pm$0.1443 & 9.62$\pm$0.90 & 5.54$\pm$0.94 & 9.77$\pm$0.49 & 5.85$\pm$0.51 & 63.98$\pm$1.13 \\ \hline
RandNE & 2.0744$\pm$0.0166 & 14.64$\pm$0.79 & 9.51$\pm$0.58 & 14.92$\pm$0.58 & 10.28$\pm$0.64 & 73.11$\pm$1.94 \\
LouvainNE & 0.1800$\pm$0.0025 & 15.63$\pm$1.24 & 13.37$\pm$0.67 & 15.12$\pm$0.32 & 13.23$\pm$0.21 & 79.58$\pm$1.51 \\
SketchNE & 1.0414$\pm$0.0517 & 21.02$\pm$1.50 & 14.40$\pm$1.26 & 20.94$\pm$0.62 & 14.61$\pm$0.88 & 76.74$\pm$1.76 \\
node2binary & 53.2905$\pm$7.6874 & 17.08$\pm$1.38 & 13.06$\pm$1.08 & 17.47$\pm$0.43 & 13.99$\pm$0.25 & 83.75$\pm$1.12 \\ \hline
S3GC & 107.0300$\pm$0.7505 & 21.16$\pm$1.19 & 14.19$\pm$1.26 & 20.61$\pm$0.66 & 14.21$\pm$0.85 & 73.82$\pm$0.69 \\
MAGI & 7.3130$\pm$0.3360 & 14.02$\pm$0.92 & 6.61$\pm$0.58 & 13.72$\pm$0.43 & 6.62$\pm$0.46 & 83.95$\pm$0.25 \\
GraLSP & 128.5290$\pm$20.0905 & 8.28$\pm$1.00 & 5.80$\pm$0.98 & 8.24$\pm$0.47 & 5.74$\pm$0.43 & 67.28$\pm$1.34 \\
DRSR & 15.1257$\pm$0.1645 & 12.10$\pm$0.65 & 8.35$\pm$0.92 & 12.46$\pm$0.42 & 8.61$\pm$0.41 & 55.72$\pm$0.36 \\
DGI & 2.1321$\pm$0.4393 & 11.84$\pm$1.07 & 6.34$\pm$0.89 & 11.62$\pm$0.57 & 6.54$\pm$0.53 & 82.13$\pm$0.33 \\
GraphMAE2 & 5.3486$\pm$2.5031 & 9.63$\pm$0.76 & 3.40$\pm$0.22 & 9.64$\pm$0.58 & 3.57$\pm$0.30 & 76.38$\pm$0.39 \\
GGD & 1.6078$\pm$0.0394 & 10.77$\pm$0.99 & 4.62$\pm$0.64 & 10.45$\pm$0.47 & 4.54$\pm$0.42 & 76.09$\pm$0.27 \\
E2Neg & 6.9376$\pm$0.1139 & 10.57$\pm$1.15 & 8.52$\pm$0.90 & 10.32$\pm$0.38 & 8.68$\pm$0.30 & 76.07$\pm$0.71 \\ \hline
\textbf{PI-HIST}~(R)  & 0.0061$\pm$0.0011 & 21.05$\pm$1.23 & 16.01$\pm$0.92 & 21.24$\pm$0.22 & 16.78$\pm$0.56 & 79.75$\pm$0.70 \\
\textbf{PI-HIST}~(A) & 0.9303$\pm$0.0096 & 9.78$\pm$0.97 & 3.95$\pm$0.68 & 9.77$\pm$0.34 & 3.96$\pm$0.32 & 63.01$\pm$1.26 \\ \hline
\end{tabular}
\end{table}

\begin{table}[]\tiny
\centering
\caption{Detailed results of node-level tasks on BlogCatalog with node position ground-truth.}
\label{Tab:NC-BlogCatalog}
\begin{tabular}{l|l|ll|ll|l}
\hline
\multirow{2}{*}{} & \textbf{Time} & \multicolumn{2}{c|}{\textbf{Validation Set}(\%)} & \multicolumn{2}{c|}{\textbf{Test Set}(\%)} & \textbf{Conductance} \\ \cline{3-6}
 & (sec)$\downarrow$ & \textbf{Micro~F1}$\uparrow$ & \textbf{Macro~F1}$\uparrow$ & \textbf{Micro~F1}$\uparrow$ & \textbf{Macro~F1}$\uparrow$ & (\%)$\downarrow$ \\ \hline
node2vec & 1613.4396$\pm$24.3453 & 34.42$\pm$1.25 & 20.37$\pm$1.22 & 34.18$\pm$0.49 & 21.32$\pm$0.69 & 77.84$\pm$1.87 \\
PhUSION~(P) & 257.2738$\pm$3.8922 & 39.41$\pm$1.19 & 25.04$\pm$1.47 & 39.05$\pm$0.29 & 25.01$\pm$0.38 & 78.58$\pm$1.56 \\
PaCEr~(P) & 901.4042$\pm$3.6675 & 32.48$\pm$1.28 & 21.60$\pm$1.30 & 31.57$\pm$0.36 & 21.16$\pm$0.34 & 68.32$\pm$1.23 \\ \hline
struc2vec & 1094.7581$\pm$16.1948 & 12.50$\pm$0.94 & 4.99$\pm$0.54 & 12.04$\pm$0.28 & 4.85$\pm$0.14 & 74.71$\pm$3.36 \\
PhUSION~(I) & 1292.3129$\pm$5.3792 & 10.63$\pm$0.85 & 6.14$\pm$0.97 & 10.35$\pm$0.28 & 5.97$\pm$0.27 & 50.32$\pm$2.21 \\ 
PaCEr~(I) & 849.3407$\pm$3.8926 & 16.72$\pm$0.80 & 3.55$\pm$0.78 & 16.29$\pm$0.44 & 3.47$\pm$0.16 & 60.38$\pm$1.07 \\ \hline
RandNE & 123.8933$\pm$11.0708 & 29.95$\pm$0.83 & 12.79$\pm$0.90 & 29.67$\pm$0.33 & 13.09$\pm$0.34 & 50.85$\pm$1.43 \\
LouvainNE & 0.9442$\pm$0.0128 & 23.65$\pm$1.04 & 12.82$\pm$0.63 & 23.06$\pm$0.17 & 13.47$\pm$0.27 & 82.20$\pm$0.48 \\
SketchNE & 1.9865$\pm$0.0491 & 40.77$\pm$1.17 & 21.22$\pm$1.18 & 40.28$\pm$0.35 & 21.91$\pm$0.56 & 78.91$\pm$1.11 \\
node2binary & 266.1865$\pm$47.5454 & 18.97$\pm$0.64 & 7.46$\pm$0.33 & 18.23$\pm$0.40 & 7.43$\pm$0.28 & 79.01$\pm$0.97 \\ \hline
S3GC & 120.4018$\pm$1.7052 & 35.31$\pm$1.01 & 15.85$\pm$0.60 & 34.45$\pm$0.42 & 15.96$\pm$0.57 & 72.65$\pm$0.82 \\
MAGI & 46.9488$\pm$8.1153 & 18.96$\pm$0.68 & 4.38$\pm$0.28 & 18.37$\pm$0.36 & 4.32$\pm$0.13 & 70.05$\pm$0.90 \\
GraLSP & 1008.7149$\pm$74.9156 & 15.68$\pm$0.65 & 4.04$\pm$0.45 & 15.55$\pm$0.40 & 4.23$\pm$0.24 & 70.24$\pm$0.43 \\
DRSR & 249.6355$\pm$1.3925 & 25.02$\pm$1.12 & 8.93$\pm$0.53 & 24.54$\pm$0.32 & 9.27$\pm$0.09 & 54.37$\pm$1.97 \\
DGI & 16.5468$\pm$3.1438 & 20.69$\pm$0.68 & 5.35$\pm$0.27 & 20.10$\pm$0.25 & 5.31$\pm$0.22 & 69.85$\pm$0.85 \\
GraphMAE2 & 4.6504$\pm$2.1057 & 18.19$\pm$0.59 & 3.84$\pm$0.36 & 17.97$\pm$0.39 & 4.01$\pm$0.15 & 79.13$\pm$0.50 \\
GGD & 9.0169$\pm$0.0283 & 18.44$\pm$0.62 & 3.94$\pm$0.34 & 17.98$\pm$0.23 & 3.96$\pm$0.16 & 68.37$\pm$1.09 \\
E2Neg & 9.3498$\pm$0.1002 & 20.35$\pm$0.79 & 10.70$\pm$0.71 & 19.57$\pm$0.36 & 10.23$\pm$0.27 & 61.22$\pm$1.10 \\ \hline
\textbf{PI-HIST}~(R) & 0.0225$\pm$0.0034 & 39.30$\pm$0.92 & 24.52$\pm$0.82 & 39.12$\pm$0.28 & 24.93$\pm$0.36 & 68.51$\pm$2.80 \\
\textbf{PI-HIST}~(A) & 2.7563$\pm$0.0848 & 17.38$\pm$0.60 & 5.24$\pm$0.46 & 16.82$\pm$0.30 & 5.26$\pm$0.17 & 56.32$\pm$1.43 \\ \hline
\end{tabular}
\end{table}

\begin{table}[]\tiny
\centering
\caption{Detailed results of node-level tasks on DBLP with node position ground-truth.}
\label{Tab:NC-DBLP}
\begin{tabular}{l|l|ll|ll|l}
\hline
\multirow{2}{*}{} & \textbf{Time} & \multicolumn{2}{c|}{\textbf{Validation Set}(\%)} & \multicolumn{2}{c|}{\textbf{Test Set}(\%)} & \textbf{Conductance} \\ \cline{3-6}
 & (sec)$\downarrow$ & \textbf{Micro~F1}$\uparrow$ & \textbf{Macro~F1}$\uparrow$ & \textbf{Micro~F1}$\uparrow$ & \textbf{Macro~F1}$\uparrow$ & (\%)$\downarrow$ \\ \hline
node2vec & \multicolumn{6}{c}{OOT} \\
PhUSION~(P) & \multicolumn{6}{c}{OOM} \\
PaCEr~(P) & \multicolumn{6}{c}{OOM} \\ \hline
struc2vec & \multicolumn{6}{c}{OOT} \\
PhUSION~(I) & \multicolumn{6}{c}{OOM} \\ 
PaCEr~(I) & \multicolumn{6}{c}{OOM} \\ \hline
RandNE & \multicolumn{1}{l|}{188.6241$\pm$15.6970} & 45.37$\pm$0.70 & \multicolumn{1}{l|}{29.53$\pm$2.26} & 45.21$\pm$0.17 & \multicolumn{1}{l|}{29.96$\pm$1.69} & 44.84$\pm$1.25 \\
LouvainNE & \multicolumn{1}{l|}{13.9016$\pm$1.0948} & 55.82$\pm$0.62 & \multicolumn{1}{l|}{65.97$\pm$1.55} & 55.78$\pm$0.26 & \multicolumn{1}{l|}{67.12$\pm$0.62} & 51.83$\pm$0.56 \\
SketchNE & \multicolumn{1}{l|}{12.4036$\pm$0.4253} & 63.49$\pm$0.50 & \multicolumn{1}{l|}{68.53$\pm$1.20} & 63.74$\pm$0.22 & \multicolumn{1}{l|}{70.57$\pm$1.05} & 52.47$\pm$3.22 \\
node2binary & \multicolumn{6}{c}{OOM} \\ \hline
S3GC & \multicolumn{1}{l|}{4540.4367$\pm$85.9206} & 57.30$\pm$0.51 & \multicolumn{1}{l|}{39.14$\pm$2.26} & 57.29$\pm$0.30 & \multicolumn{1}{l|}{40.38$\pm$2.63} & 53.11$\pm$0.15 \\
MAGI & \multicolumn{1}{l|}{1170.9504$\pm$43.3845} & 26.32$\pm$0.45 & \multicolumn{1}{l|}{6.17$\pm$0.35} & 26.31$\pm$0.21 & \multicolumn{1}{l|}{6.22$\pm$0.20} & 60.07$\pm$0.46 \\
GraLSP & \multicolumn{6}{c}{OOM} \\
DRSR & \multicolumn{6}{c}{OOM} \\
DGI & \multicolumn{1}{l|}{135.8253$\pm$0.1137} & 18.01$\pm$0.50 & \multicolumn{1}{l|}{1.87$\pm$0.11} & 17.91$\pm$0.13 & \multicolumn{1}{l|}{1.86$\pm$0.04} & 67.76$\pm$0.46 \\
GraphMAE2 & \multicolumn{1}{l|}{65.1096$\pm$9.2301} & 18.68$\pm$0.54 & \multicolumn{1}{l|}{2.33$\pm$0.37} & 18.68$\pm$0.12 & \multicolumn{1}{l|}{2.37$\pm$0.24} & 64.03$\pm$0.17 \\
GGD & \multicolumn{1}{l|}{30.1133$\pm$0.0453} & 25.70$\pm$0.39 & \multicolumn{1}{l|}{28.38$\pm$1.56} & 25.73$\pm$0.21 & \multicolumn{1}{l|}{28.73$\pm$0.67} & 58.29$\pm$1.27 \\
E2Neg & \multicolumn{1}{l|}{44.5398$\pm$5.2809} & 20.02$\pm$0.53 & \multicolumn{1}{l|}{14.20$\pm$1.65} & 20.31$\pm$0.17 & \multicolumn{1}{l|}{15.43$\pm$0.77} & 69.96$\pm$1.03 \\ \hline
\textbf{PI-HIST}~(R) & \multicolumn{1}{l|}{0.1870$\pm$0.0064} & 61.74$\pm$0.50 & \multicolumn{1}{l|}{73.08$\pm$1.37} & 61.72$\pm$0.22 & \multicolumn{1}{l|}{74.60$\pm$0.51} & 51.94$\pm$3.51 \\
\textbf{PI-HIST}~(A) & \multicolumn{1}{l|}{20.5245$\pm$0.9441} & 20.76$\pm$0.40 & \multicolumn{1}{l|}{4.26$\pm$0.58} & 20.70$\pm$0.16 & \multicolumn{1}{l|}{4.61$\pm$0.40} & 44.95$\pm$2.71 \\ \hline
\end{tabular}
\end{table}

\begin{table}[]\tiny
\centering
\caption{Detailed results of node-level tasks on Amazon with node position ground-truth.}
\label{Tab:NC-Amazon}
\begin{tabular}{l|l|ll|ll|l}
\hline
\multirow{2}{*}{} & \textbf{Time} & \multicolumn{2}{c|}{\textbf{Validation Set}(\%)} & \multicolumn{2}{c|}{\textbf{Test Set}(\%)} & \textbf{Conductance} \\ \cline{3-6}
 & (sec)$\downarrow$ & \textbf{Micro~F1}$\uparrow$ & \textbf{Macro~F1}$\uparrow$ & \textbf{Micro~F1}$\uparrow$ & \textbf{Macro~F1}$\uparrow$ & (\%)$\downarrow$ \\ \hline
node2vec & \multicolumn{6}{c}{OOT} \\
PhUSION~(P) & \multicolumn{6}{c}{OOM} \\
PaCEr~(P) & \multicolumn{6}{c}{OOM} \\ \hline
struc2vec & \multicolumn{6}{c}{OOT} \\
PhUSION~(I) & \multicolumn{6}{c}{OOM} \\ 
PaCEr~(I) & \multicolumn{6}{c}{OOM} \\ \hline
RandNE & \multicolumn{1}{l|}{78.4991$\pm$3.6214} & 91.55$\pm$1.77 & \multicolumn{1}{l|}{93.10$\pm$1.80} & 91.30$\pm$1.70 & \multicolumn{1}{l|}{92.98$\pm$1.54} & 74.13$\pm$0.06 \\
LouvainNE & \multicolumn{1}{l|}{9.1255$\pm$1.1484} & 98.60$\pm$0.72 & \multicolumn{1}{l|}{98.45$\pm$0.76} & 98.63$\pm$0.36 & \multicolumn{1}{l|}{98.50$\pm$0.45} & 67.31$\pm$0.98 \\
SketchNE & \multicolumn{1}{l|}{9.8709$\pm$0.1689} & 99.08$\pm$0.50 & \multicolumn{1}{l|}{99.02$\pm$0.55} & 99.08$\pm$0.39 & \multicolumn{1}{l|}{99.02$\pm$0.46} & 68.16$\pm$1.51 \\
node2binary & \multicolumn{6}{c}{OOM} \\ \hline
S3GC & \multicolumn{1}{l|}{1442.9805$\pm$8.8064} & 94.99$\pm$2.00 & \multicolumn{1}{l|}{93.81$\pm$2.30} & 94.98$\pm$1.69 & \multicolumn{1}{l|}{94.30$\pm$2.18} & 68.00$\pm$0.51 \\
MAGI & \multicolumn{1}{l|}{1198.7392$\pm$13.9581} & 68.20$\pm$2.91 & \multicolumn{1}{l|}{71.51$\pm$3.20} & 67.96$\pm$1.94 & \multicolumn{1}{l|}{71.42$\pm$2.26} & 82.16$\pm$0.04 \\
GraLSP & \multicolumn{6}{c}{OOM} \\
DRSR & \multicolumn{6}{c}{OOM} \\
DGI & \multicolumn{1}{l|}{48.3331$\pm$0.6428} & 27.08$\pm$1.96 & \multicolumn{1}{l|}{21.30$\pm$2.30} & 27.07$\pm$0.96 & \multicolumn{1}{l|}{21.83$\pm$1.25} & 70.49$\pm$0.73 \\
GraphMAE2 & \multicolumn{1}{l|}{38.4175$\pm$1.7095} & 20.03$\pm$2.23 & \multicolumn{1}{l|}{10.42$\pm$1.97} & 20.51$\pm$1.27 & \multicolumn{1}{l|}{10.95$\pm$1.76} & 76.59$\pm$0.63 \\
GGD & \multicolumn{1}{l|}{16.7129$\pm$0.1496} & 58.68$\pm$2.08 & \multicolumn{1}{l|}{60.27$\pm$2.68} & 58.58$\pm$1.54 & \multicolumn{1}{l|}{60.43$\pm$1.57} & 67.33$\pm$0.97 \\
E2Neg & \multicolumn{1}{l|}{42.2155$\pm$2.3756} & 54.77$\pm$3.47 & \multicolumn{1}{l|}{56.57$\pm$4.05} & 54.38$\pm$1.08 & \multicolumn{1}{l|}{57.03$\pm$1.10} & 85.31$\pm$0.02 \\ \hline
\textbf{PI-HIST}~(R) & \multicolumn{1}{l|}{0.1099$\pm$0.0046} & 98.83$\pm$0.61 & \multicolumn{1}{l|}{98.85$\pm$0.50} & 98.84$\pm$0.48 & \multicolumn{1}{l|}{98.85$\pm$0.48} & 68.91$\pm$5.35 \\
\textbf{PI-HIST}~(A) & \multicolumn{1}{l|}{17.3222$\pm$1.5904} & 42.22$\pm$2.25 & \multicolumn{1}{l|}{39.27$\pm$2.73} & 43.49$\pm$1.50 & \multicolumn{1}{l|}{41.65$\pm$1.98} & 56.70$\pm$1.75 \\ \hline
\end{tabular}
\end{table}

\begin{table}[]\tiny
\centering
\caption{Detailed inference time (sec)$\downarrow$ of different steps in PI-HIST~(A) evaluated by node-level tasks, where RW, AW-MAP, AW-IDX, AW-CNT, CAT, HIER, and FFP denote the time (sec) of (\romannumeral1) RW sampling, (\romannumeral2) AW mapping, (\romannumeral3) AW indexing, (\romannumeral4) AW counting, (\romannumeral5) statistic gathering, (\romannumeral6) hierarchical structure extraction, and (\romannumeral7) one FFP.}
\label{Tab:Time-A}
\begin{tabular}{l|p{1.25cm}p{1.25cm}p{1.25cm}p{1.25cm}p{1.25cm}p{1.25cm}l|l}
\hline
 & RW & AW-MAP & AW-IDX & AW-CNT & GAT & HIER & FFP & \textbf{Total} \\ \hline
\textbf{Europe} & 0.0030$\pm$0.0011 & 0.0004$\pm$0.0000 & 0.0004$\pm$0.0000 & 0.0138$\pm$0.0018 & 0.0009$\pm$0.0001 & 0.0079$\pm$0.0015 & 0.0627$\pm$0.0003 & 0.0891$\pm$0.0021 \\
\textbf{USA} & 0.0082$\pm$0.0050 & 0.0004$\pm$0.0000 & 0.0004$\pm$0.0000 & 0.0435$\pm$0.0014 & 0.0069$\pm$0.0032 & 0.0091$\pm$0.0021 & 0.0627$\pm$0.0002 & 0.1313$\pm$0.0064 \\
\textbf{Actor} & 0.0389$\pm$0.0087 & 0.0005$\pm$0.0001 & 0.0006$\pm$0.0004 & 0.3493$\pm$0.0065 & 0.0042$\pm$0.0003 & 0.0352$\pm$0.0099 & 0.2442$\pm$0.0011 & 0.6730$\pm$0.0099 \\
\textbf{Film} & 2.2775$\pm$0.2241 & 0.0028$\pm$0.0004 & 0.0604$\pm$0.0724 & 1.3298$\pm$0.0637 & 0.3584$\pm$0.0497 & 0.2700$\pm$0.0142 & 0.5985$\pm$0.0012 & 4.8974$\pm$0.1999 \\ \hline
\textbf{PPI} & 0.0243$\pm$0.0055 & 0.0004$\pm$0.0000 & 0.0026$\pm$0.0031 & 0.2049$\pm$0.0043 & 0.0241$\pm$0.0034 & 0.0772$\pm$0.0104 & 0.5968$\pm$0.0012 & 0.9303$\pm$0.0096 \\
\textbf{BlogC}. & 0.7299$\pm$0.1022 & 0.0012$\pm$0.0002 & 0.0239$\pm$0.0592 & 0.5393$\pm$0.0352 & 0.0388$\pm$0.0040 & 0.1081$\pm$0.0194 & 1.3153$\pm$0.0016 & 2.7563$\pm$0.0848 \\
\textbf{DBLP} & 6.5442$\pm$0.4550 & 0.0101$\pm$0.0004 & 0.0844$\pm$0.0189 & 2.6930$\pm$0.0308 & 5.6679$\pm$0.8540 & 5.5233$\pm$0.0418 & 0.0015$\pm$0.0001 & 20.5245$\pm$0.9441 \\
\textbf{Amazon} & 4.8365$\pm$0.5854 & 0.0110$\pm$0.0019 & 0.0859$\pm$0.0175 & 2.5868$\pm$0.0456 & 5.3334$\pm$1.2971 & 3.5042$\pm$0.0295 & 0.9645$\pm$0.0042 & 17.3222$\pm$1.5904 \\ \hline
\end{tabular}
\end{table}

\begin{table}[]\tiny
\centering
\caption{Detailed results of link prediction on all datasets, where Val and Test denote results on the validation and test sets, respectively.}
\label{Tab:LP}
\begin{tabular}{l|ll|ll|ll|ll}
\hline
\multirow{2}{*}{} & \multicolumn{2}{c|}{\textbf{Europe}} & \multicolumn{2}{c|}{\textbf{USA}} & \multicolumn{2}{c|}{\textbf{PPI}} & \multicolumn{2}{c}{\textbf{Actor}} \\ \cline{2-9} 
 & \textbf{Val}(\%) & \textbf{Test}(\%) & \textbf{Val}(\%) & \textbf{Test}(\%) & \textbf{Val}(\%) & \textbf{Test}(\%) &\textbf{Val}(\%) & \textbf{Test}(\%) \\
 & \textbf{AUC}$\uparrow$ & \textbf{AUC}$\uparrow$ & \textbf{AUC}$\uparrow$ & \textbf{AUC}$\uparrow$ & \textbf{AUC}$\uparrow$ & \textbf{AUC}$\uparrow$ & \textbf{AUC}$\uparrow$ & \textbf{AUC}$\uparrow$ \\ \hline
node2vec & 88.18$\pm$0.84 & 89.03$\pm$0.88 & 91.41$\pm$0.80 & 91.40$\pm$0.34 & 87.99$\pm$0.41 & 88.31$\pm$0.20 & 81.01$\pm$0.40 & 80.18$\pm$0.86 \\
PhUSION~(P) & 86.89$\pm$0.80 & 87.99$\pm$1.25 & 89.45$\pm$0.36 & 89.38$\pm$0.71 & 87.54$\pm$0.55 & 87.83$\pm$0.30 & 83.60$\pm$0.56 & 82.97$\pm$0.55 \\
PaCEr~(P) & 87.30$\pm$1.19 & 88.39$\pm$1.22 & 89.94$\pm$1.24 & 89.72$\pm$0.42 & 87.03$\pm$0.43 & 87.01$\pm$0.20 & 83.84$\pm$0.47 & 83.08$\pm$0.87 \\ \hline
struc2vec & 89.65$\pm$0.69 & 90.80$\pm$0.84 & 95.35$\pm$0.49 & 95.01$\pm$0.36 & 91.01$\pm$0.36 & 91.02$\pm$0.25 & 90.71$\pm$0.59 & 90.17$\pm$0.49 \\
PhUSION~(I) & 89.87$\pm$0.66 & 90.98$\pm$0.80 & 95.58$\pm$0.59 & 95.35$\pm$0.21 & 91.08$\pm$0.37 & 91.09$\pm$0.10 & 89.60$\pm$0.53 & 89.13$\pm$0.50 \\
PaCEr~(I) & 88.77$\pm$0.89 & 89.97$\pm$0.71 & 93.11$\pm$0.41 & 92.85$\pm$0.21 & 89.85$\pm$0.33 & 89.76$\pm$0.25 & 88.91$\pm$0.42 & 88.30$\pm$0.49 \\ \hline
RandNE & 89.02$\pm$0.89 & 90.15$\pm$0.78 & 94.49$\pm$0.59 & 94.24$\pm$0.32 & 88.35$\pm$0.29 & 88.26$\pm$0.15 & 83.57$\pm$0.54 & 83.69$\pm$0.67 \\
LouvainNE & 76.16$\pm$1.33 & 77.10$\pm$1.25 & 83.26$\pm$0.97 & 82.99$\pm$0.24 & 74.26$\pm$0.90 & 74.65$\pm$0.46 & 72.88$\pm$0.84 & 72.54$\pm$1.00 \\
SketchNE & 81.48$\pm$1.14 & 83.02$\pm$1.63 & 88.39$\pm$0.49 & 87.82$\pm$0.36 & 86.06$\pm$0.61 & 86.07$\pm$0.63 & 82.55$\pm$0.82 & 82.35$\pm$0.50 \\
node2binary & 82.57$\pm$1.26 & 83.92$\pm$0.99 & 90.50$\pm$0.77 & 89.79$\pm$0.85 & 80.73$\pm$2.65 & 80.89$\pm$2.28 & 76.37$\pm$0.60 & 75.63$\pm$1.27 \\ \hline
S3GC & 87.67$\pm$0.85 & 89.08$\pm$1.18 & 89.66$\pm$0.41 & 89.01$\pm$0.32 & 89.79$\pm$0.41 & 89.93$\pm$0.14 & 77.07$\pm$0.45 & 77.12$\pm$1.10 \\
MAGI & 89.14$\pm$0.92 & 90.12$\pm$0.95 & 94.17$\pm$0.74 & 93.89$\pm$0.73 & 90.68$\pm$0.50 & 90.75$\pm$0.18 & 89.97$\pm$0.55 & 89.38$\pm$0.37 \\
GraLSP & 89.91$\pm$0.79 & 90.72$\pm$0.71 & 95.26$\pm$0.42 & 94.95$\pm$0.13 & 91.08$\pm$0.39 & 91.14$\pm$0.15 & 90.58$\pm$0.62 & 89.81$\pm$0.34 \\
DRSR & 88.91$\pm$1.08 & 90.17$\pm$0.75 & 95.91$\pm$0.34 & 95.74$\pm$0.30 & 91.09$\pm$0.41 & 91.05$\pm$0.28 & 90.68$\pm$0.42 & 90.15$\pm$0.70 \\
DGI & 88.91$\pm$0.83 & 89.90$\pm$0.91 & 95.27$\pm$0.45 & 95.13$\pm$0.35 & 91.28$\pm$0.39 & 91.24$\pm$0.14 & 88.46$\pm$0.54 & 87.81$\pm$0.46 \\
GraphMAE2 & 85.99$\pm$2.23 & 87.46$\pm$1.73 & 88.85$\pm$2.22 & 88.04$\pm$2.47 & 87.70$\pm$1.55 & 87.98$\pm$1.06 & 86.27$\pm$1.72 & 85.75$\pm$1.64 \\
GGD & 89.48$\pm$0.71 & 90.84$\pm$0.80 & 96.05$\pm$0.38 & 95.89$\pm$0.25 & 91.34$\pm$0.41 & 91.22$\pm$0.11 & 90.29$\pm$0.48 & 89.68$\pm$0.44 \\
E2Neg & 89.05$\pm$0.78 & 90.14$\pm$0.71 & 95.34$\pm$0.50 & 95.10$\pm$0.15 & 91.05$\pm$0.43 & 90.99$\pm$0.25 & 90.36$\pm$0.37 & 89.72$\pm$0.47 \\ \hline
\textbf{PI-HIST}~(R) & 87.83$\pm$1.05 & 88.92$\pm$1.01 & 92.01$\pm$0.27 & 91.75$\pm$0.32 & 87.67$\pm$0.51 & 87.77$\pm$0.25 & 84.51$\pm$0.51 & 84.69$\pm$0.63 \\
\textbf{PI-HIST}~(A) & 89.57$\pm$1.06 & 90.89$\pm$0.70 & 95.64$\pm$0.30 & 95.46$\pm$0.30 & 91.15$\pm$0.32 & 91.14$\pm$0.22 & 90.34$\pm$0.62 & 89.80$\pm$0.50 \\
\textbf{PI-HIST}~(R\&A) & 89.91$\pm$0.88 & 91.07$\pm$0.80 & 96.07$\pm$0.35 & 95.85$\pm$0.20 & 91.40$\pm$0.46 & 91.38$\pm$0.25 & 90.51$\pm$0.52 & 89.94$\pm$0.47 \\ \hline
\multirow{2}{*}{} & \multicolumn{2}{c|}{\textbf{BlogCatalog}} & \multicolumn{2}{c|}{\textbf{Film}} & \multicolumn{2}{c|}{\textbf{DBLP}} & \multicolumn{2}{c}{\textbf{Amazon}} \\ \cline{2-9} 
 & \textbf{Val}(\%) & \textbf{Test}(\%) & \textbf{Val}(\%) & \textbf{Test}(\%) & \textbf{Val}(\%) & \textbf{Test}(\%) &\textbf{Val}(\%) & \textbf{Test}(\%) \\
 & \textbf{AUC}$\uparrow$ & \textbf{AUC}$\uparrow$ & \textbf{AUC}$\uparrow$ & \textbf{AUC}$\uparrow$ & \textbf{AUC}$\uparrow$ & \textbf{AUC}$\uparrow$ & \textbf{AUC}$\uparrow$ & \textbf{AUC}$\uparrow$ \\ \hline
node2vec & 95.31$\pm$0.09 & 95.29$\pm$0.08 & 79.79$\pm$0.16 & 79.65$\pm$0.55 & \multicolumn{2}{c|}{OOT} & \multicolumn{2}{c}{OOT} \\
PhUSION~(P) & 95.25$\pm$0.08 & 95.24$\pm$0.05 & 82.88$\pm$0.13 & 82.68$\pm$0.54 & \multicolumn{2}{c|}{OOM} & \multicolumn{2}{c}{OOM} \\
PaCEr~(P) & 94.76$\pm$0.10 & 94.73$\pm$0.09 & 80.91$\pm$0.55 & 80.76$\pm$0.53 & \multicolumn{2}{c|}{OOM} & \multicolumn{2}{c}{OOM} \\ \hline
struc2vec & 95.53$\pm$0.07 & 95.48$\pm$0.11 & 88.71$\pm$0.14 & 88.41$\pm$0.25 & \multicolumn{2}{c|}{OOT} & \multicolumn{2}{c}{OOT} \\
PhUSION~(I) & 95.87$\pm$0.08 & 95.85$\pm$0.09 & 85.62$\pm$0.08 & 85.32$\pm$0.39 & \multicolumn{2}{c|}{OOM} & \multicolumn{2}{c}{OOM} \\
PaCEr~(I) & 94.75$\pm$0.06 & 94.73$\pm$0.10 & 87.67$\pm$0.15 & 87.33$\pm$0.38 & \multicolumn{2}{c|}{OOM} & \multicolumn{2}{c}{OOM} \\ \hline
RandNE & 95.13$\pm$0.07 & 95.12$\pm$0.09 & 86.22$\pm$0.19 & 86.13$\pm$0.34 & 69.75$\pm$0.09 & 69.91$\pm$0.08 & 53.68$\pm$0.20 & 53.88$\pm$0.12 \\
LouvainNE & 75.48$\pm$0.51 & 75.17$\pm$0.34 & 67.72$\pm$0.50 & 68.08$\pm$0.76 & 62.57$\pm$0.16 & 62.48$\pm$0.24 & 54.98$\pm$0.22 & 55.11$\pm$0.24 \\
SketchNE & 94.69$\pm$0.16 & 94.74$\pm$0.12 & 78.10$\pm$0.48 & 78.04$\pm$0.66 & 68.61$\pm$0.16 & 68.55$\pm$0.20 & 58.33$\pm$0.09 & 58.32$\pm$0.11 \\
node2binary & 80.74$\pm$1.70 & 80.64$\pm$1.62 & 76.01$\pm$4.29 & 75.85$\pm$4.30 & \multicolumn{2}{c|}{OOM} & \multicolumn{2}{c}{OOM} \\ \hline
S3GC & 93.02$\pm$0.09 & 93.00$\pm$0.08 & 73.12$\pm$0.40 & 72.83$\pm$0.63 & 80.43$\pm$0.22 & 80.46$\pm$0.29 & 62.44$\pm$0.16 & 62.48$\pm$0.24 \\
MAGI & 95.23$\pm$0.13 & 95.22$\pm$0.16 & 87.04$\pm$2.41 & 86.81$\pm$2.27 & 84.42$\pm$0.46 & 84.44$\pm$0.38 & 72.77$\pm$0.74 & 72.89$\pm$0.67 \\
GraLSP & 95.92$\pm$0.07 & 95.89$\pm$0.08 & 88.24$\pm$0.11 & 88.01$\pm$0.29 & \multicolumn{2}{c|}{OOM} & \multicolumn{2}{c}{OOM} \\
DRSR & 95.48$\pm$0.13 & 95.47$\pm$0.17 & 88.27$\pm$0.13 & 87.97$\pm$0.20 & \multicolumn{2}{c|}{OOM} & \multicolumn{2}{c}{OOM} \\
DGI & 96.15$\pm$0.06 & 96.12$\pm$0.07 & 87.02$\pm$0.15 & 86.70$\pm$0.29 & 88.23$\pm$0.10 & 88.23$\pm$0.07 & 78.62$\pm$0.11 & 78.69$\pm$0.04 \\
GraphMAE2 & 93.36$\pm$0.66 & 93.32$\pm$0.63 & 85.01$\pm$0.65 & 84.75$\pm$0.78 & 85.96$\pm$0.12 & 85.89$\pm$0.16 & 77.78$\pm$0.22 & 77.85$\pm$0.17 \\
GGD & 95.99$\pm$0.08 & 95.98$\pm$0.09 & 88.44$\pm$0.14 & 88.20$\pm$0.21 & 88.57$\pm$0.12 & 88.59$\pm$0.08 & 79.19$\pm$0.13 & 79.25$\pm$0.06 \\
E2Neg & 95.47$\pm$0.12 & 95.45$\pm$0.19 & 88.21$\pm$0.19 & 87.97$\pm$0.34 & 88.03$\pm$0.07 & 88.07$\pm$0.03 & 78.93$\pm$0.09 & 79.01$\pm$0.02 \\ \hline
\textbf{PI-HIST}~(R) & 95.16$\pm$0.04 & 95.14$\pm$0.06 & 83.68$\pm$0.23 & 83.42$\pm$0.44 & 82.28$\pm$0.06 & 82.28$\pm$0.10 & 66.57$\pm$0.07 & 66.62$\pm$0.14 \\
\textbf{PI-HIST}~(A) & 95.88$\pm$0.09 & 95.91$\pm$0.07 & 87.29$\pm$0.16 & 86.96$\pm$0.38 & 87.61$\pm$0.11 & 87.59$\pm$0.03 & 78.81$\pm$0.07 & 78.90$\pm$0.06 \\
\textbf{PI-HIST}~(R\&A) & 96.29$\pm$0.08 & 96.30$\pm$0.08 & 88.18$\pm$0.24 & 87.92$\pm$0.31 & 88.24$\pm$0.09 & 88.25$\pm$0.05 & 78.99$\pm$0.11 & 79.07$\pm$0.03 \\ \hline
\end{tabular}
\end{table}

\begin{table}[]\tiny
\centering
\caption{Detailed results of graph reconstruction on all datasets.}
\label{Tab:GR}
\begin{tabular}{l|l|l|l|l|l|l|l|l}
\hline
\multirow{2}{*}{} & \textbf{Europe} & \textbf{USA} & \textbf{PPI} & \textbf{Actor} & \textbf{BlogCatalog} & \textbf{Film} & \textbf{DBLP} & \textbf{Amazon} \\ 
 & \textbf{AUC}(\%)$\uparrow$ & \textbf{AUC}(\%)$\uparrow$ & \textbf{AUC}(\%)$\uparrow$ & \textbf{AUC}(\%)$\uparrow$ & \textbf{AUC}(\%)$\uparrow$ & \textbf{AUC}(\%)$\uparrow$ & \textbf{AUC}(\%)$\uparrow$ & \textbf{AUC}(\%)$\uparrow$ \\ \hline
node2vec & 90.96$\pm$0.42 & 91.02$\pm$0.26 & 87.75$\pm$0.31 & 88.46$\pm$1.38 & 95.41$\pm$0.20 & 92.58$\pm$1.15 & OOT & OOT \\
PhUSION~(P) & 90.64$\pm$0.33 & 87.84$\pm$0.27 & 87.15$\pm$0.34 & 89.98$\pm$0.91 & 95.30$\pm$0.20 & 93.70$\pm$1.08 & OOM & OOM \\
PaCEr~(P) & 90.17$\pm$0.46 & 92.81$\pm$0.22 & 86.67$\pm$0.24 & 91.74$\pm$0.98 & 94.64$\pm$0.22 & 94.07$\pm$1.18 & OOM & OOM \\ \hline
struc2vec & 92.04$\pm$0.38 & 94.25$\pm$0.27 & 89.53$\pm$0.19 & 88.95$\pm$1.61 & 94.90$\pm$0.21 & 88.01$\pm$1.73 & OOT & OOT \\
PhUSION~(I) & 92.02$\pm$0.27 & 95.39$\pm$0.16 & 89.44$\pm$0.25 & 88.22$\pm$1.21 & 95.26$\pm$0.22 & 85.13$\pm$2.45 & OOM & OOM \\
PaCEr~(I) & 91.15$\pm$0.33 & 92.51$\pm$0.24 & 87.46$\pm$0.17 & 85.37$\pm$1.32 & 94.85$\pm$0.16 & 81.18$\pm$1.45 & OOM & OOM \\ \hline
RandNE & 91.44$\pm$0.36 & 92.67$\pm$0.30 & 86.21$\pm$0.26 & 87.80$\pm$0.99 & 94.64$\pm$0.19 & 93.04$\pm$0.78 & 88.46$\pm$1.42 & 84.08$\pm$3.55 \\
LouvainNE & 79.12$\pm$0.91 & 81.74$\pm$0.61 & 73.66$\pm$0.27 & 78.19$\pm$1.84 & 76.45$\pm$0.18 & 71.44$\pm$2.83 & 85.08$\pm$1.97 & 77.79$\pm$3.43 \\
SketchNE & 89.58$\pm$0.42 & 92.72$\pm$0.23 & 85.48$\pm$0.22 & 90.58$\pm$0.68 & 90.92$\pm$0.32 & 94.80$\pm$0.64 & 85.83$\pm$2.12 & 78.52$\pm$5.10 \\
node2binary & 87.80$\pm$0.52 & 91.41$\pm$0.33 & 83.69$\pm$0.27 & 86.92$\pm$1.34 & 83.71$\pm$0.58 & 94.72$\pm$0.88 & OOM & OOM \\ \hline
S3GC & 89.56$\pm$0.41 & 89.06$\pm$0.26 & 88.82$\pm$0.20 & 85.22$\pm$0.84 & 93.59$\pm$0.20 & 79.17$\pm$2.46 & 85.61$\pm$1.77 & 76.69$\pm$4.06 \\
MAGI & 91.64$\pm$0.42 & 95.18$\pm$0.26 & 88.95$\pm$0.15 & 87.17$\pm$0.72 & 95.02$\pm$0.17 & 85.56$\pm$1.66 & 89.20$\pm$1.64 & 84.10$\pm$4.04 \\
GraLSP & 91.51$\pm$0.24 & 94.04$\pm$0.24 & 89.66$\pm$0.20 & 84.31$\pm$1.16 & 95.45$\pm$0.19 & 85.24$\pm$1.86 & OOM & OOM \\
DRSR & 91.65$\pm$0.30 & 95.11$\pm$0.26 & 89.09$\pm$0.21 & 84.39$\pm$1.16 & 95.65$\pm$0.18 & 74.95$\pm$3.16 & OOM & OOM \\
DGI & 91.80$\pm$0.35 & 94.73$\pm$0.21 & 89.74$\pm$0.19 & 88.41$\pm$1.18 & 96.10$\pm$0.17 & 87.56$\pm$2.12 & 88.88$\pm$1.64 & 82.62$\pm$4.35 \\
GraphMAE2 & 90.93$\pm$0.34 & 92.16$\pm$0.33 & 87.59$\pm$0.19 & 85.68$\pm$0.68 & 94.98$\pm$0.20 & 84.07$\pm$2.10 & 88.61$\pm$1.57 & 83.41$\pm$4.67 \\
GGD & 92.25$\pm$0.31 & 95.34$\pm$0.23 & 89.12$\pm$0.14 & 89.81$\pm$0.82 & 95.31$\pm$0.17 & 91.82$\pm$1.39 & 89.19$\pm$1.60 & 84.66$\pm$3.64 \\
E2Neg & 92.04$\pm$0.32 & 94.77$\pm$0.24 & 89.53$\pm$0.22 & 86.00$\pm$1.19 & 95.99$\pm$0.17 & 86.91$\pm$1.28 & 88.24$\pm$1.89 & 81.57$\pm$4.73 \\ \hline
\textbf{PI-HIST}~(R) & 90.13$\pm$0.38 & 91.92$\pm$0.26 & 85.11$\pm$0.44 & 90.83$\pm$0.65 & 94.65$\pm$0.19 & 94.31$\pm$0.76 & 88.12$\pm$1.59 & 83.78$\pm$4.19 \\
\textbf{PI-HIST}~(A) & 92.09$\pm$0.34 & 95.06$\pm$0.20 & 89.42$\pm$0.17 & 87.63$\pm$0.67 & 95.40$\pm$0.14 & 84.91$\pm$2.42 & 83.37$\pm$2.60 & 82.56$\pm$4.70 \\
\textbf{PI-HIST}~(R\&A) & 92.87$\pm$0.36 & 95.61$\pm$0.17 & 90.36$\pm$0.18 & 95.63$\pm$0.54 & 96.21$\pm$0.13 & 95.39$\pm$0.69 & 90.80$\pm$1.27 & 87.68$\pm$3.14 \\ \hline
\end{tabular}
\end{table}

\begin{table}[]\tiny
\centering
\caption{Detailed results of PI-HIST~(R\&A) with different settings of $\alpha$, where LP and GR denote results of link prediction and graph reconstruction, respectively.}
\label{Tab:RA-alpha}
\begin{tabular}{p{0.45cm}|p{0.8cm}|p{0.85cm}p{0.85cm}p{0.85cm}p{0.85cm}p{0.85cm}p{0.85cm}p{0.85cm}p{0.85cm}p{0.85cm}p{0.85cm}p{0.85cm}l}
\hline
 &  & $\alpha$=0.0 & 0.1 & 0.2 & 0.3 & 0.4 & 0.5 & 0.6 & 0.7 & 0.8 & 0.9 & 1.0 \\ \hline
\multirow{8}{*}{\textbf{\begin{tabular}[c]{@{}c@{}c@{}}\textbf{LP}\\ \textbf{AUC}\\(\%)$\uparrow$\end{tabular}}} & \textbf{Europe} & 90.67$\pm$0.77 & 90.81$\pm$0.84 & 90.83$\pm$0.83 & 90.86$\pm$0.83 & 90.91$\pm$0.82 & 90.96$\pm$0.82 & 91.02$\pm$0.81 & 91.06$\pm$0.81 & 91.07$\pm$0.80 & 91.05$\pm$0.78 & 89.79$\pm$0.88 \\
 & \textbf{USA} & 95.58$\pm$0.23 & 95.76$\pm$0.20 & 95.79$\pm$0.20 & 95.80$\pm$0.21 & 95.82$\pm$0.21 & 95.83$\pm$0.21 & 95.85$\pm$0.21 & 95.85$\pm$0.20 & 95.85$\pm$0.20 & 95.85$\pm$0.20 & 92.70$\pm$0.32 \\
 & \textbf{PPI} & 91.17$\pm$0.16 & 91.32$\pm$0.23 & 91.34$\pm$0.23 & 91.36$\pm$0.24 & 91.37$\pm$0.24 & 91.38$\pm$0.24 & 91.38$\pm$0.25 & 91.37$\pm$0.25 & 91.36$\pm$0.26 & 91.35$\pm$0.26 & 88.20$\pm$0.21 \\
 & \textbf{Actor} & 89.72$\pm$0.41 & 89.61$\pm$0.41 & 89.67$\pm$0.41 & 89.74$\pm$0.41 & 89.79$\pm$0.42 & 89.83$\pm$0.44 & 89.87$\pm$0.46 & 89.91$\pm$0.47 & 89.94$\pm$0.47 & 89.93$\pm$0.49 & 85.31$\pm$0.66 \\
 & \textbf{BlogC}. & 95.82$\pm$0.07 & 96.19$\pm$0.07 & 96.21$\pm$0.07 & 96.23$\pm$0.07 & 96.25$\pm$0.07 & 96.28$\pm$0.07 & 96.29$\pm$0.08 & 96.30$\pm$0.08 & 96.29$\pm$0.08 & 96.29$\pm$0.08 & 95.21$\pm$0.07 \\
 & \textbf{Film} & 87.41$\pm$0.32 & 87.53$\pm$0.32 & 87.59$\pm$0.32 & 87.67$\pm$0.32 & 87.76$\pm$0.32 & 87.83$\pm$0.31 & 87.87$\pm$0.31 & 87.90$\pm$0.31 & 87.91$\pm$0.31 & 87.92$\pm$0.31 & 84.29$\pm$0.40 \\
 & \textbf{DBLP} & 87.50$\pm$0.07 & 87.55$\pm$0.07 & 87.60$\pm$0.06 & 87.68$\pm$0.06 & 87.80$\pm$0.07 & 87.95$\pm$0.06 & 88.09$\pm$0.06 & 88.20$\pm$0.06 & 88.25$\pm$0.05 & 88.24$\pm$0.06 & 82.11$\pm$0.11 \\
 & \textbf{Amazon} & 78.34$\pm$0.06 & 78.35$\pm$0.06 & 78.45$\pm$0.07 & 78.60$\pm$0.05 & 78.75$\pm$0.05 & 78.87$\pm$0.05 & 78.95$\pm$0.04 & 79.03$\pm$0.04 & 79.06$\pm$0.04 & 79.07$\pm$0.03 & 65.08$\pm$0.14 \\ \hline
\multirow{8}{*}{\textbf{\begin{tabular}[c]{@{}c@{}c@{}}\textbf{GR}\\ \textbf{AUC}\\(\%)$\uparrow$\end{tabular}}} & \textbf{Europe} & 91.87$\pm$0.38 & 92.27$\pm$0.37 & 92.51$\pm$0.37 & 92.66$\pm$0.36 & 92.78$\pm$0.36 & 92.85$\pm$0.36 & 92.87$\pm$0.36 & 92.86$\pm$0.36 & 92.79$\pm$0.35 & 92.58$\pm$0.34 & 90.93$\pm$0.41 \\
 & \textbf{USA} & 95.23$\pm$0.17 & 95.59$\pm$0.17 & 95.59$\pm$0.17 & 95.60$\pm$0.18 & 95.61$\pm$0.17 & 95.61$\pm$0.16 & 95.58$\pm$0.15 & 95.53$\pm$0.16 & 95.45$\pm$0.18 & 95.26$\pm$0.17 & 92.03$\pm$0.27 \\
 & \textbf{PPI} & 89.53$\pm$0.14 & 90.17$\pm$0.20 & 90.34$\pm$0.17 & 90.36$\pm$0.18 & 90.35$\pm$0.17 & 90.30$\pm$0.17 & 90.18$\pm$0.18 & 90.11$\pm$0.20 & 89.99$\pm$0.15 & 89.79$\pm$0.19 & 86.21$\pm$0.30 \\
 & \textbf{Actor} & 89.49$\pm$0.62 & 94.86$\pm$0.78 & 95.63$\pm$0.54 & 95.39$\pm$0.54 & 94.73$\pm$0.56 & 94.35$\pm$0.54 & 94.25$\pm$0.59 & 93.76$\pm$0.41 & 93.34$\pm$0.47 & 93.28$\pm$0.66 & 90.34$\pm$0.56 \\
 & \textbf{BlogC}. & 95.31$\pm$0.18 & 95.88$\pm$0.16 & 96.08$\pm$0.13 & 96.19$\pm$0.14 & 96.17$\pm$0.13 & 96.16$\pm$0.14 & 96.17$\pm$0.11 & 96.21$\pm$0.13 & 96.15$\pm$0.07 & 96.05$\pm$0.13 & 95.06$\pm$0.14 \\
 & \textbf{Film} & 84.04$\pm$1.75 & 84.69$\pm$1.83 & 89.12$\pm$4.08 & 87.84$\pm$3.32 & 92.45$\pm$1.81 & 91.01$\pm$2.69 & 94.18$\pm$1.02 & 91.85$\pm$1.26 & 95.34$\pm$0.60 & 95.39$\pm$0.69 & 94.11$\pm$1.14 \\
 & \textbf{DBLP} & 86.63$\pm$2.11 & 90.51$\pm$1.25 & 90.73$\pm$1.27 & 90.80$\pm$1.27 & 90.74$\pm$1.27 & 90.67$\pm$1.26 & 90.61$\pm$1.28 & 90.49$\pm$1.35 & 90.26$\pm$1.40 & 89.96$\pm$1.50 & 88.55$\pm$1.58 \\
 & \textbf{Amazon} & 83.44$\pm$4.47 & 87.31$\pm$3.22 & 87.66$\pm$3.16 & 87.64$\pm$3.09 & 87.63$\pm$3.13 & 87.68$\pm$3.14 & 87.67$\pm$3.17 & 87.56$\pm$3.23 & 87.37$\pm$3.28 & 87.04$\pm$3.34 & 83.77$\pm$3.96 \\ \hline
\end{tabular}
\end{table}


\end{appendices}


\begin{thebibliography}{79}
\ifx \bisbn   \undefined \def \bisbn  #1{ISBN #1}\fi
\ifx \binits  \undefined \def \binits#1{#1}\fi
\ifx \bauthor  \undefined \def \bauthor#1{#1}\fi
\ifx \batitle  \undefined \def \batitle#1{#1}\fi
\ifx \bjtitle  \undefined \def \bjtitle#1{#1}\fi
\ifx \bvolume  \undefined \def \bvolume#1{\textbf{#1}}\fi
\ifx \byear  \undefined \def \byear#1{#1}\fi
\ifx \bissue  \undefined \def \bissue#1{#1}\fi
\ifx \bfpage  \undefined \def \bfpage#1{#1}\fi
\ifx \blpage  \undefined \def \blpage #1{#1}\fi
\ifx \burl  \undefined \def \burl#1{\textsf{#1}}\fi
\ifx \doiurl  \undefined \def \doiurl#1{\url{https://doi.org/#1}}\fi
\ifx \betal  \undefined \def \betal{\textit{et al.}}\fi
\ifx \binstitute  \undefined \def \binstitute#1{#1}\fi
\ifx \binstitutionaled  \undefined \def \binstitutionaled#1{#1}\fi
\ifx \bctitle  \undefined \def \bctitle#1{#1}\fi
\ifx \beditor  \undefined \def \beditor#1{#1}\fi
\ifx \bpublisher  \undefined \def \bpublisher#1{#1}\fi
\ifx \bbtitle  \undefined \def \bbtitle#1{#1}\fi
\ifx \bedition  \undefined \def \bedition#1{#1}\fi
\ifx \bseriesno  \undefined \def \bseriesno#1{#1}\fi
\ifx \blocation  \undefined \def \blocation#1{#1}\fi
\ifx \bsertitle  \undefined \def \bsertitle#1{#1}\fi
\ifx \bsnm \undefined \def \bsnm#1{#1}\fi
\ifx \bsuffix \undefined \def \bsuffix#1{#1}\fi
\ifx \bparticle \undefined \def \bparticle#1{#1}\fi
\ifx \barticle \undefined \def \barticle#1{#1}\fi
\bibcommenthead
\ifx \bconfdate \undefined \def \bconfdate #1{#1}\fi
\ifx \botherref \undefined \def \botherref #1{#1}\fi
\ifx \url \undefined \def \url#1{\textsf{#1}}\fi
\ifx \bchapter \undefined \def \bchapter#1{#1}\fi
\ifx \bbook \undefined \def \bbook#1{#1}\fi
\ifx \bcomment \undefined \def \bcomment#1{#1}\fi
\ifx \oauthor \undefined \def \oauthor#1{#1}\fi
\ifx \citeauthoryear \undefined \def \citeauthoryear#1{#1}\fi
\ifx \endbibitem  \undefined \def \endbibitem {}\fi
\ifx \bconflocation  \undefined \def \bconflocation#1{#1}\fi
\ifx \arxivurl  \undefined \def \arxivurl#1{\textsf{#1}}\fi
\csname PreBibitemsHook\endcsname

\bibitem[\protect\citeauthoryear{Reji~Chacko et~al.}{2025}]{reji2025species}
\begin{barticle}
\bauthor{\bsnm{Reji~Chacko}, \binits{M.}},
\bauthor{\bsnm{Albouy}, \binits{C.}},
\bauthor{\bsnm{Altermatt}, \binits{F.}},
\bauthor{\bsnm{Boussange}, \binits{V.}},
\bauthor{\bsnm{Br{\"a}ndle}, \binits{M.}},
\bauthor{\bsnm{Farwig}, \binits{N.}},
\bauthor{\bsnm{Gossner}, \binits{M.M.}},
\bauthor{\bsnm{Ho}, \binits{H.-C.}},
\bauthor{\bsnm{Joss}, \binits{A.}},
\bauthor{\bsnm{Neff}, \binits{F.}}, \betal:
\batitle{Species loss in key habitats accelerates regional food web disruption}.
\bjtitle{Communications Biology}
\bvolume{8}(\bissue{1}),
\bfpage{988}
(\byear{2025})
\end{barticle}
\endbibitem

\bibitem[\protect\citeauthoryear{Sporns et~al.}{2005}]{sporns2005human}
\begin{barticle}
\bauthor{\bsnm{Sporns}, \binits{O.}},
\bauthor{\bsnm{Tononi}, \binits{G.}},
\bauthor{\bsnm{K{\"o}tter}, \binits{R.}}:
\batitle{The human connectome: {A} structural description of the human brain}.
\bjtitle{PLoS Computational Biology}
\bvolume{1}(\bissue{4}),
\bfpage{42}
(\byear{2005})
\end{barticle}
\endbibitem

\bibitem[\protect\citeauthoryear{Pastor-Satorras and Vespignani}{2007}]{pastor2007evolution}
\begin{bbook}
\bauthor{\bsnm{Pastor-Satorras}, \binits{R.}},
\bauthor{\bsnm{Vespignani}, \binits{A.}}:
\bbtitle{Evolution and Structure of the Internet: {A} Statistical Physics Approach}.
\bpublisher{Cambridge University Press},
\blocation{Cambridge, UK}
(\byear{2007})
\end{bbook}
\endbibitem

\bibitem[\protect\citeauthoryear{Brummitt et~al.}{2012}]{brummitt2012suppressing}
\begin{barticle}
\bauthor{\bsnm{Brummitt}, \binits{C.D.}},
\bauthor{\bsnm{D’Souza}, \binits{R.M.}},
\bauthor{\bsnm{Leicht}, \binits{E.A.}}:
\batitle{Suppressing cascades of load in interdependent networks}.
\bjtitle{Proceedings of the National Academy of Sciences (PNAS)}
\bvolume{109}(\bissue{12}),
\bfpage{680}--\blpage{689}
(\byear{2012})
\end{barticle}
\endbibitem

\bibitem[\protect\citeauthoryear{Guimera et~al.}{2005}]{guimera2005worldwide}
\begin{barticle}
\bauthor{\bsnm{Guimera}, \binits{R.}},
\bauthor{\bsnm{Mossa}, \binits{S.}},
\bauthor{\bsnm{Turtschi}, \binits{A.}},
\bauthor{\bsnm{Amaral}, \binits{L.N.}}:
\batitle{The worldwide air transportation network: {A}nomalous centrality, community structure, and cities' global roles}.
\bjtitle{Proceedings of the National Academy of Sciences (PNAS)}
\bvolume{102}(\bissue{22}),
\bfpage{7794}--\blpage{7799}
(\byear{2005})
\end{barticle}
\endbibitem

\bibitem[\protect\citeauthoryear{Giot et~al.}{2003}]{giot2003protein}
\begin{barticle}
\bauthor{\bsnm{Giot}, \binits{L.}},
\bauthor{\bsnm{Bader}, \binits{J.S.}},
\bauthor{\bsnm{Brouwer}, \binits{C.}},
\bauthor{\bsnm{Chaudhuri}, \binits{A.}},
\bauthor{\bsnm{Kuang}, \binits{B.}},
\bauthor{\bsnm{Li}, \binits{Y.}},
\bauthor{\bsnm{Hao}, \binits{Y.}},
\bauthor{\bsnm{Ooi}, \binits{C.}},
\bauthor{\bsnm{Godwin}, \binits{B.}},
\bauthor{\bsnm{Vitols}, \binits{E.}}, \betal:
\batitle{A protein interaction map of {D}rosophila {M}elanogaster}.
\bjtitle{Science}
\bvolume{302}(\bissue{5651}),
\bfpage{1727}--\blpage{1736}
(\byear{2003})
\end{barticle}
\endbibitem

\bibitem[\protect\citeauthoryear{Schneider et~al.}{2011}]{schneider2011mitigation}
\begin{barticle}
\bauthor{\bsnm{Schneider}, \binits{C.M.}},
\bauthor{\bsnm{Moreira}, \binits{A.A.}},
\bauthor{\bsnm{Andrade~Jr}, \binits{J.S.}},
\bauthor{\bsnm{Havlin}, \binits{S.}},
\bauthor{\bsnm{Herrmann}, \binits{H.J.}}:
\batitle{Mitigation of malicious attacks on networks}.
\bjtitle{Proceedings of the National Academy of Sciences}
\bvolume{108}(\bissue{10}),
\bfpage{3838}--\blpage{3841}
(\byear{2011})
\end{barticle}
\endbibitem

\bibitem[\protect\citeauthoryear{Fortunato and Newman}{2022}]{fortunato202220}
\begin{barticle}
\bauthor{\bsnm{Fortunato}, \binits{S.}},
\bauthor{\bsnm{Newman}, \binits{M.E.}}:
\batitle{20 years of network community detection}.
\bjtitle{Nature Physics}
\bvolume{18}(\bissue{8}),
\bfpage{848}--\blpage{850}
(\byear{2022})
\end{barticle}
\endbibitem

\bibitem[\protect\citeauthoryear{Su et~al.}{2022}]{su2022comprehensive}
\begin{barticle}
\bauthor{\bsnm{Su}, \binits{X.}},
\bauthor{\bsnm{Xue}, \binits{S.}},
\bauthor{\bsnm{Liu}, \binits{F.}},
\bauthor{\bsnm{Wu}, \binits{J.}},
\bauthor{\bsnm{Yang}, \binits{J.}},
\bauthor{\bsnm{Zhou}, \binits{C.}},
\bauthor{\bsnm{Hu}, \binits{W.}},
\bauthor{\bsnm{Paris}, \binits{C.}},
\bauthor{\bsnm{Nepal}, \binits{S.}},
\bauthor{\bsnm{Jin}, \binits{D.}}, \betal:
\batitle{A comprehensive survey on community detection with deep learning}.
\bjtitle{IEEE Transactions on Neural Networks and Learning Systems (TNNLS)}
\bvolume{35}(\bissue{4}),
\bfpage{4682}--\blpage{4702}
(\byear{2022})
\end{barticle}
\endbibitem

\bibitem[\protect\citeauthoryear{Qiao et~al.}{2025}]{qiao2025deep}
\begin{botherref}
\oauthor{\bsnm{Qiao}, \binits{H.}},
\oauthor{\bsnm{Tong}, \binits{H.}},
\oauthor{\bsnm{An}, \binits{B.}},
\oauthor{\bsnm{King}, \binits{I.}},
\oauthor{\bsnm{Aggarwal}, \binits{C.}},
\oauthor{\bsnm{Pang}, \binits{G.}}:
Deep graph anomaly detection: {A} survey and new perspectives.
IEEE Transactions on Knowledge and Data Engineering (TKDE)
(2025)
\end{botherref}
\endbibitem

\bibitem[\protect\citeauthoryear{Ekle and Eberle}{2024}]{ekle2024anomaly}
\begin{barticle}
\bauthor{\bsnm{Ekle}, \binits{O.A.}},
\bauthor{\bsnm{Eberle}, \binits{W.}}:
\batitle{Anomaly detection in dynamic graphs: {A} comprehensive survey}.
\bjtitle{ACM Transactions on Knowledge Discovery from Data (TKDD)}
\bvolume{18}(\bissue{8}),
\bfpage{1}--\blpage{44}
(\byear{2024})
\end{barticle}
\endbibitem

\bibitem[\protect\citeauthoryear{Mart{\'\i}nez et~al.}{2016}]{martinez2016survey}
\begin{barticle}
\bauthor{\bsnm{Mart{\'\i}nez}, \binits{V.}},
\bauthor{\bsnm{Berzal}, \binits{F.}},
\bauthor{\bsnm{Cubero}, \binits{J.-C.}}:
\batitle{A survey of link prediction in complex networks}.
\bjtitle{ACM computing surveys (CSUR)}
\bvolume{49}(\bissue{4}),
\bfpage{1}--\blpage{33}
(\byear{2016})
\end{barticle}
\endbibitem

\bibitem[\protect\citeauthoryear{Kumar et~al.}{2020}]{kumar2020link}
\begin{barticle}
\bauthor{\bsnm{Kumar}, \binits{A.}},
\bauthor{\bsnm{Singh}, \binits{S.S.}},
\bauthor{\bsnm{Singh}, \binits{K.}},
\bauthor{\bsnm{Biswas}, \binits{B.}}:
\batitle{Link prediction techniques, applications, and performance: {A} survey}.
\bjtitle{Physica A: Statistical Mechanics and its Applications}
\bvolume{553},
\bfpage{124289}
(\byear{2020})
\end{barticle}
\endbibitem

\bibitem[\protect\citeauthoryear{Qin and Yeung}{2023}]{qin2023temporal}
\begin{barticle}
\bauthor{\bsnm{Qin}, \binits{M.}},
\bauthor{\bsnm{Yeung}, \binits{D.-Y.}}:
\batitle{Temporal link prediction: {A} unified framework, taxonomy, and review}.
\bjtitle{ACM Computing Surveys (CSUR)}
\bvolume{56}(\bissue{4}),
\bfpage{1}--\blpage{40}
(\byear{2023})
\end{barticle}
\endbibitem

\bibitem[\protect\citeauthoryear{Bhattacharya et~al.}{2026}]{Bhattacharya2026}
\begin{bbook}
\bauthor{\bsnm{Bhattacharya}, \binits{R.}},
\bauthor{\bsnm{Nagwani}, \binits{N.K.}},
\bauthor{\bsnm{Asudani}, \binits{D.S.}},
\bauthor{\bsnm{Chhabra}, \binits{G.S.}},
\bauthor{\bsnm{Bhattacharya}, \binits{S.}},
\bauthor{\bsnm{Kadam}, \binits{S.}}:
In: \beditor{\bsnm{Bhattacharya}, \binits{R.}},
\beditor{\bsnm{Rathore}, \binits{Y.K.}},
\beditor{\bsnm{Tran}, \binits{T.A.}},
\beditor{\bsnm{Swarnkar}, \binits{S.K.}} (eds.)
\bbtitle{A Comprehensive Overview of Graph Convolutional Network},
pp. \bfpage{1}--\blpage{19}.
\bpublisher{Springer},
\blocation{Cham, Switzerland}
(\byear{2026})
\end{bbook}
\endbibitem

\bibitem[\protect\citeauthoryear{Shehzad et~al.}{2026}]{shehzad2026graph}
\begin{botherref}
\oauthor{\bsnm{Shehzad}, \binits{A.}},
\oauthor{\bsnm{Xia}, \binits{F.}},
\oauthor{\bsnm{Abid}, \binits{S.}},
\oauthor{\bsnm{Peng}, \binits{C.}},
\oauthor{\bsnm{Yu}, \binits{S.}},
\oauthor{\bsnm{Zhang}, \binits{D.}},
\oauthor{\bsnm{Verspoor}, \binits{K.}}:
Graph transformers: {A} survey.
IEEE Transactions on Neural Networks and Learning Systems (TNNLS)
(2026)
\end{botherref}
\endbibitem

\bibitem[\protect\citeauthoryear{Rossi et~al.}{2020}]{rossi2020proximity}
\begin{barticle}
\bauthor{\bsnm{Rossi}, \binits{R.A.}},
\bauthor{\bsnm{Jin}, \binits{D.}},
\bauthor{\bsnm{Kim}, \binits{S.}},
\bauthor{\bsnm{Ahmed}, \binits{N.K.}},
\bauthor{\bsnm{Koutra}, \binits{D.}},
\bauthor{\bsnm{Lee}, \binits{J.B.}}:
\batitle{On proximity and structural role-based embeddings in networks: {M}isconceptions, techniques, and applications}.
\bjtitle{ACM Transactions on Knowledge Discovery from Data (TKDD)}
\bvolume{14}(\bissue{5}),
\bfpage{1}--\blpage{37}
(\byear{2020})
\end{barticle}
\endbibitem

\bibitem[\protect\citeauthoryear{Zhu et~al.}{2021}]{zhu2021node}
\begin{bchapter}
\bauthor{\bsnm{Zhu}, \binits{J.}},
\bauthor{\bsnm{Lu}, \binits{X.}},
\bauthor{\bsnm{Heimann}, \binits{M.}},
\bauthor{\bsnm{Koutra}, \binits{D.}}:
\bctitle{Node proximity is all you need: {U}nified structural and positional node and graph embedding}.
In: \bbtitle{Proceedings of the 2021 SIAM International Conference on Data Mining (SDM)},
pp. \bfpage{163}--\blpage{171}
(\byear{2021}).
\bcomment{SIAM}
\end{bchapter}
\endbibitem

\bibitem[\protect\citeauthoryear{Grando et~al.}{2018}]{grando2018machine}
\begin{barticle}
\bauthor{\bsnm{Grando}, \binits{F.}},
\bauthor{\bsnm{Granville}, \binits{L.Z.}},
\bauthor{\bsnm{Lamb}, \binits{L.C.}}:
\batitle{Machine learning in network centrality measures: {T}utorial and outlook}.
\bjtitle{ACM Computing Surveys (CSUR)}
\bvolume{51}(\bissue{5}),
\bfpage{1}--\blpage{32}
(\byear{2018})
\end{barticle}
\endbibitem

\bibitem[\protect\citeauthoryear{Qin and Yeung}{2024}]{mengirwe}
\begin{botherref}
\oauthor{\bsnm{Qin}, \binits{M.}},
\oauthor{\bsnm{Yeung}, \binits{D.-Y.}}:
{IRWE}: Inductive random walk for joint inference of identity and position network embedding.
Transactions on Machine Learning Research (TMLR)
\textbf{2024}
(2024)
\end{botherref}
\endbibitem

\bibitem[\protect\citeauthoryear{Ribeiro et~al.}{2017}]{ribeiro2017struc2vec}
\begin{bchapter}
\bauthor{\bsnm{Ribeiro}, \binits{L.F.}},
\bauthor{\bsnm{Saverese}, \binits{P.H.}},
\bauthor{\bsnm{Figueiredo}, \binits{D.R.}}:
\bctitle{struc2vec: {L}earning node representations from structural identity}.
In: \bbtitle{Proceedings of the 23rd ACM SIGKDD International Conference on Knowledge Discovery and Data Mining},
pp. \bfpage{385}--\blpage{394}
(\byear{2017})
\end{bchapter}
\endbibitem

\bibitem[\protect\citeauthoryear{Ju et~al.}{2024}]{ju2024comprehensive}
\begin{barticle}
\bauthor{\bsnm{Ju}, \binits{W.}},
\bauthor{\bsnm{Fang}, \binits{Z.}},
\bauthor{\bsnm{Gu}, \binits{Y.}},
\bauthor{\bsnm{Liu}, \binits{Z.}},
\bauthor{\bsnm{Long}, \binits{Q.}},
\bauthor{\bsnm{Qiao}, \binits{Z.}},
\bauthor{\bsnm{Qin}, \binits{Y.}},
\bauthor{\bsnm{Shen}, \binits{J.}},
\bauthor{\bsnm{Sun}, \binits{F.}},
\bauthor{\bsnm{Xiao}, \binits{Z.}}, \betal:
\batitle{A comprehensive survey on deep graph representation learning}.
\bjtitle{Neural Networks}
\bvolume{173},
\bfpage{106207}
(\byear{2024})
\end{barticle}
\endbibitem

\bibitem[\protect\citeauthoryear{Ju et~al.}{2025}]{ju2025survey}
\begin{botherref}
\oauthor{\bsnm{Ju}, \binits{W.}},
\oauthor{\bsnm{Yi}, \binits{S.}},
\oauthor{\bsnm{Wang}, \binits{Y.}},
\oauthor{\bsnm{Xiao}, \binits{Z.}},
\oauthor{\bsnm{Mao}, \binits{Z.}},
\oauthor{\bsnm{Li}, \binits{H.}},
\oauthor{\bsnm{Gu}, \binits{Y.}},
\oauthor{\bsnm{Qin}, \binits{Y.}},
\oauthor{\bsnm{Yin}, \binits{N.}},
\oauthor{\bsnm{Wang}, \binits{S.}}, et al.:
A survey of graph neural networks in real world: {I}mbalance, noise, privacy and ood challenges.
IEEE Transactions on Pattern Analysis and Machine Intelligence (TPAMI)
(2025)
\end{botherref}
\endbibitem

\bibitem[\protect\citeauthoryear{Zheng et~al.}{2026}]{zheng2026graph}
\begin{botherref}
\oauthor{\bsnm{Zheng}, \binits{X.}},
\oauthor{\bsnm{Wang}, \binits{Y.}},
\oauthor{\bsnm{Liu}, \binits{Y.}},
\oauthor{\bsnm{Li}, \binits{M.}},
\oauthor{\bsnm{Zhang}, \binits{M.}},
\oauthor{\bsnm{Jin}, \binits{D.}},
\oauthor{\bsnm{Yu}, \binits{P.S.}},
\oauthor{\bsnm{Pan}, \binits{S.}}:
Graph neural networks for graphs with heterophily: {A} survey.
IEEE Transactions on Knowledge and Data Engineering (TKDE)
(2026)
\end{botherref}
\endbibitem

\bibitem[\protect\citeauthoryear{Newman and Clauset}{2016}]{newman2016structure}
\begin{barticle}
\bauthor{\bsnm{Newman}, \binits{M.E.}},
\bauthor{\bsnm{Clauset}, \binits{A.}}:
\batitle{Structure and inference in annotated networks}.
\bjtitle{Nature communications}
\bvolume{7}(\bissue{1}),
\bfpage{11863}
(\byear{2016})
\end{barticle}
\endbibitem

\bibitem[\protect\citeauthoryear{Peel et~al.}{2017}]{peel2017ground}
\begin{barticle}
\bauthor{\bsnm{Peel}, \binits{L.}},
\bauthor{\bsnm{Larremore}, \binits{D.B.}},
\bauthor{\bsnm{Clauset}, \binits{A.}}:
\batitle{The ground truth about metadata and community detection in networks}.
\bjtitle{Science Advances}
\bvolume{3}(\bissue{5}),
\bfpage{1602548}
(\byear{2017})
\end{barticle}
\endbibitem

\bibitem[\protect\citeauthoryear{Qin et~al.}{2018}]{qin2018adaptive}
\begin{barticle}
\bauthor{\bsnm{Qin}, \binits{M.}},
\bauthor{\bsnm{Jin}, \binits{D.}},
\bauthor{\bsnm{Lei}, \binits{K.}},
\bauthor{\bsnm{Gabrys}, \binits{B.}},
\bauthor{\bsnm{Musial-Gabrys}, \binits{K.}}:
\batitle{Adaptive community detection incorporating topology and content in social networks}.
\bjtitle{Knowledge-Based Systems}
\bvolume{161},
\bfpage{342}--\blpage{356}
(\byear{2018})
\end{barticle}
\endbibitem

\bibitem[\protect\citeauthoryear{Qin and Lei}{2021}]{qin2021dual}
\begin{barticle}
\bauthor{\bsnm{Qin}, \binits{M.}},
\bauthor{\bsnm{Lei}, \binits{K.}}:
\batitle{Dual-channel hybrid community detection in attributed networks}.
\bjtitle{Information Sciences}
\bvolume{551},
\bfpage{146}--\blpage{167}
(\byear{2021})
\end{barticle}
\endbibitem

\bibitem[\protect\citeauthoryear{Wang et~al.}{2020}]{wang2020gcn}
\begin{bchapter}
\bauthor{\bsnm{Wang}, \binits{X.}},
\bauthor{\bsnm{Zhu}, \binits{M.}},
\bauthor{\bsnm{Bo}, \binits{D.}},
\bauthor{\bsnm{Cui}, \binits{P.}},
\bauthor{\bsnm{Shi}, \binits{C.}},
\bauthor{\bsnm{Pei}, \binits{J.}}:
\bctitle{{AM-GCN}: Adaptive multi-channel graph convolutional networks}.
In: \bbtitle{Proceedings of the 26th ACM SIGKDD International Conference on Knowledge Discovery and Data Mining},
pp. \bfpage{1243}--\blpage{1253}
(\byear{2020})
\end{bchapter}
\endbibitem

\bibitem[\protect\citeauthoryear{Zhang et~al.}{2024}]{zhang2024disentangled}
\begin{barticle}
\bauthor{\bsnm{Zhang}, \binits{W.}},
\bauthor{\bsnm{Pan}, \binits{L.}},
\bauthor{\bsnm{Guo}, \binits{X.}},
\bauthor{\bsnm{Jiao}, \binits{P.}}:
\batitle{Disentangled representation learning for structural role discovery}.
\bjtitle{IEEE Transactions on Computational Social Systems (TCSS)}
\bvolume{11}(\bissue{6}),
\bfpage{7200}--\blpage{7211}
(\byear{2024})
\end{barticle}
\endbibitem

\bibitem[\protect\citeauthoryear{Tang et~al.}{2025}]{tang2025hyperrole}
\begin{bchapter}
\bauthor{\bsnm{Tang}, \binits{H.}},
\bauthor{\bsnm{Du}, \binits{M.}},
\bauthor{\bsnm{Jiao}, \binits{P.}},
\bauthor{\bsnm{Wu}, \binits{H.}},
\bauthor{\bsnm{Zhao}, \binits{Z.}}:
\bctitle{{H}yper{R}ole: {H}yperbolic graph transformer for role discovery in online social networks}.
In: \bbtitle{IEEE INFOCOM 2025 - Proceedings of the 2025 IEEE Conference on Computer Communications},
pp. \bfpage{1}--\blpage{10}
(\byear{2025}).
\bcomment{IEEE}
\end{bchapter}
\endbibitem

\bibitem[\protect\citeauthoryear{You et~al.}{2021}]{you2021identity}
\begin{bchapter}
\bauthor{\bsnm{You}, \binits{J.}},
\bauthor{\bsnm{Gomes-Selman}, \binits{J.M.}},
\bauthor{\bsnm{Ying}, \binits{R.}},
\bauthor{\bsnm{Leskovec}, \binits{J.}}:
\bctitle{Identity-aware graph neural networks}.
In: \bbtitle{Proceedings of the 35th AAAI Conference on Artificial Intelligence},
vol. \bseriesno{35},
pp. \bfpage{10737}--\blpage{10745}
(\byear{2021})
\end{bchapter}
\endbibitem

\bibitem[\protect\citeauthoryear{You et~al.}{2019}]{you2019position}
\begin{bchapter}
\bauthor{\bsnm{You}, \binits{J.}},
\bauthor{\bsnm{Ying}, \binits{R.}},
\bauthor{\bsnm{Leskovec}, \binits{J.}}:
\bctitle{Position-aware graph neural networks}.
In: \bbtitle{Proceedings of the 36th International Conference on Machine Learning (ICML)},
pp. \bfpage{7134}--\blpage{7143}
(\byear{2019}).
\bcomment{PMLR}
\end{bchapter}
\endbibitem

\bibitem[\protect\citeauthoryear{Devvrit et~al.}{2022}]{devvrit2022s3gc}
\begin{bchapter}
\bauthor{\bsnm{Devvrit}, \binits{F.}},
\bauthor{\bsnm{Sinha}, \binits{A.}},
\bauthor{\bsnm{Dhillon}, \binits{I.}},
\bauthor{\bsnm{Jain}, \binits{P.}}:
\bctitle{{S3GC}: {S}calable self-supervised graph clustering}.
In: \bbtitle{Proceedings of the 36th International Conference on Neural Information Processing Systems (NeurIPS)},
pp. \bfpage{3248}--\blpage{3261}
(\byear{2022})
\end{bchapter}
\endbibitem

\bibitem[\protect\citeauthoryear{Liu et~al.}{2024}]{liu2024revisiting}
\begin{bchapter}
\bauthor{\bsnm{Liu}, \binits{Y.}},
\bauthor{\bsnm{Li}, \binits{J.}},
\bauthor{\bsnm{Chen}, \binits{Y.}},
\bauthor{\bsnm{Wu}, \binits{R.}},
\bauthor{\bsnm{Wang}, \binits{E.}},
\bauthor{\bsnm{Zhou}, \binits{J.}},
\bauthor{\bsnm{Tian}, \binits{S.}},
\bauthor{\bsnm{Shen}, \binits{S.}},
\bauthor{\bsnm{Fu}, \binits{X.}},
\bauthor{\bsnm{Meng}, \binits{C.}}, \betal:
\bctitle{Revisiting modularity maximization for graph clustering: {A} contrastive learning perspective}.
In: \bbtitle{Proceedings of the 30th ACM SIGKDD Conference on Knowledge Discovery and Data Mining},
pp. \bfpage{1968}--\blpage{1979}
(\byear{2024})
\end{bchapter}
\endbibitem

\bibitem[\protect\citeauthoryear{Grover and Leskovec}{2016}]{grover2016node2vec}
\begin{bchapter}
\bauthor{\bsnm{Grover}, \binits{A.}},
\bauthor{\bsnm{Leskovec}, \binits{J.}}:
\bctitle{node2vec: {S}calable feature learning for networks}.
In: \bbtitle{Proceedings of the 22nd ACM SIGKDD International Conference on Knowledge Discovery and Data Mining},
pp. \bfpage{855}--\blpage{864}
(\byear{2016})
\end{bchapter}
\endbibitem

\bibitem[\protect\citeauthoryear{Ivanov and Burnaev}{2018}]{ivanov2018anonymous}
\begin{bchapter}
\bauthor{\bsnm{Ivanov}, \binits{S.}},
\bauthor{\bsnm{Burnaev}, \binits{E.}}:
\bctitle{Anonymous walk embeddings}.
In: \bbtitle{Proceedings of the 35th International Conference on Machine Learning (ICML)},
pp. \bfpage{2186}--\blpage{2195}
(\byear{2018}).
\bcomment{PMLR}
\end{bchapter}
\endbibitem

\bibitem[\protect\citeauthoryear{Micali and Zhu}{2016}]{micali2016reconstructing}
\begin{barticle}
\bauthor{\bsnm{Micali}, \binits{S.}},
\bauthor{\bsnm{Zhu}, \binits{Z.A.}}:
\batitle{Reconstructing markov processes from independent and anonymous experiments}.
\bjtitle{Discrete Applied Mathematics}
\bvolume{200},
\bfpage{108}--\blpage{122}
(\byear{2016})
\end{barticle}
\endbibitem

\bibitem[\protect\citeauthoryear{Zachary}{1977}]{zachary1977information}
\begin{barticle}
\bauthor{\bsnm{Zachary}, \binits{W.W.}}:
\batitle{An information flow model for conflict and fission in small groups}.
\bjtitle{Journal of anthropological research}
\bvolume{33}(\bissue{4}),
\bfpage{452}--\blpage{473}
(\byear{1977})
\end{barticle}
\endbibitem

\bibitem[\protect\citeauthoryear{Van~der Maaten and Hinton}{2008}]{van2008visualizing}
\begin{botherref}
\oauthor{\bsnm{Maaten}, \binits{L.}},
\oauthor{\bsnm{Hinton}, \binits{G.}}:
Visualizing data using t-{SNE}.
Journal of Machine Learning Research (JMLR)
\textbf{9}(11)
(2008)
\end{botherref}
\endbibitem

\bibitem[\protect\citeauthoryear{Jiao et~al.}{2021}]{jiao2021survey}
\begin{barticle}
\bauthor{\bsnm{Jiao}, \binits{P.}},
\bauthor{\bsnm{Guo}, \binits{X.}},
\bauthor{\bsnm{Pan}, \binits{T.}},
\bauthor{\bsnm{Zhang}, \binits{W.}},
\bauthor{\bsnm{Pei}, \binits{Y.}},
\bauthor{\bsnm{Pan}, \binits{L.}}:
\batitle{A survey on role-oriented network embedding}.
\bjtitle{IEEE Transactions on Big Data}
\bvolume{8}(\bissue{4}),
\bfpage{933}--\blpage{952}
(\byear{2021})
\end{barticle}
\endbibitem

\bibitem[\protect\citeauthoryear{Yang and Leskovec}{2015}]{yang2015defining}
\begin{barticle}
\bauthor{\bsnm{Yang}, \binits{J.}},
\bauthor{\bsnm{Leskovec}, \binits{J.}}:
\batitle{Defining and evaluating network communities based on ground-truth}.
\bjtitle{Knowledge and Information Systems}
\bvolume{42}(\bissue{1}),
\bfpage{181}--\blpage{213}
(\byear{2015})
\end{barticle}
\endbibitem

\bibitem[\protect\citeauthoryear{Von~Luxburg}{2007}]{von2007tutorial}
\begin{barticle}
\bauthor{\bsnm{Von~Luxburg}, \binits{U.}}:
\batitle{A tutorial on spectral clustering}.
\bjtitle{Statistics and computing}
\bvolume{17}(\bissue{4}),
\bfpage{395}--\blpage{416}
(\byear{2007})
\end{barticle}
\endbibitem

\bibitem[\protect\citeauthoryear{Bhowmick et~al.}{2020}]{bhowmick2020louvainne}
\begin{bchapter}
\bauthor{\bsnm{Bhowmick}, \binits{A.K.}},
\bauthor{\bsnm{Meneni}, \binits{K.}},
\bauthor{\bsnm{Danisch}, \binits{M.}},
\bauthor{\bsnm{Guillaume}, \binits{J.-L.}},
\bauthor{\bsnm{Mitra}, \binits{B.}}:
\bctitle{{L}ouvain{NE}: {H}ierarchical {L}ouvain method for high quality and scalable network embedding}.
In: \bbtitle{Proceedings of the 13th International Conference on Web Search and Data Mining (WSDM)},
pp. \bfpage{43}--\blpage{51}
(\byear{2020})
\end{bchapter}
\endbibitem

\bibitem[\protect\citeauthoryear{Xie et~al.}{2023}]{xie2023sketchne}
\begin{barticle}
\bauthor{\bsnm{Xie}, \binits{Y.}},
\bauthor{\bsnm{Dong}, \binits{Y.}},
\bauthor{\bsnm{Qiu}, \binits{J.}},
\bauthor{\bsnm{Yu}, \binits{W.}},
\bauthor{\bsnm{Feng}, \binits{X.}},
\bauthor{\bsnm{Tang}, \binits{J.}}:
\batitle{{S}ketch{NE}: {E}mbedding billion-scale networks accurately in one hour}.
\bjtitle{IEEE Transactions on Knowledge and Data Engineering (TKDE)}
\bvolume{35}(\bissue{10}),
\bfpage{10666}--\blpage{10680}
(\byear{2023})
\end{barticle}
\endbibitem

\bibitem[\protect\citeauthoryear{Yang et~al.}{2020}]{yang2020homogeneous}
\begin{barticle}
\bauthor{\bsnm{Yang}, \binits{R.}},
\bauthor{\bsnm{Shi}, \binits{J.}},
\bauthor{\bsnm{Xiao}, \binits{X.}},
\bauthor{\bsnm{Yang}, \binits{Y.}},
\bauthor{\bsnm{Bhowmick}, \binits{S.S.}}:
\batitle{Homogeneous network embedding for massive graphs via reweighted personalized {P}age{R}ank}.
\bjtitle{Proceedings of the VLDB Endowment}
\bvolume{13}(\bissue{5}),
\bfpage{670}--\blpage{683}
(\byear{2020})
\end{barticle}
\endbibitem

\bibitem[\protect\citeauthoryear{Milo et~al.}{2004}]{milo2004superfamilies}
\begin{barticle}
\bauthor{\bsnm{Milo}, \binits{R.}},
\bauthor{\bsnm{Itzkovitz}, \binits{S.}},
\bauthor{\bsnm{Kashtan}, \binits{N.}},
\bauthor{\bsnm{Levitt}, \binits{R.}},
\bauthor{\bsnm{Shen-Orr}, \binits{S.}},
\bauthor{\bsnm{Ayzenshtat}, \binits{I.}},
\bauthor{\bsnm{Sheffer}, \binits{M.}},
\bauthor{\bsnm{Alon}, \binits{U.}}:
\batitle{Superfamilies of evolved and designed networks}.
\bjtitle{Science}
\bvolume{303}(\bissue{5663}),
\bfpage{1538}--\blpage{1542}
(\byear{2004})
\end{barticle}
\endbibitem

\bibitem[\protect\citeauthoryear{Dong et~al.}{2017}]{dong2017structural}
\begin{bchapter}
\bauthor{\bsnm{Dong}, \binits{Y.}},
\bauthor{\bsnm{Johnson}, \binits{R.A.}},
\bauthor{\bsnm{Xu}, \binits{J.}},
\bauthor{\bsnm{Chawla}, \binits{N.V.}}:
\bctitle{Structural diversity and homophily: {A} study across more than one hundred big networks}.
In: \bbtitle{Proceedings of the 23rd ACM SIGKDD International Conference on Knowledge Discovery and Data Mining},
pp. \bfpage{807}--\blpage{816}
(\byear{2017})
\end{bchapter}
\endbibitem

\bibitem[\protect\citeauthoryear{Lancichinetti et~al.}{2008}]{lancichinetti2008benchmark}
\begin{barticle}
\bauthor{\bsnm{Lancichinetti}, \binits{A.}},
\bauthor{\bsnm{Fortunato}, \binits{S.}},
\bauthor{\bsnm{Radicchi}, \binits{F.}}:
\batitle{Benchmark graphs for testing community detection algorithms}.
\bjtitle{Physical Review E—Statistical, Nonlinear, and Soft Matter Physics}
\bvolume{78}(\bissue{4}),
\bfpage{046110}
(\byear{2008})
\end{barticle}
\endbibitem

\bibitem[\protect\citeauthoryear{Milo et~al.}{2002}]{milo2002network}
\begin{barticle}
\bauthor{\bsnm{Milo}, \binits{R.}},
\bauthor{\bsnm{Shen-Orr}, \binits{S.}},
\bauthor{\bsnm{Itzkovitz}, \binits{S.}},
\bauthor{\bsnm{Kashtan}, \binits{N.}},
\bauthor{\bsnm{Chklovskii}, \binits{D.}},
\bauthor{\bsnm{Alon}, \binits{U.}}:
\batitle{Network motifs: {S}imple building blocks of complex networks}.
\bjtitle{Science}
\bvolume{298}(\bissue{5594}),
\bfpage{824}--\blpage{827}
(\byear{2002})
\end{barticle}
\endbibitem

\bibitem[\protect\citeauthoryear{Ugander et~al.}{2012}]{ugander2012structural}
\begin{barticle}
\bauthor{\bsnm{Ugander}, \binits{J.}},
\bauthor{\bsnm{Backstrom}, \binits{L.}},
\bauthor{\bsnm{Marlow}, \binits{C.}},
\bauthor{\bsnm{Kleinberg}, \binits{J.}}:
\batitle{Structural diversity in social contagion}.
\bjtitle{Proceedings of the National Academy of Sciences (PNAS)}
\bvolume{109}(\bissue{16}),
\bfpage{5962}--\blpage{5966}
(\byear{2012})
\end{barticle}
\endbibitem

\bibitem[\protect\citeauthoryear{Srinivasan and Ribeiro}{2020}]{srinivasanequivalence}
\begin{bchapter}
\bauthor{\bsnm{Srinivasan}, \binits{B.}},
\bauthor{\bsnm{Ribeiro}, \binits{B.}}:
\bctitle{On the equivalence between positional node embeddings and structural graph representations}.
In: \bbtitle{Proceedings of the 8th International Conference on Learning Representations (ICLR)}
(\byear{2020}).
\burl{https://openreview.net/forum?id=SJxzFySKwH}
\end{bchapter}
\endbibitem

\bibitem[\protect\citeauthoryear{Yan et~al.}{2024}]{yan2024pacer}
\begin{bchapter}
\bauthor{\bsnm{Yan}, \binits{Y.}},
\bauthor{\bsnm{Hu}, \binits{Y.}},
\bauthor{\bsnm{Zhou}, \binits{Q.}},
\bauthor{\bsnm{Liu}, \binits{L.}},
\bauthor{\bsnm{Zeng}, \binits{Z.}},
\bauthor{\bsnm{Chen}, \binits{Y.}},
\bauthor{\bsnm{Pan}, \binits{M.}},
\bauthor{\bsnm{Chen}, \binits{H.}},
\bauthor{\bsnm{Das}, \binits{M.}},
\bauthor{\bsnm{Tong}, \binits{H.}}:
\bctitle{{P}a{CE}r: Network embedding from positional to structural}.
In: \bbtitle{Proceedings of the ACM Web Conference 2024},
pp. \bfpage{2485}--\blpage{2496}
(\byear{2024})
\end{bchapter}
\endbibitem

\bibitem[\protect\citeauthoryear{Qin et~al.}{2025a}]{qin2025efficient}
\begin{bchapter}
\bauthor{\bsnm{Qin}, \binits{M.}},
\bauthor{\bsnm{Liu}, \binits{J.}},
\bauthor{\bsnm{King}, \binits{I.}}:
\bctitle{Efficient identity and position graph embedding via spectral-based random feature aggregation}.
In: \bbtitle{Proceedings of the 31st ACM SIGKDD Conference on Knowledge Discovery and Data Mining V.2},
pp. \bfpage{2350}--\blpage{2361}
(\byear{2025})
\end{bchapter}
\endbibitem

\bibitem[\protect\citeauthoryear{Qin et~al.}{2025b}]{qin2025infraredgp}
\begin{bchapter}
\bauthor{\bsnm{Qin}, \binits{M.}},
\bauthor{\bsnm{Li}, \binits{W.}},
\bauthor{\bsnm{Cui}, \binits{J.}},
\bauthor{\bsnm{Pei}, \binits{S.}}:
\bctitle{{I}nfrared{GP}: Efficient graph partitioning via spectral graph neural networks with negative corrections}.
In: \bbtitle{2025 IEEE High Performance Extreme Computing Conference (HPEC)},
pp. \bfpage{1}--\blpage{7}
(\byear{2025}).
\bcomment{IEEE}
\end{bchapter}
\endbibitem

\bibitem[\protect\citeauthoryear{Jin et~al.}{2020}]{jin2020gralsp}
\begin{bchapter}
\bauthor{\bsnm{Jin}, \binits{Y.}},
\bauthor{\bsnm{Song}, \binits{G.}},
\bauthor{\bsnm{Shi}, \binits{C.}}:
\bctitle{{G}ra{LSP}: {G}raph neural networks with local structural patterns}.
In: \bbtitle{Proceedings of the AAAI Conference on Artificial Intelligence},
vol. \bseriesno{34},
pp. \bfpage{4361}--\blpage{4368}
(\byear{2020})
\end{bchapter}
\endbibitem

\bibitem[\protect\citeauthoryear{Yan et~al.}{2024}]{yan2024topological}
\begin{bchapter}
\bauthor{\bsnm{Yan}, \binits{Y.}},
\bauthor{\bsnm{Hu}, \binits{Y.}},
\bauthor{\bsnm{Zhou}, \binits{Q.}},
\bauthor{\bsnm{Wu}, \binits{S.}},
\bauthor{\bsnm{Wang}, \binits{D.}},
\bauthor{\bsnm{Tong}, \binits{H.}}:
\bctitle{Topological anonymous walk embedding: {A} new structural node embedding approach}.
In: \bbtitle{Proceedings of the 33rd ACM International Conference on Information and Knowledge Management (CIKM)},
pp. \bfpage{2796}--\blpage{2806}
(\byear{2024})
\end{bchapter}
\endbibitem

\bibitem[\protect\citeauthoryear{Joly et~al.}{2025}]{joly2025taxonomy}
\begin{bchapter}
\bauthor{\bsnm{Joly}, \binits{A.}},
\bauthor{\bsnm{Keriven}, \binits{N.}},
\bauthor{\bsnm{Roumy}, \binits{A.}}:
\bctitle{Taxonomy of reduction matrices for graph coarsening}.
In: \bbtitle{Advances in Neural Information Processing Systems (NeurIPS)},
pp. \bfpage{96058}--\blpage{96083}
(\byear{2025})
\end{bchapter}
\endbibitem

\bibitem[\protect\citeauthoryear{Huang and Villar}{2021}]{huang2021short}
\begin{bchapter}
\bauthor{\bsnm{Huang}, \binits{N.T.}},
\bauthor{\bsnm{Villar}, \binits{S.}}:
\bctitle{A short tutorial on the {W}eisfeiler-{L}ehman test and its variants}.
In: \bbtitle{Proceedings of the 2021 IEEE International Conference on Acoustics, Speech and Signal Processing (ICASSP)},
pp. \bfpage{8533}--\blpage{8537}
(\byear{2021}).
\bcomment{IEEE}
\end{bchapter}
\endbibitem

\bibitem[\protect\citeauthoryear{Peng et~al.}{2021}]{peng2021hyperbolic}
\begin{barticle}
\bauthor{\bsnm{Peng}, \binits{W.}},
\bauthor{\bsnm{Varanka}, \binits{T.}},
\bauthor{\bsnm{Mostafa}, \binits{A.}},
\bauthor{\bsnm{Shi}, \binits{H.}},
\bauthor{\bsnm{Zhao}, \binits{G.}}:
\batitle{Hyperbolic deep neural networks: {A} survey}.
\bjtitle{IEEE Transactions on Pattern Analysis and Machine Intelligence (TPAMI)}
\bvolume{44}(\bissue{12}),
\bfpage{10023}--\blpage{10044}
(\byear{2021})
\end{barticle}
\endbibitem

\bibitem[\protect\citeauthoryear{Sadat et~al.}{2026}]{sadat2026hyperbolic}
\begin{botherref}
\oauthor{\bsnm{Sadat}, \binits{S.A.}},
\oauthor{\bsnm{Miliani}, \binits{M.Y.T.}},
\oauthor{\bsnm{El~Hames}, \binits{K.H.}},
\oauthor{\bsnm{Seba}, \binits{H.}},
\oauthor{\bsnm{Haddad}, \binits{M.}}:
Hyperbolic graph embedding: {A} survey and an evaluation on anomaly detection.
Pattern Recognition,
114385
(2026)
\end{botherref}
\endbibitem

\bibitem[\protect\citeauthoryear{Hamilton et~al.}{2017}]{hamilton2017inductive}
\begin{bchapter}
\bauthor{\bsnm{Hamilton}, \binits{W.L.}},
\bauthor{\bsnm{Ying}, \binits{R.}},
\bauthor{\bsnm{Leskovec}, \binits{J.}}:
\bctitle{Inductive representation learning on large graphs}.
In: \bbtitle{Proceedings of the 31st International Conference on Neural Information Processing Systems (NeurIPS)},
pp. \bfpage{1025}--\blpage{1035}
(\byear{2017})
\end{bchapter}
\endbibitem

\bibitem[\protect\citeauthoryear{Wu et~al.}{2019}]{wu2019simplifying}
\begin{bchapter}
\bauthor{\bsnm{Wu}, \binits{F.}},
\bauthor{\bsnm{Souza}, \binits{A.}},
\bauthor{\bsnm{Zhang}, \binits{T.}},
\bauthor{\bsnm{Fifty}, \binits{C.}},
\bauthor{\bsnm{Yu}, \binits{T.}},
\bauthor{\bsnm{Weinberger}, \binits{K.}}:
\bctitle{Simplifying graph convolutional networks}.
In: \bbtitle{Proceedings of the 36th International Conference on Machine Learning (ICML)},
pp. \bfpage{6861}--\blpage{6871}
(\byear{2019}).
\bcomment{PMLR}
\end{bchapter}
\endbibitem

\bibitem[\protect\citeauthoryear{Dasgupta and Gupta}{2003}]{dasgupta2003elementary}
\begin{barticle}
\bauthor{\bsnm{Dasgupta}, \binits{S.}},
\bauthor{\bsnm{Gupta}, \binits{A.}}:
\batitle{An elementary proof of a theorem of {J}ohnson and {L}indenstrauss}.
\bjtitle{Random Structures and Algorithms}
\bvolume{22}(\bissue{1}),
\bfpage{60}--\blpage{65}
(\byear{2003})
\end{barticle}
\endbibitem

\bibitem[\protect\citeauthoryear{Zhang et~al.}{2018}]{zhang2018billion}
\begin{bchapter}
\bauthor{\bsnm{Zhang}, \binits{Z.}},
\bauthor{\bsnm{Cui}, \binits{P.}},
\bauthor{\bsnm{Li}, \binits{H.}},
\bauthor{\bsnm{Wang}, \binits{X.}},
\bauthor{\bsnm{Zhu}, \binits{W.}}:
\bctitle{Billion-scale network embedding with iterative random projection}.
In: \bbtitle{Proceedings of the 2018 IEEE International Conference on Data Mining (ICDM)},
pp. \bfpage{787}--\blpage{796}
(\byear{2018}).
\bcomment{IEEE}
\end{bchapter}
\endbibitem

\bibitem[\protect\citeauthoryear{Talukder et~al.}{2025}]{talukder2025node2binary}
\begin{bchapter}
\bauthor{\bsnm{Talukder}, \binits{N.}},
\bauthor{\bsnm{Gyurek}, \binits{C.}},
\bauthor{\bsnm{Hasan}, \binits{M.A.}}:
\bctitle{{N}ode2binary: {C}ompact graph node embeddings using binary vectors}.
In: \bbtitle{Proceedings of the ACM on Web Conference 2025},
pp. \bfpage{495}--\blpage{505}
(\byear{2025})
\end{bchapter}
\endbibitem

\bibitem[\protect\citeauthoryear{Veli{\v{c}}kovi{\'c} et~al.}{2019}]{velivckovicdeep}
\begin{bchapter}
\bauthor{\bsnm{Veli{\v{c}}kovi{\'c}}, \binits{P.}},
\bauthor{\bsnm{Fedus}, \binits{W.}},
\bauthor{\bsnm{Hamilton}, \binits{W.L.}},
\bauthor{\bsnm{Li{\`o}}, \binits{P.}},
\bauthor{\bsnm{Bengio}, \binits{Y.}},
\bauthor{\bsnm{Hjelm}, \binits{R.D.}}:
\bctitle{Deep graph infomax}.
In: \bbtitle{Proceedings of the 7th International Conference on Learning Representations (ICLR)}
(\byear{2019}).
\burl{https://openreview.net/forum?id=rklz9iAcKQ}
\end{bchapter}
\endbibitem

\bibitem[\protect\citeauthoryear{Hou et~al.}{2023}]{hou2023graphmae2}
\begin{bchapter}
\bauthor{\bsnm{Hou}, \binits{Z.}},
\bauthor{\bsnm{He}, \binits{Y.}},
\bauthor{\bsnm{Cen}, \binits{Y.}},
\bauthor{\bsnm{Liu}, \binits{X.}},
\bauthor{\bsnm{Dong}, \binits{Y.}},
\bauthor{\bsnm{Kharlamov}, \binits{E.}},
\bauthor{\bsnm{Tang}, \binits{J.}}:
\bctitle{{G}raph{MAE}2: {A} decoding-enhanced masked self-supervised graph learner}.
In: \bbtitle{Proceedings of the ACM Web Conference 2023},
pp. \bfpage{737}--\blpage{746}
(\byear{2023})
\end{bchapter}
\endbibitem

\bibitem[\protect\citeauthoryear{Zheng et~al.}{2022}]{zheng2022rethinking}
\begin{bchapter}
\bauthor{\bsnm{Zheng}, \binits{Y.}},
\bauthor{\bsnm{Pan}, \binits{S.}},
\bauthor{\bsnm{Lee}, \binits{V.C.}},
\bauthor{\bsnm{Zheng}, \binits{Y.}},
\bauthor{\bsnm{Yu}, \binits{P.S.}}:
\bctitle{Rethinking and scaling up graph contrastive learning: {A}n extremely efficient approach with group discrimination}.
In: \bbtitle{Proceedings of the 36th International Conference on Neural Information Processing Systems (NeurIPS)},
pp. \bfpage{10809}--\blpage{10820}
(\byear{2022})
\end{bchapter}
\endbibitem

\bibitem[\protect\citeauthoryear{Huang et~al.}{2025}]{huang2025does}
\begin{bchapter}
\bauthor{\bsnm{Huang}, \binits{Y.}},
\bauthor{\bsnm{Zhao}, \binits{J.}},
\bauthor{\bsnm{He}, \binits{D.}},
\bauthor{\bsnm{Jin}, \binits{D.}},
\bauthor{\bsnm{Huang}, \binits{Y.}},
\bauthor{\bsnm{Wang}, \binits{Z.}}:
\bctitle{Does {GCL} need a large number of negative samples? {E}nhancing graph contrastive learning with effective and efficient negative sampling}.
In: \bbtitle{Proceedings of the AAAI Conference on Artificial Intelligence},
pp. \bfpage{17511}--\blpage{17518}
(\byear{2025})
\end{bchapter}
\endbibitem

\bibitem[\protect\citeauthoryear{Arriaga and Vempala}{2006}]{arriaga2006algorithmic}
\begin{barticle}
\bauthor{\bsnm{Arriaga}, \binits{R.I.}},
\bauthor{\bsnm{Vempala}, \binits{S.}}:
\batitle{An algorithmic theory of learning: {R}obust concepts and random projection}.
\bjtitle{Machine Learning}
\bvolume{63}(\bissue{2}),
\bfpage{161}--\blpage{182}
(\byear{2006})
\end{barticle}
\endbibitem

\bibitem[\protect\citeauthoryear{Martinsson et~al.}{2011}]{martinsson2011randomized}
\begin{barticle}
\bauthor{\bsnm{Martinsson}, \binits{P.-G.}},
\bauthor{\bsnm{Rokhlin}, \binits{V.}},
\bauthor{\bsnm{Tygert}, \binits{M.}}:
\batitle{A randomized algorithm for the decomposition of matrices}.
\bjtitle{Applied and Computational Harmonic Analysis}
\bvolume{30}(\bissue{1}),
\bfpage{47}--\blpage{68}
(\byear{2011})
\end{barticle}
\endbibitem

\bibitem[\protect\citeauthoryear{Qiu et~al.}{2019}]{qiu2019netsmf}
\begin{bchapter}
\bauthor{\bsnm{Qiu}, \binits{J.}},
\bauthor{\bsnm{Dong}, \binits{Y.}},
\bauthor{\bsnm{Ma}, \binits{H.}},
\bauthor{\bsnm{Li}, \binits{J.}},
\bauthor{\bsnm{Wang}, \binits{C.}},
\bauthor{\bsnm{Wang}, \binits{K.}},
\bauthor{\bsnm{Tang}, \binits{J.}}:
\bctitle{Net{SMF}: {L}arge-scale network embedding as sparse matrix factorization}.
In: \bbtitle{The World Wide Web Conference},
pp. \bfpage{1509}--\blpage{1520}
(\byear{2019})
\end{bchapter}
\endbibitem

\bibitem[\protect\citeauthoryear{Perozzi et~al.}{2014}]{perozzi2014deepwalk}
\begin{bchapter}
\bauthor{\bsnm{Perozzi}, \binits{B.}},
\bauthor{\bsnm{Al-Rfou}, \binits{R.}},
\bauthor{\bsnm{Skiena}, \binits{S.}}:
\bctitle{{D}eep{W}alk: {O}nline learning of social representations}.
In: \bbtitle{Proceedings of the 20th ACM SIGKDD International Conference on Knowledge Discovery and Data Mining},
pp. \bfpage{701}--\blpage{710}
(\byear{2014})
\end{bchapter}
\endbibitem

\bibitem[\protect\citeauthoryear{Blondel et~al.}{2008}]{blondel2008fast}
\begin{barticle}
\bauthor{\bsnm{Blondel}, \binits{V.D.}},
\bauthor{\bsnm{Guillaume}, \binits{J.-L.}},
\bauthor{\bsnm{Lambiotte}, \binits{R.}},
\bauthor{\bsnm{Lefebvre}, \binits{E.}}:
\batitle{Fast unfolding of communities in large networks}.
\bjtitle{Journal of Statistical Mechanics: Theory and Experiment}
\bvolume{2008}(\bissue{10}),
\bfpage{10008}
(\byear{2008})
\end{barticle}
\endbibitem

\bibitem[\protect\citeauthoryear{Traag et~al.}{2019}]{traag2019louvain}
\begin{barticle}
\bauthor{\bsnm{Traag}, \binits{V.A.}},
\bauthor{\bsnm{Waltman}, \binits{L.}},
\bauthor{\bsnm{Van~Eck}, \binits{N.J.}}:
\batitle{From {L}ouvain to {L}eiden: {G}uaranteeing well-connected communities}.
\bjtitle{Scientific Reports}
\bvolume{9}(\bissue{1}),
\bfpage{5233}
(\byear{2019})
\end{barticle}
\endbibitem

\bibitem[\protect\citeauthoryear{Newman}{2006}]{newman2006modularity}
\begin{barticle}
\bauthor{\bsnm{Newman}, \binits{M.E.}}:
\batitle{Modularity and community structure in networks}.
\bjtitle{Proceedings of the National Academy of Sciences (PNAS)}
\bvolume{103}(\bissue{23}),
\bfpage{8577}--\blpage{8582}
(\byear{2006})
\end{barticle}
\endbibitem

\bibitem[\protect\citeauthoryear{Bentley}{1975}]{bentley1975multidimensional}
\begin{barticle}
\bauthor{\bsnm{Bentley}, \binits{J.L.}}:
\batitle{Multidimensional binary search trees used for associative searching}.
\bjtitle{Communications of the ACM}
\bvolume{18}(\bissue{9}),
\bfpage{509}--\blpage{517}
(\byear{1975})
\end{barticle}
\endbibitem

\bibitem[\protect\citeauthoryear{Mizutani}{2021}]{mizutani2021improved}
\begin{barticle}
\bauthor{\bsnm{Mizutani}, \binits{T.}}:
\batitle{Improved analysis of spectral algorithm for clustering}.
\bjtitle{Optimization Letters}
\bvolume{15}(\bissue{4}),
\bfpage{1303}--\blpage{1325}
(\byear{2021})
\end{barticle}
\endbibitem

\end{thebibliography}
\end{document}